\documentclass{article}

\usepackage[preprint]{neurips_2026}

\usepackage[utf8]{inputenc} 
\usepackage[T1]{fontenc}    
\usepackage{hyperref}       
\usepackage{url}            
\usepackage{booktabs}       
\usepackage{amsfonts}       
\usepackage{nicefrac}       
\usepackage{microtype}      
\usepackage{xcolor}         
\usepackage{graphicx}      
\usepackage{wrapfig}       

\usepackage[most]{tcolorbox}

\tcbset{
  takeawaybox/.style={
    colback=white,      
    colframe=black,     
    boxrule=0.8pt,      
    arc=0mm,            
    left=6pt,
    right=6pt,
    top=6pt,
    bottom=6pt
  }
}

\newcommand{\RB}{\mathbb{R}}

\newcommand{\NB}{\mathbb{N}}

\newcommand{\FC}{\mathcal{F}}

\newcommand{\CC}{\mathcal{C}}

\newcommand{\GD}{\mathrm{GD}}

\newcommand{\widetildeparallel}[1]{%
  \widetilde{#1}%
  \kern0.1em 
  ^{\vphantom{#1}\parallel}%
}
\newcommand{\widetildeperp}[1]{%
  \widetilde{#1}%
  \kern0.1em 
  ^{\vphantom{#1}\perp}%
}
\newcommand{\Hom}{\mathrm{Hom}}

\newcommand{\tw}{\widetilde{w}}

\newcommand{\bw}{\overline{w}}
\newcommand{\bW}{\smash{\overline{W}}\vphantom{W}}
\newcommand{\bM}{\overline{M}}
\newcommand{\tr}{\mathrm{tr}}

\usepackage{amsmath}
\usepackage{amssymb}
\usepackage{mathtools}
\usepackage{amsthm}
\usepackage{tabularx}
\usepackage{enumitem}
\usepackage{tikz}
\usetikzlibrary{positioning}

\theoremstyle{plain}
\newtheorem{theorem}{Theorem}[section]
\newtheorem{proposition}[theorem]{Proposition}

\newtheorem{lemma}[theorem]{Lemma}

\theoremstyle{definition}

\newtheorem{assumption}[theorem]{Assumption}
\theoremstyle{remark}

\title{
Dynamics of Gradient Descent with Large Step Size Near a Manifold of Flat Minima\\
}

\author{%
  Lachlan Ewen MacDonald \quad \textbf{Ren\'{e} Vidal} \\
  Innovation in Data Engineering and Science (IDEAS)\\
  University of Pennsylvania\\
  Pennsylvania, PA 19104 \\
  \texttt{lemacdonald@protonmail.com} \\
}

\begin{document}

\maketitle

\begin{abstract}  
    An important quantity in the theory of gradient descent (GD) is the \emph{sharpness}, defined as the largest eigenvalue of the objective Hessian. Classical analyses typically require the step size to be uniformly smaller than twice the reciprocal of the sharpness, but this condition is frequently violated in the training of deep neural networks. Recent work \cite{macdonaldeos} bridges this gap in the setting of overparametrised least-squares with a \emph{single scalar output}, providing a normal form for large-step GD in a neighbourhood of an \emph{isolated} flat minimum and establishing three corresponding convergence results. In this paper, we extend this theory in two directions: (1) to overparametrised least-squares with \emph{vector-valued outputs} (including regression with arbitrarily many observations), and (2) to a neighbourhood of a \emph{manifold} of flat minima (which we show is essential for applications such as matrix factorisation). We generalise both the normal form and all three convergence theorems of \cite{macdonaldeos} to this broader setting, overcoming several technical challenges. We further show that our framework applies to deep matrix factorisation under mild assumptions, yielding several new structural results. In particular, we prove that the set of flat minima forms a fibre bundle over a product of spheres, and that the sharpness is Morse-Bott along this manifold.
        
\end{abstract}

\section{Introduction}

A natural way of finding a minimum of a smooth function $\ell:\RB^p\rightarrow\RB$ is by iterating the \emph{gradient descent} (GD) map
\begin{align}\label{eq:GD}
x\mapsto x-\eta\nabla\ell(x)
\end{align}
for some choice of step size $\eta>0$. Originally introduced by Cauchy almost two centuries ago \cite{cauchygd}, the method has more recently found success as the quintessential technique for training deep neural networks (DNNs). Unfortunately, despite its age, fame and simplicity, GD remains poorly understood in its application to deep learning (DL).

A critical quantity in the theory of GD is the \emph{sharpness}, i.e. the largest eigenvalue $\lambda_1$ of the Hessian of $\ell$. Classical theories of the convergence of GD typically assume both convexity of the objective $\ell$ and the pointwise stability condition
\begin{align}\label{eq:stability}
\eta < 2/\lambda_1(x),
\end{align}
without which GD diverges in the case of quadratic objectives. Both of these conditions are typically violated in the training of DNNs in practice, but convergence to a global minimiser is frequently observed nonetheless. The use of large step size in particular is often observed to accelerate convergence \cite{coheneos} and results in an implicit bias toward ``flat minima'' (i.e. minima with small sharpness $\lambda_1$), which have been associated with better statistical performance \cite{keskar}. Although tools from optimisation theory have shown some success in removing the stability condition \eqref{eq:stability} in convex settings \cite{grimmer1, grimmer2, grimmer3, altschuler1, altschuler2, wu2023eos, wu2024eos, wu2025eos}, convergence proofs with large step size in the \emph{non}-convex settings of DL have stubbornly resisted analysis by these means.

It is gradually becoming clear that analysis of GD in the regimes appropriate to DL instead require the tools of dynamical systems theory \cite{leegd, damianeos, macdonaldeos}. Of special significance for the present work is \cite{macdonaldeos}, which vindicates the dynamical systems approach in giving quantitative convergence theorems for gradient descent with large step size in ``codimension 1'' least squares problems, corresponding to a single scalar output and $(p-1)$-dimensional manifold $M\subset\RB^p$ of minimisers. Since extension of the results of \cite{macdonaldeos} is the central concern of the present paper, we briefly recall them here.

\textbf{The contributions of \cite{macdonaldeos}:} A \emph{normal form} is a change of coordinates in which the equations defining a dynamical system become easier to analyse. For codimension 1 least squares problems, \cite{macdonaldeos} provides conditions for such a normal form for GD with large step size in a neighbourhood of an isolated flat minimum $x_*\in M$. This normal form makes apparent that GD implicitly performs Riemannian gradient descent on the \emph{sharpness} $\lambda_1$ along the minima manifold $M$ (systematising insights of \cite{aroraeos, damianeos}) with step size controlled by the square distance of the iterates from $M$; these distances meanwhile evolve as a bifurcating dynamical system in the direction orthogonal to $M$.

Armed with this normal form, \cite{macdonaldeos} proves that the dynamics of GD bifurcate into three regimes in terms of the sharpness value $\lambda_*:=\lambda_1(x_*)$ at the flat minimum. In the \emph{subcritical regime}, when $\eta<2/\lambda_*$, exponential convergence to a suboptimally flat minimum is guaranteed following an initial period of non-monotonic iterate behaviour; in the \emph{critical regime}, when $\eta = 2/\lambda_*$, the iterates converge non-monotonically with a polynomial rate to the flat minimum; and in the \emph{supercritical regime}, when $\eta>2/\lambda_*$ is sufficiently small, the iterates converge exponentially to a stable period-2 orbit along the span of the normal vector through the flat minimum. These theorems rigorously make sense of a number of empirical observations made in prior work \cite{chen2023eos, kalraeos}, but their codimension 1 hypothesis is impractically restrictive, applying to regression of only a single datum.

\textbf{Technical challenges of analysis:} Although of limited practical interest, the codimension 1 setting considered in \cite{macdonaldeos} is theoretically far from trivial. The normal form in \cite{macdonaldeos} arises from the composite of several non-trivial coordinate transformations, one of which has a flawed proof in \cite{macdonaldeos} whose non-trivial rectification we perform in this paper. The convergence theorems that follow are also non-trivial, requiring careful of accounting of nonlinear behaviour.

Extending these results to higher codimension problems and thus obtaining theory of more practical relevance makes the technical challenges faced in \cite{macdonaldeos} vastly more difficult.
\begin{enumerate}[leftmargin=*]
    \item Increasing the codimension of the problem increases the dimension of the bifurcating component of the system, thus necessitating the introduction of a dimension-reduction technique for analysis.
    \item Natural \emph{examples} of higher codimension problems do \emph{not} admit \emph{isolated} flat minima but instead admit \emph{manifolds} of flat minima (see Subsection \ref{subsec:mf}), further increasing the degrees of freedom that must be dealt with by any theoretical analysis.
    \item Correct proof of the normal form presented in \cite{macdonaldeos} requires the solution of a partial differential equation (PDE) which is singular at flat minima. In the isolated flat minimum case considered in \cite{macdonaldeos}, existing literature can be used to solve this problem \cite{singularpde}; however, when the PDE is singular along a \emph{manifold} of flat minima, different techniques are required.
\end{enumerate}

\textbf{Paper contributions:} In this paper, we overcome all of these technical challenges and prove a vast generalisation of the theory of \cite{macdonaldeos} beyond the codimension 1 setting, encompassing underdetermined least squares problems of arbitrary codimension corresponding to overparametrised regression of arbitrarily many data. The high-level takeaway is a generalisation of that of \cite{macdonaldeos}:

\begin{tcolorbox}
\centering
For least squares problems of \emph{arbitrary codimension}, in a neighbourhood of a \emph{manifold of flat minima}, GD with a large step size implicitly performs \emph{Riemannian GD on the sharpness} along the solution manifold, and oscillates as a \emph{bifurcating dynamical system} in the directions orthogonal to the solution manifold.
\end{tcolorbox}

Specifically, we establish the following results for least squares problems of arbitrary codimension:
\begin{enumerate}[leftmargin=*]
    \item We provide a set of geometric hypotheses (see Subsection \ref{subsec:problemsetting}) sufficient to prove a normal form for gradient descent with large step size in a neighbourhood of a manifold of flat minima (see Section \ref{sec:normalform}), vastly generalising the normal form of \cite{macdonaldeos} which considers only codimension 1 problems with isolated flat minima . Our normal form reveals that GD acts as Riemannian GD on the sharpness along the solution manifold, oscillates as a bifurcating dynamical system along the top eigendirection of the Hessian orthogonal to the solution manifold, and contracts exponentially to zero along the other eigendirections of the Hessian.
    \item We prove generalisations of the subcritical, critical and supercritical convergence theorems derived in \cite{macdonaldeos} to this more general setting (see Section \ref{sec:convergencetheorems}).
    \item We prove that deep matrix factorisation problems fit into our framework (see Subsection \ref{subsec:mf}). In particular, we prove a number of novel results about the loss landscapes of matrix factorisation problems, including that the flat minima of such problems are a smooth fibre bundle over a product of spheres, and that the sharpness $\lambda_1$ is Morse-Bott (i.e., ``normally strongly convex") along this manifold. 
    \item We verify our theory with numerical experiments for matrix factorisation problems (see Section~\ref{sec:convergencetheorems}).
\end{enumerate}

\section{Related work}

\textbf{Gradient descent in DL with small step size:} 
A large body of work has analysed gradient descent (GD) for training deep neural networks under small step sizes by invoking the Polyak--Łojasiewicz (PL) inequality \cite{pl}. In overparametrised settings, this inequality can be deduced from the full-rank condition of the neural tangent kernel (NTK), i.e.\ the Gram matrix of parameter derivatives of the model \cite{jacot}, leading to numerous convergence guarantees for GD with step size satisfying the classical stability condition \eqref{eq:stability}; see, e.g., \cite{zhu, du1, du2, leewide, nguyen1, nguyen2, nguyen3, bombari}. In this context, overparametrisation has a precise technical meaning: that the NTK is full-rank. In our setting, this corresponds to Assumption~\ref{ass:regularity}, which ensures a smooth manifold structure for the solution set of least-squares problems and is equivalent to the NTK being full-rank along this set. While these approaches yield strong convergence guarantees, the small step size and initialisation regimes they require are known to limit feature learning \cite{chizat1}. Moreover, by relying on a PL inequality, they do not naturally capture implicit bias phenomena, which are central to understanding deep learning. In contrast, the present work focuses explicitly on the dynamics of GD iterates.

\textbf{Gradient descent in DL with large step size:} 
Classical stability analysis shows that \eqref{eq:stability} is necessary for convergence to a minimum of a given sharpness even in simple settings \cite{wudynamicalstability}. However, empirical studies beginning with \cite{coheneos} demonstrated that, in deep learning, GD with large step sizes often does not diverge; instead, it can converge at an accelerated rate. This behaviour is typically characterised by an initial \emph{progressive sharpening} phase, in which the sharpness increases along the iterates, followed by an \emph{edge of stability} regime in which the sharpness stabilises around $2/\eta$ and the loss decreases in a non-monotonic fashion. A substantial body of work has since sought to explain these phenomena \cite{ahneos, aroraeos, wangeos, wang2022eos, wang2023eos, damianeos, leeeos, zhueos, agarwalaeos, chen2023eos, kreislereos, wu2023eos, wu2024eos, cai2024eos, kalraeos, liueos, ghosh2025eos, yooeos, wu2025eos}. Broadly speaking, these works either aim to identify general mechanisms underlying edge-of-stability dynamics \cite{aroraeos, damianeos, cohen2025understanding}, or to obtain detailed analyses for specific model classes \cite{wang2022eos, wang2023eos, zhueos, agarwalaeos, chen2023eos, kreislereos, wu2023eos, wu2024eos, cai2024eos, kalraeos, liueos, yooeos, ghosh2025eos, wu2025eos}. While the former provide conceptual explanations, the latter often yield stronger guarantees, including convergence theorems in certain settings \cite{wu2023eos, wu2024eos, wu2025eos, liueos}. The work of \cite{macdonaldeos} bridges these perspectives by introducing geometric hypotheses that abstract from problem-specific details while remaining verifiable in concrete settings, and which are sufficiently strong to establish convergence results. The present paper continues this program by substantially extending these geometric hypotheses and their associated guarantees.

\section{Theoretical setting}

\subsection{Notation}

The Euclidean norm on Euclidean space will be denoted $\|\cdot\|$, and $I$ will denote the identity operator. Given a $C^k$ function $f:\RB^m\rightarrow\RB^n$, $D^kf$ will denote its $k^{th}$ order derivative. Beyond this, our geometric setting must be expressed in the language of differential geometry; the following notation will be used throughout. 

The tangent bundle of a manifold $M$ is denoted $TM$, with tangent space fibres $T_xM$ for $x\in M$. If $M$ is Riemannian, $d_M(x,y)$ will denote the geodesic distance (the length of the shortest curve) between $x,y\in M$. If $S\subset M$ is a submanifold, then we denote $d_M(x,S):=\inf_{y\in S}d_M(x,y)$, and denote by $\nu^MS$ the normal bundle of $S$ in $M$, whose fibre over $x\in S$ is the orthogonal complement of $T_xS\subset T_xM$; if $M$ is clear from context we will use $\nu S$ in place of $\nu^MS$.

If $f:M\rightarrow N$ is a map of manifolds and $S\subset M$, then $f|_S:S\rightarrow N$ will denote restriction. If $N=\RB$ and $f$ is $C^2$, then $\nabla_Mf$ and $\nabla_M^2f$ will denote its Riemannian gradient and Riemannian Hessian respectively; when $M$ is Euclidean, these are the ordinary gradient $\nabla f$ and Hessian $\nabla^2f$.

Given a fibre bundle $B\rightarrow M$ and a subset $S\subset M$, $B|_S\rightarrow S$ will denote the restriction of $B$ to $S$.

\subsection{Problem setting}\label{subsec:problemsetting}

In this subsection, we describe the geometric setting for our results, which generalises that of \cite{macdonaldeos}. We consider a $C^{\infty}$ model $f:\RB^p\rightarrow\RB^q$, where $p>q\geq 1$. Given $\tau\in\RB^q$, we consider the least squares objective $\ell:\RB^p\rightarrow[0,\infty)$ given by
\begin{align}\label{eq:ell}
\ell(w):=\frac{1}{2}\|\tau-f(w)\|^2,\qquad w\in\RB^p.
\end{align}
With $\ell$ assumed, the notation $\GD:\RB\times \RB^p\rightarrow\RB\times \RB^p$ will denote gradient descent, with step size $\eta>0$ included as the first parameter:
\begin{align}
\GD(\eta,w):=\big(\eta,w-\eta\nabla\ell(w)\big),\qquad w\in\RB^p.
\end{align}
This augmentation is necessary for the $\eta$-dependent normal form we provide in Theorem \ref{thm:normalformmain}.

Our first assumption gives a manifold structure to the minimisers of \eqref{eq:ell}.

\begin{assumption}\label{ass:regularity}
    The target vector $\tau\in\RB^q$ is contained in the range of $f$ and is a regular value of $f$.\footnote{$\tau$ is a regular value of $f$ if and only if $Df(x):\RB^p\rightarrow\RB^q$ is full-rank for all $x\in f^{-1}\{\tau\}$.}
\end{assumption}

By the regular value theorem \cite[Chapter 1, Theorem 3.2]{hirsch}, Assumption \ref{ass:regularity} implies that the minimisers $M:=f^{-1}\{\tau\}$ of $\ell$ form a nonempty $C^{\infty}$ submanifold of $\RB^p$; $M$ inherits from this embedding a Riemannian metric with which we assume it to be equipped in what follows. From here on, we will denote points in $M$ by $x$ to distinguish them from arbitrary points $w$ of $\RB^p$.

Observe that at any point $x\in M$, since $f(x) = \tau$ one has 
\begin{align}
\nabla^2\ell(x) = Df(x)^TDf(x) + (f(x)-\tau)^TD^2f(x) = Df(x)^TDf(x),
\end{align}
from which the following result is immediate.

\begin{proposition}\label{prop:decomposition}
At any $x\in M$:
    \begin{enumerate}[leftmargin=*]
        \item The kernels of $\nabla^2\ell(x)$ and $Df(x)$ coincide and are equal to $T_xM$.
        \item The eigenvectors (respectively, left-singular vectors) of $\nabla^2\ell(x)$ (resp. $Df(x)$) with nontrivial eigenvalue (resp. singular value) span the normal space $\nu_xM$.
    \end{enumerate}
\end{proposition}

Proposition \ref{prop:decomposition} plays a key conceptual role: it gives a correspondence between the linear geometry of $\nabla^2\ell$ (via its kernel and cokernel) and the differential geometry of $M$ (via its tangent and normal directions). It is ultimately these directions which appear (at least to first order) as the coordinates $x$ and $y$ respectively in our normal form (Theorem \ref{thm:normalformmain}).

We will denote by $\lambda_1\geq\dots\geq\lambda_q:M\rightarrow(0,\infty)$ the nonzero eigenvalue fields of $\nabla^2\ell$ along $M$. Our normal form requires $\lambda_1$ to be differentiable, however in general this only holds where $\lambda_1\neq \lambda_2$. Our next assumption says that this set is non-empty.

\begin{assumption}\label{ass:unique}
    The closed set $S:=\{x\in M:\lambda_1(x) = \lambda_2(x)\}$, outside of which $\lambda_1$ is $C^{\infty}$, is not equal to $M$; thus $\lambda_1$ is simple, hence $C^{\infty}$, on a nonempty open subset of $M$.
\end{assumption}

We will denote by $\nu_1$ the top eigenvector field of $\nabla^2\ell$ on $M\setminus S$; by Assumption \ref{ass:unique}, $\nu_1$ is $C^{\infty}$. While Proposition \ref{prop:decomposition} says that the directions normal to $M$ are spanned by \emph{all} the non-kernel eigenvectors of $\nabla^2\ell$, as we will see in Section \ref{sec:normalform} it is the $\nu_1$ direction (corresponding, to first order, to the variable $y_1$ in Theorem \ref{thm:normalformmain}) that is the most important of these for the dynamics of GD.

While Assumption \ref{ass:unique} guarantees nice behaviour in the $\nu_1$ direction \emph{orthogonal} to $M$, our next assumption guarantees nice behaviour \emph{along} $M$ toward a manifold of flat minima.

\begin{assumption}\label{ass:morsebott}
    There is a submanifold $F\subset M\setminus S$ of local minima for $\lambda_1|_{M\setminus S}$ such that:
    \begin{enumerate}[leftmargin=1.5em]
        \item $\lambda_1$ is Morse-Bott along $F$: more precisely, denoting $\nu^MF\subset TM$ for the normal bundle of $F$ in $M$, $\nabla_M^2\lambda_1|_{\nu^MF}\succ 0$.
        \item An open neighbourhood of zero in the span of $\nu_1|_{F}$ is invariant under gradient descent on $\ell$.
    \end{enumerate}
\end{assumption}

\begin{wrapfigure}{r}{0.55\textwidth}
\vspace{-8pt}
\centering
\resizebox{0.55\textwidth}{!}{
\begin{tikzpicture}[
    bubble/.style={
        draw,
        rounded corners,
        align=center,
        text width=2.9cm,
        inner sep=5pt,
        font=\normalsize
    },
    arrow/.style={->, thick}
]

\node[bubble, text width=6.5cm] (a31) at (0,0)
{\textbf{Assumption 3.1}\\
Solution manifold $M$\\
well-defined tangent/normal directions};

\node[bubble] (a33) at (-3.8,-2.1)
{\textbf{Assumption 3.3}\\
Nicely behaved\\
normal direction};

\node[bubble] (a34) at (0,-2.1)
{\textbf{Assumption 3.4}\\
Nicely behaved\\
tangent direction};

\node[bubble] (a35) at (3.8,-2.1)
{\textbf{Assumption 3.5}\\
Regularity of \\ normal form PDE};

\node[bubble, text width=4.8cm] (thm) at (0,-4.4)
{\textbf{Theorem 4.1}\\
Normal form for GD};

\draw[arrow] (a31) -- (a33);
\draw[arrow] (a31) -- (a34);
\draw[arrow] (a31) -- (a35);

\draw[arrow] (a33) -- (thm);
\draw[arrow] (a34) -- (thm);
\draw[arrow] (a35) -- (thm);

\end{tikzpicture}
}
\caption{Roles of the assumptions in establishing the normal form for gradient descent.}\label{fig:assumptions}
\vspace{-8pt}
\end{wrapfigure}
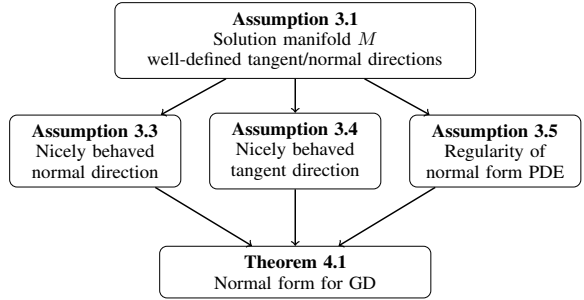

That $\lambda_1$ is Morse-Bott along $F$ reduces to the geodesic strong convexity of $\lambda_1$ in \cite{macdonaldeos} when $F$ is a single point; thus the Morse-Bott assumption is a kind of strong convexity of $\lambda_1$ in the directions normal to $F$ in $M$. Invariance of the span of $\nu_1|_F$ is a technical assumption required for correct decay in the higher order terms of the normal form; without it, GD is not necessarily attracted to $F$.

Our next and final assumption is used in the solution of a singular partial differential equation (PDE) required for our normal form.

\begin{assumption}\label{ass:quality}
    Denote by $\nu_1M$ and $\nu_{2:q}M$ the span of $\nu_1$ and its orthogonal complement in $\nu M|_{M\setminus S}$ respectively, with $P_1:T\RB^p|_{M\setminus S}\rightarrow\nu_1M$ and $P_{2:q}:=I-\nu_1\nu_1^T:T\RB^p|_{M\setminus S}\rightarrow \nu_{2:q}M$ the orthogonal projections. For each $\eta\in\RB$, define $A_{\eta}:M\rightarrow \RB^{p\times p}$ by
    \begin{align}
    A_{\eta}:=((2\lambda_1-\eta\lambda_1^2)I-\nabla^2\ell)|_{\nu_{2:q}M}^{-1}P_{2:q} + \lambda_1^{-1}(1-\eta\lambda_1)^{-1}P_1,
    \end{align}
    and define $\alpha_{\eta}:M\rightarrow \RB$ by
    \begin{align}
    \alpha_{\eta}:=-\bigg(D^3\ell[\nu_1,\nu_1,A_{\eta}\nabla^3\ell[\nu_1,\nu_1]] + \frac{1}{3}D^4\ell[\nu_1,\nu_1,\nu_1,\nu_1]\bigg)
    \end{align}
    Then, for all $x\in F$ and all $\eta$ in a neighbourhood of $2/\lambda_1|_F$, one has $\alpha_{\eta}(x)>0$ and $D\alpha_{\eta}(x) = 0$.
\end{assumption}

As shown in Theorem \ref{thm:normalformappendix}, $\alpha_{\eta}$ of Assumption \ref{ass:quality} appears as a coefficient in the PDE that must be solved for our normal form (Theorem \ref{thm:normalformmain}); that $\alpha_{\eta}|_F>0$ guarantees that this PDE can be solved to zeroth order, while $D\alpha_{\eta}|_F =0$ guarantees that its solution is $C^1$.

Collectively, these assumptions should be understood as forming a hierarchy of geometric conditions underpinning our analysis, see Figure \ref{fig:assumptions}. Viewed in this way, the assumptions are not merely technical, but rather reflect a structured loss landscape geometry underlying the dynamics of large-step GD.

\subsection{Case study: matrix factorisation}\label{subsec:mf}

In this subsection, we tie down these abstract assumptions by illustrating their application to a class of examples of interest in deep learning. Fixing $L\in\NB\setminus\{0,1\}$, consider \emph{$L$-layer matrix factorisation}, for which $p = \sum_{l=1}^L d_l\times d_{l-1}$, $q = d_L\times d_0$ and $f:\prod_{l=1}^L\RB^{d_l\times d_{l-1}}\rightarrow\RB^{d_L\times d_0}$ is given by
\begin{align}
    f(W_1,\dots,W_L):=W_L\cdots W_1.
\end{align}
Assumption \ref{ass:regularity} then holds when $d_l\geq d_L$ for all $l\leq L$ and $\tau$ has full rank $d_L$ with simple top singular value $\sigma_1$ (Proposition \ref{prop:regularvalue}). Assumptions \ref{ass:unique}, \ref{ass:morsebott} and \ref{ass:quality} also all apply under these conditions; see Appendix \ref{sec:matrixfactorisation} for details. The flat minima manifold $F$ in this case has a particularly interesting structure, not noted in prior work \cite{mulayoff2, marion}, which we now describe.

First note that up to a linear isometry which preserves $f$, $\tau$ may be assumed to be (rectangular) diagonal, $\tau = \mathrm{diag}(\sigma_1,\dots,\sigma_{d_L})$. Proposition \ref{prop:fibrebundle} then says that the global minimisers $F$ of $\lambda_1$ in $M\setminus S$ form a fibre bundle over the product $\prod_{l=1}^{L-1}S^{d_l-1}$ of unit spheres $S^{d_l-1}\subset\RB^{d_l}$, whose typical fibre is an open subset of the solution manifold $\overline{M}\subset \prod_{l=1}^L\RB^{(d_l-1)\times (d_{l-1}-1)}$ for the lower-dimensional factorisation of $\tau_{2:d} = \mathrm{diag}(\sigma_2,\dots,\sigma_d)$ (see Figure \ref{fig:fibrebundle} for a simple example). \begin{wrapfigure}{r}{0.5\textwidth}
  \centering
  \includegraphics[width=0.5\textwidth]{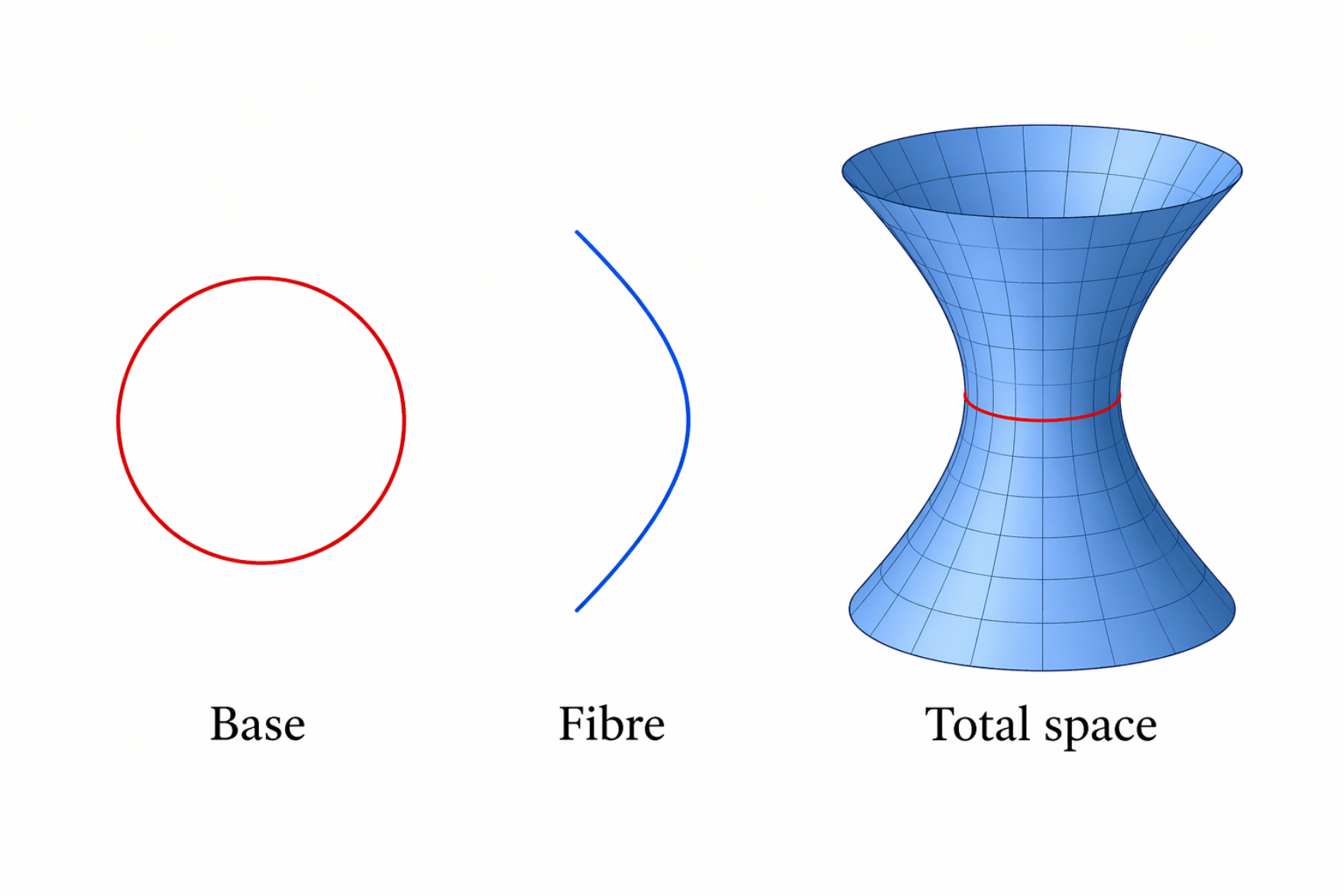}
  \caption{The base (left), connected component of fibre (centre) and connected component of total space (right) of $F$ for 2-layer matrix factorisation with $d_0=d_1=d_2=2$. The base is the circle $S^1$, while the fibre is an open subset of the solution manifold of the factorisation problem of 1 dimension lower (which, for $d_0=d_1=d_2=2$, is simply a 1-dimensional hyperbola). The total space is obtained by attaching a copy of the fibre to each point in the base. This total space is a 2-dimensional submanifold of the 4-dimensional solution manifold $M$, which is itself a submanifold of the 8-dimensional parameter space $\RB^{2\times2}\times\RB^{2\times 2}$.}
  \label{fig:fibrebundle}
  \vspace{-2em}
\end{wrapfigure}Moreover, up to a linear isometry that preserves $f$, every point of $F$ has the form
\begin{align}\label{eq:wnormalformmain}
    \Bigg(\begin{pmatrix} \sigma_1^{1/L} & 0 \\ 0 & \overline{W}_l\end{pmatrix}\Bigg)_{l=1}^L
\end{align}
for some $(\overline{W}_1,\dots,\overline{W}_L)\in \overline{M}$.

This fibre bundle structure has important consequences for the computation of the spectrum of $\nabla^2_M\lambda_1|_F$: it turns out that the directions in $M$ that are \emph{normal} to $F$ correspond to either (i) conjugations of the factors of \eqref{eq:wnormalformmain} by scalings of the top standard basis vector; (ii) conjugations of the factors of \eqref{eq:wnormalformmain} by certain linear maps sending the top standard basis vector into other subspaces. The eigenvalues of $\nabla^2_M\lambda_1|_F$ corresponding to (i) are all constant, equal to $4\sigma_1^{2-4/L}$ (and in particular equal to $4$ when $L=2$). Those corresponding to (ii) are strictly positive and blow up to infinity approaching $S$. In either case, all eigenvalues of $\nabla^2\lambda_1|_{\nu^MF}$ are strictly positive, making $\lambda_1$ Morse-Bott along $F$ (Proposition \ref{prop:morsebott}).

\section{Normal form for GD in arbitrary codimension near flat minima manifold}\label{sec:normalform}

In this section, we state our normal form for large-step GD in a neighbourhood of the flat minima manifold $F$ (Theorem \ref{thm:normalformmain}).

Like the normal form in \cite{macdonaldeos}, our more general normal form splits the dynamics of GD into directions tangent to $M$ (which we coordinatise by $x$) having the form of Riemannian GD (RGD) on the sharpness $\lambda_1$, and orthogonal to $M$ (which we coordinatise by $y = (y_1,y_{2:q})$). Unlike in \cite{macdonaldeos}, however, in which the entire orthogonal direction is only 1-dimensional and evolves as a flip bifurcation, in our more general setting there are $q$ independent orthogonal directions corresponding to our codimension $q$ setting.  The first of these orthogonal coordinates, $y_1$, aligns to first order with the top eigendirection $\nu_1$ of $\nabla^2\ell$ and has the same flip bifurcation form as in \cite{macdonaldeos}; the $q-1$ remaining directions, $y_{2:q}$, correspond to the smaller eigendirections of $\nabla^2\ell$ and are stable contracting directions.

\begin{theorem}[Informal]\label{thm:normalformmain}
    Under Assumptions \ref{ass:regularity}, \ref{ass:unique}, \ref{ass:morsebott} and \ref{ass:quality}, about any $\bar{x}\in F$ there is a $\eta$-dependent change of coordinates $(x,y_1,y_{2:q})\in M\times \RB\times\RB^{q-1}$ which is $C^4$ in $y$ and $C^1$ in $x$, in which $\GD$ takes the form $\GD(\eta,x,y_1,y_{2:q}) = \big(\eta,\GD_x(\eta,x,y_1),\GD_{y_1}(\eta,x,y_1),\GD_{y_{2:q}}(\eta,x,y_1,y_{2:q})\big)$, where
    \begin{align}
    \GD_x(\eta,x,y_1) &= x - \zeta(\eta,x)y_1^2\nabla_M\lambda_1(x) + O\big(|y_1|^3d_M(x,F)\big),
\\
    \GD_{y_1}(\eta,x,y_1) &= (1-\eta\lambda_1(x))y_1 +  y_1^3 + O\big( y_1^4\big),
  \\
    \GD_{y_{2:q}}(\eta,x,y_1,y_{2:q}) &= \big(I_{q-1}-\eta \Lambda_{2:q}(x)\big)y_{2:q} + O\big(|y_1|\|y_{2:q}\|,\|y_{2:q}\|^2\big)
    \end{align}
    uniformly over all $\eta$ in a neighbourhood of $2/\lambda_1|_F$ and over all $(x,y)$ in a neighbourhood of $\bar{x}$ as $y_1,y_{2:q}\rightarrow 0$ and $x\rightarrow F$. Here $\zeta$ is a $C^1$ function such that $\zeta(\eta,\cdot)|_F\equiv \alpha_{\eta}^{-1}|_F$, and $\Lambda_{2:q}$ is a $C^{\infty}$ field of symmetric matrices whose eigenvalues coincide at each $x\in M$ with those of $\nabla^2\ell(x)|_{\nu_{2:q}M}$.
\end{theorem}

At an intuitive level, our Theorem \ref{thm:normalformmain} allows us to anticipate how the dynamics will behave even without heavy analysis. For $\eta$ near $2/\lambda_1|_F$, the spectral norm $\|I_{q-1}-\eta\Lambda_{2:q}(\bar{x})\|_2<1$, so that $\GD_{2:q}$ acts as a linear contraction near $\bar{x}$. Meanwhile, $\GD_x$ acts as RGD on $\lambda_1$, with step size controlled by $y_1^2$ and $\zeta$, while $\GD_{y_1}$ makes $y_1$ evolve as a flip bifurcation from dynamical systems theory \cite{kuznetsov} with linear component differing from $-1$ by a perturbation of $\eta\lambda_1|_F-2$. In particular:
\begin{enumerate}[leftmargin=*]
    \item If $\eta\lambda_1|_F<2$, $\GD_{y_1}$ contracts \emph{linearly} to zero, so $\GD_x$ behaves as RGD with \emph{exponentially} decaying step size, leading to \emph{exponential} convergence to a global minimum of $\ell$ which is suboptimally flat unless the $x$-coordinate starts in $F$.
    \item If $\eta\lambda_1|_F=2$, $\GD_{y_1}$ contracts \emph{sub-linearly} to zero, so $\GD_x$ behaves as RGD with \emph{polynomially} decaying step size, leading to \emph{polynomial} convergence to $F$.
    \item If $\eta\lambda_1|_F>2$, $\GD_{y_1}$ tends \emph{exponentially} toward a nonzero periodic orbit, so $\GD_x$ behaves as RGD with \emph{constant} step size, leading to \emph{exponential} convergence to a periodic orbit centred on $F$.
\end{enumerate}
These behaviours are formalised in convergence theorems in the next section.

On a technical level, our Theorem \ref{thm:normalformmain} is substantially more difficult than the corresponding result in \cite{macdonaldeos}. First, a dimension reduction must be undertaken using a centre manifold and stable foliation theorem \cite{hpsinvariant, fenichel1974asymptotic,fenichel_singular} to reduce the effective dimension of the problem from $q>1$ to $q=1$, which is unnecessary in the $q=1$ setting of \cite{macdonaldeos}. Second, the dynamics along this centre manifold can be transformed into the claimed normal form only by solving a PDE which is singular along $F$ (Theorem \ref{thm:pdesolve}). Proving the existence of a sufficiently regular solution to this PDE (without which the normal form is impossible) requires the invocation of further techniques from dynamical systems theory \cite{belitskii1978,sternberg1957}; in contrast, the corresponding PDE in \cite{macdonaldeos} is singular only at a point, and has been well studied in prior literature \cite{singularpde}.

\section{Convergence theorems}\label{sec:convergencetheorems}

In this section we state our convergence theorems and provide the results of numerical simulations supporting them. As in \cite{macdonaldeos}, the dynamics of GD admit three qualitatively different convergence behaviours: subcritical, critical and supercritical, corresponding to the value of $\eta$ relative to the stability threshold $2/\lambda_1|_F$. Given a point $(x,y):=(x,y_1,y_{2:q})$ in the coordinates of Theorem \ref{thm:normalformmain}, its GD iterates will be denoted by $\GD^t(x,y_1,y_{2:q}) =: (x_t,y_{1,t},y_{2:q,t})=:(x_t,y_t)$. We assume in all of our theorems that $y_{1,0}\neq 0$; otherwise, it follows immediately from our normal form that the iterates contract exponentially rapidly to a nearby point in $M$ without any implicit bias toward $F$ in general.

\subsection{Subcritical regime}

We first state our convergence theorem in the \emph{subcritical regime}, where the step size satisfies $\eta<2/\lambda_1|_F$. In this case, after a possible transient period of initial instability (in which the iterates move further away from the solution manifold), due to the descent on $\lambda_1$ appearing in Theorem \ref{thm:normalformmain} the iterates eventually reach a point where $\eta<2/\lambda_1|_M$ and thereafter converge exponentially to a global minimum of $\ell$ which is suboptimally flat unless the initial $x$-coordinate lies in $F$, see Figure \ref{fig:subcriticalmain}. After appropriately accounting for the generalisation from $\lambda_1$ being geodesically strongly convex about a single flat minimum, as in \cite{macdonaldeos}, to $\lambda_1$ being Morse-Bott along the manifold $F$, as in this setting, the proof of Theorem \ref{thm:subcriticalmain} follows a similar argument to that of \cite[Theorem 5.1]{macdonaldeos}. See Subsection \ref{subsec:subcriticalappendix} for details.

\begin{theorem}\label{thm:subcriticalmain}
    Suppose that Assumptions \ref{ass:regularity}, \ref{ass:unique}, \ref{ass:morsebott} and \ref{ass:quality} hold, and that $\eta <2/\lambda_1|_F$. Then there is $\gamma>0$ such that for all $(x_0,y_0)$ sufficiently close to $F$ with $y_{1,0}\neq 0$, there is
    \begin{align}
    T = O\bigg(y_{1,0}^{-2}\bigg(\frac{\lambda_1(x_0)-\lambda_1|_F}{2/\eta-\lambda_1|_F}\bigg)^{\gamma}\bigg)
    \end{align}
    such that $\eta <2/\lambda_1(x_t)$ for all $t\geq T$, following which there is $\beta<1$ such that the iterates $(x_t,y_t)$ converge with rate $O(\beta^t)$ to  global minimum $(x_{\infty},0)$ of $\ell$. If, furthermore $d_M(x_0,F)>0$, this minimum is suboptimally flat:
    \begin{align}
    \lambda_1(x_{\infty})-\lambda_1|_F\geq \exp\big(-O(y_{1,T}^2(1-\beta^2)^{-1})\big)(\lambda_1(x_T)-\lambda_1|_F).
    \end{align}
\end{theorem}

\begin{figure}[t]
    \centering
    \includegraphics[width=\linewidth]{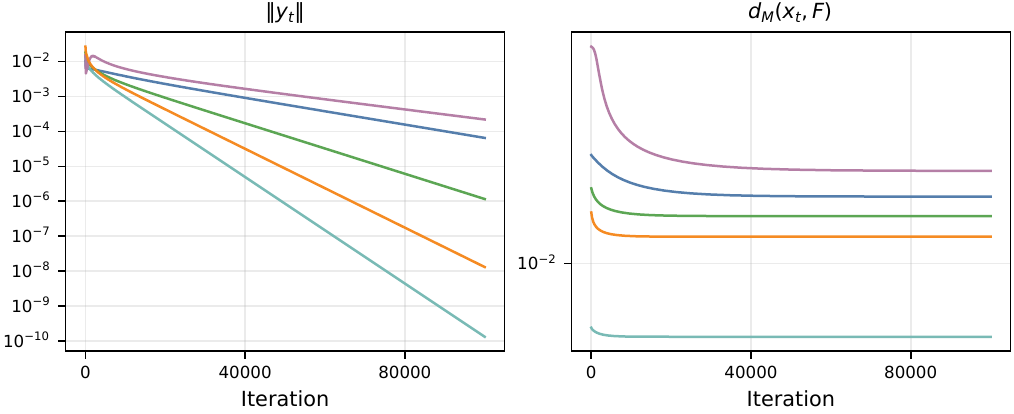}
    \caption{Log $y$-scale plots of $\|y_t\|$ (left) and $d_M(x_t,F)$ (right) for 5 independent trials  of gradient descent in the \textbf{subcritical} regime on a 3 layer, $2\times 2$ matrix factorisation problem. Initial instability (rising $\|y_t\|$) is overcome in finite time followed by exponential convergence to a suboptimally flat minimum.}
    \label{fig:subcriticalmain}
\end{figure}

\subsection{Critical regime}

The \emph{critical regime} is when $\eta=2/\lambda_1|_F$. In this regime, one observes the iterates converge non-monotonically at a $t^{-1/2}$ rate to a flat minimiser in $F$, see Figure \ref{fig:criticalmain}. The corresponding theorem in \cite{macdonaldeos} invokes an analytic parabolic invariant manifold theorem \cite{hakimhigher} to obtain the convergence rate, however since the \emph{corrected} normal form is not necessarily analytic, this proof technique does not work without modification. In contrast, our more general critical regime convergence theorem below uses a simple Lyapunov function argument to establish the convergence guarantee in the low-regularity setting supplied by the corrected normal form.

\begin{theorem}\label{thm:criticalmain}
    Suppose that Assumptions \ref{ass:regularity}, \ref{ass:unique}, \ref{ass:morsebott} and \ref{ass:quality} hold. Then the iterates $(x_t,y_t)$ of any $(x_0,y_0)$ sufficiently close to $F$ with $y_{1,0}\neq 0$ converge to a point in $F$ with rate $\Theta(t^{-1/2})$.
\end{theorem}

\begin{figure}[t]
    \centering
    \includegraphics[width=\linewidth]{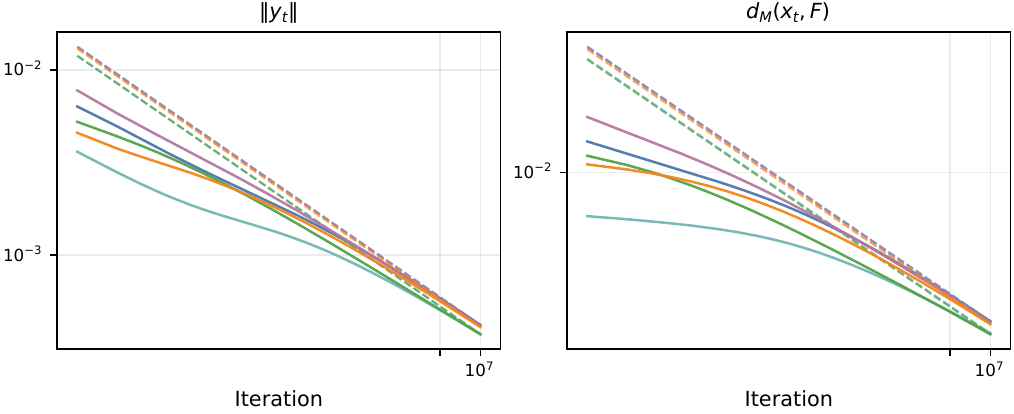}
    \caption{Log-log plots of $\|y_t\|$ (left, solid) and $d_M(x_t,F)$ (right, solid) for 5 independent trials  of gradient descent in the \textbf{critical} regime on a 3 layer, $2\times 2$ matrix factorisation problem. Dotted lines show $t^{-1/2}$ passing through the final values of each trial for reference. All trials exhibit the predicted asymptotic $t^{-1/2}$ convergence to an element of $F$.}
    \label{fig:criticalmain}
\end{figure}

\subsection{Supercritical regime}

The \emph{supercritical} regime is that in which the step size $\eta$ is (slightly) larger than the stability threshold $2/\lambda_1|_F$. In this regime, the iterates converge exponentially to a period-two cycle along the span of the vector field $\nu_1|_F$, see Figure \ref{fig:supercriticalmain}.  Our proof uses the same scalar period-doubling estimates as \cite{macdonaldeos}, together with a new invariant-neighbourhood and finite-entry argument which yields a neighbourhood whose radius remains $\Theta(1)$ as $\eta\lambda_1|_F-2\to0$, in contrast to the corresponding result in \cite{macdonaldeos}. See Subsection \ref{subsec:supercriticalappendix} for details.

\begin{theorem}\label{thm:supercriticalmain}
    Suppose that Assumptions \ref{ass:regularity}, \ref{ass:unique}, \ref{ass:morsebott} and \ref{ass:quality} hold. Then there is a constant $C>0$ such that for all sufficiently small $\eta>2/\lambda_1|_F$ and all $(x_0,y_0)$ sufficiently close to $F$ with $y_{1,0}\neq 0$, the iterates $(x_t,y_t)$ converge to a period-two orbit of amplitude $\Theta((\eta\lambda_1|_F-2)^{1/2})$ along the span of $\nu_1|_F$ with rate $O((1-C(\eta\lambda_1|_F-2))^t)$.
\end{theorem}

\begin{figure}[t]
    \centering
    \includegraphics[width=\linewidth]{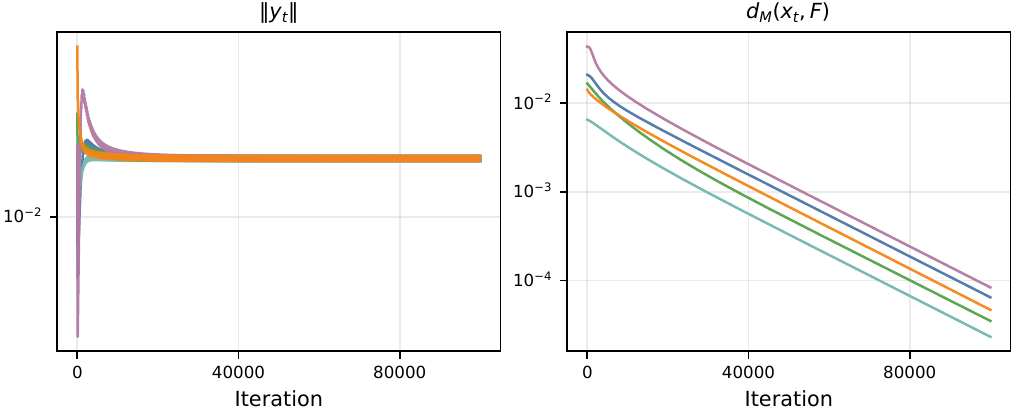}
    \caption{Log $y$-scale plots of $\|y_t\|$ (left) and $d_M(x_t,F)$ (right) for 5 independent trials  of gradient descent in the \textbf{supercritical} regime on a 3 layer, $2\times 2$ matrix factorisation problem. All trials exhibit the same exponential convergence rate to the claimed period-2 cycle along $\nu_1|_F$.}
    \label{fig:supercriticalmain}
\end{figure}

\section{Limitations, discussion and conclusion}

Our work addresses one of the open questions left by \cite{macdonaldeos}, namely the extension of the theory to higher codimension. Our extension reveals that the insights of \cite{macdonaldeos} persist in higher dimension, with GD dynamics splitting into RGD along $M$ coupled to a flip bifurcation and contracting stable directions orthogonal to $M$. However, several questions remain unaddressed by our work.

\textbf{Beyond the supercritical regime:} Prior work indicates that further increasing the step size beyond the supercritical limit considered in this work results in attractors of higher periodicity and chaos \cite{kalraeos}. Addressing this higher order behaviour remains an open problem, for which we hope our work will serve as foundation.

\textbf{Global convergence:} Although our convergence theorems hold in a neighbourhood of any point along the flat minima manifold, they remain merely \emph{local} in that they say nothing of convergence \emph{away} from this manifold. In particular, we do not believe that our results address either the ``progressive sharpening'' or ``edge of stability'' regimes identified in \cite{coheneos}, except at the tail-end after the flat minima manifold has been found. Nonetheless, we anticipate that the geometric spirit of our analysis will be useful in providing rigorous theory for these regimes in future work.

\textbf{Geometry of flat minima:} The popular hypothesis that flat minima generalise better \cite{hochreiter, keskar} has recently been challenged, with \cite{granziol, wen} showing that flatness is not sufficient for good generalisation. Our discovery of the manifold structure of flat minima in matrix factorisation raises the intriguing possibility that \emph{flat minima are not all equal} from the perspective of generalisation. A promising avenue for future research is the identification of disinguished submanifolds of flat minima which do exhibit better generalisation, and more efficient algorithms directed toward such submanifolds.

\textbf{Acknowledgements:} This work was supported by Simons Foundation Grant 81420 (Theorinet). LM thanks Glen Wheeler, Masafumi Yoshino and Johan L\"{o}fberg for helpful email exchanges.

\bibliographystyle{plain}
\bibliography{references}

\newpage

\appendix

\section{Additional notation}\label{app:notation}

Given vector bundles $E$ and $E'$ over a space $X$ denote by $\mathrm{Hom}(E,E')$ the vector bundle over $X$ whose fibre over $x\in X$ is the space of linear maps $E_x\rightarrow E'_x$. Given another vector bundle $E''$ over $X$ and sections $\sigma,\sigma'$ of $\mathrm{Hom}(E,E')$ and $\mathrm{Hom}(E',E'')$ respectively, $\sigma'\circ\sigma$ will denote the section of $\Hom(E,E'')$ given by fibrewise composition of $\sigma$ and $\sigma'$. The notation $E\otimes E'$ will be used to denote the tensor product bundle whose fibre over $x\in X$ is $E_x\otimes E_x'$. Given $k\in\NB$, $E^{\otimes k}$ will denote the $k^{th}$ tensor power of $E$, which is the vector bundle over $X$ whose fibre over $x\in X$ is the $k^{th}$ tensor power
\begin{align}
E_x^{\otimes k}:=\underbrace{E_x\otimes\cdots\otimes E_x}_{\text{$k$ times}}
\end{align}
of the vector space $E_x$. Given additional vector bundles $F,F'$ over $X$ and sections $\sigma_E,\sigma_F$ of $\mathrm{Hom}(E,E')$ and $\Hom(F,F')$ respectively, their tensor product $\sigma_E\otimes\sigma_F$ is the section of $\Hom(E\otimes F,E'\otimes F')$ defined by
\begin{align}
(\sigma_E\otimes\sigma_F)(x)[v\otimes w]:=\sigma_E(x)[v]\otimes\sigma_F(x)[w],\qquad x\in X,\quad v\in E_x,\,w\in F_x.
\end{align}
Given $k\in\NB$, $\sigma^{\otimes k}$ will denote the section of $\Hom(E^{\otimes k},(E')^{\otimes k})$ defined by
\begin{align}
\sigma^{\otimes k}(x)[v_1\otimes\cdots\otimes v_k]:=\sigma(x)[v_1]\otimes\cdots\otimes\sigma(x)[v_k],\qquad (x,v_i)\in E\quad \forall i=1,\dots,k.
\end{align}
Symmetric tensor products $E\odot E'$, $\sigma_E\odot\sigma_F$ and powers $E^{\odot k}$ and $\sigma^{\odot k}$ are defined in a formally identical fashion with $\otimes$ replaced by $\odot$.

\section{Normal form for GD}

In this section, we prove the normal form Theorem \ref{thm:normalformmain} for GD. We will in fact prove the more precise Theorem \ref{thm:normalformappendix}. We adopt the notation used throughout the main body of the paper as well as that of Appendix \ref{app:notation}. Our proof will factor through three distinct coordinate changes; the first merely rewrites $\GD$ in tubular neighbourhood coordinates; the second converts these tubular coordinates into centre-manifold and stable-foliation coordinates in a manner similar to Fenichel \cite{fenichel_singular} to distinguish an appropriately invariant orthogonal coordinate corresponding to the top eigenvalue of the Hessian to the dynamics along the solution set; the third coordinate transformation deforms this distinguished centre manifold coordinate into a form more amenable to convergence analysis by solving a partial differential equation (see Section \ref{sec:singularpde}.

Let us fix $\bar{x}\in F$ as in the statement of Theorem \ref{thm:normalformmain}. Although $\nu_{2:q}M$ need not be globally trivialisable\footnote{Consider, for instance, the model $f:\RB^3\times\RB^2\rightarrow\RB^3$ defined by $f(x,y):=(\|x\|^2-1,y)$}, it is \emph{locally} trivial, so there is an open neighbourhood $V^1_x\subset M\setminus S$ of $\bar{x}$ on which we can find a $C^{\infty}$ orthonormal frame field $\nu_{2:q}:V^1_x\rightarrow \nu_{2:q}M|_{V^1_x}$, which is completed by $\nu_1$ to a local orthonormal frame field $\nu:V^1_x\rightarrow \nu M|_{V^1_x}$ for $\nu M$. We will assume this $\nu$ to be fixed in what follows.

We must specify some notation concerning derivatives of $\nu$. The field $\nu:V^1_x\rightarrow \RB^{p\times q}$ acts on vectors $z\in\RB^q$ to give $\nu(x)z\in \nu_xM\subset T_x\RB^p$ for any $x\in V^1_x$. Its \emph{derivative} $D_M\nu(x)$ therefore acts on a pair of vectors, one corresponding to the direction in which it is differentiated in $V^1_x\subset M$, and the other corresponding to the evaluation of $\nu$ on $q$-vectors as in the previous sentence; thus, given $v_x\in T_xM$ and $z\in \RB^q$, we will denote $D_M\nu(x)[v_x]z\in T_x\RB^p$ for this dual evaluation. A similar remark holds for the higher derivatives of $\nu$; given $k$ vectors $v_1,\dots,v_k\in T_xM$ and $z\in \RB^q$, we will denote $D^k_M\nu(x)[v_1,\dots,v_k]z\in T_x\RB^p$ for the $k^{th}$ order derivative of $\nu$ evaluated on the $v_1,\dots,v_k$ in the $M$-directions.

With the local frame $\nu_{1:q}$ in hand, recall that we can find an open neighbourhood of $V^1_y$ of $0$ in $\RB^q$ such that the normal exponential map $E:V^1_x\times V^1_y\rightarrow\RB^p$ defined by
\begin{align}
E(x,y):=x + \nu(x)z,\qquad (x,y)\in V^1_x\times V^1_y
\end{align}
is a diffeomorphism onto its image. Our first result gives a leading order expression for the inverse of this map, and generalises \cite[Proposition C.3]{macdonaldeos} to the case where $q$ may be greater than 1.

\begin{lemma}\label{lem:Einv}
    Given $x\in V^1_x$ and $v\in T_x\RB^p$, write $v = v_x + \nu_{1:q}(x)v_y$, where $v_x\in T_xM$ and $v_y\in \RB^q$. Then for all $x\in V_x$ and all $v\in T_x\RB^p$ sufficiently small that $x+v\in E(V^1_x\times V^1_y)$, one has
    \begin{align}
    E^{-1}(x + v) = \begin{pmatrix} x + v_x - P_{TM}(x)D_M\nu(x)[v_x]v_y \\ v_y - \nu(x)^TD_M\nu(x)[v_x]v_y\end{pmatrix} + O\big(\|v_x\|^3,\|v_x\|^2\|v_y\|,\|v_x\|\|v_y\|^2\big)
    \end{align}
    uniformly over $E(V^1_x\times V^1_y)$ as $v\rightarrow 0$.
\end{lemma}

\begin{proof}
    Since $E^{-1}$ is $C^{\infty}$, by Taylor's theorem, one has
    \begin{align}\label{eq:Einv}
    E^{-1}(x+v) = E^{-1}(x) + DE^{-1}(x)[v] + \frac{1}{2}D^2E^{-1}(x)[v,v] + \mathrm{error}(x,v),
    \end{align}
    for some $C^{\infty}$ remainder function $\mathrm{error}(x,v)$. By \cite[Lemma C.3]{macdonaldeos}, since $E(x,y) = x + \nu(x)y$ has all $z$-derivatives of order $>1$ vanishing, this remainder function satisfies
    \begin{align}
    \mathrm{error}(x,v) = O\big(\|v_x\|^3,\|v_x\|^2\|v_y\|,\|v_x\|\|v_y\|^2\big)
    \end{align}
    uniformly over $E(V^1_x\times V^1_y)$ as $v\rightarrow 0$.
    
    Now, to compute the leading order terms of \eqref{eq:Einv}, observe first that for any $(x,y)\in V$, one has
    \begin{align}
    DE(x,y)\bigg[\begin{pmatrix} v_x\\v_y\end{pmatrix}\bigg] = v_x+ D_M\nu(x)[v_x]y + \nu(x)v_y
    \end{align}
    so that in particular
    \begin{align}
    DE(x,0)\bigg[\begin{pmatrix} v_x \\ v_y\end{pmatrix}\bigg] = v_x + \nu(x)v_y,
    \end{align}
    while
    \begin{align}
        D^2E(x,y)\bigg[\begin{pmatrix} v_x\\v_y\end{pmatrix}^{\odot 2}\bigg] = D^2_M\nu(x)[v_x^{\odot 2}]y + 2D_M\nu(x)[v_x]v_y
    \end{align}
    so that in particular
    \begin{align}
    D^2E(x,0)\bigg[\begin{pmatrix} v_x\\v_y\end{pmatrix}\bigg] = 2D_M\nu(x)[v_x]v_y.
    \end{align}
    By differentiating the identity $E^{-1}\circ E = \mathrm{identity}$, one then obtains
    \begin{align}
    DE^{-1}(x) = DE(x,0)^{-1} = \begin{pmatrix} P_{TM}(x)\\\nu(x)^T\end{pmatrix},
    \end{align}
    while
    \begin{align}
    D^2E^{-1}(x)[v,v] &= -DE(x,0)^{-1}D^2E(x,0)[DE(x,0)^{-1}v,DE(x,0)^{-1}v]\\&=-2\begin{pmatrix} P_{TM}(x)D_M\nu(x)[v_x]v_y \\ \nu(x)^TD_M\nu(x)[v_x]v_y\end{pmatrix}.
    \end{align}
    This completes the proof. 
\end{proof}

We remark that the term $\nu^TD_M\nu$ that appears in Lemma \ref{lem:Einv} is, in general, the expression of the connection on the vector bundle $\nu M\rightarrow M$ inherited from its embedding into the trivial bundle $T\RB^p|_M\rightarrow M$ (with its trivial connection) and expressed in the frame $\nu$. Since this connection is metric-compatible and $\nu$ is an orthonormal frame field, one is guaranteed that $\nu^TD_M\nu$ evaluates on any tangent vector $v_x\in T_xM$ to give an antisymmetric $q\times q$ matrix $\nu(x)^TD_M\nu(x)[v_x]$; if $q=1$, this is enough to force the term to vanish, however it need not vanish in general.

Having fixed the local orthonormal frame field $\nu_{1:q}$, it will be convenient to introduce the following notation. Letting $P_{TM}:T\RB^p|_M\rightarrow TM$ denote the orthogonal projection, we denote
\begin{align}
\nabla^k\ell_{\nu}:=\nabla^k\ell\circ \nu_{1:q}^{\odot (k-1)}\in \Hom\big(V^1_x\times(\RB^q)^{\odot(k-1)},T\RB^p|_{M\setminus S}\big),
\end{align}
\begin{align}
\nabla^k\ell_{\nu}^{\parallel}:=P_{TM}\circ \nabla^k\ell\circ\nu_{1:q}^{\odot (k-1)}\in\Hom\big(V^1_x\times(\RB^q)^{\odot(k-1)},TM|_{M\setminus S}\big)
\end{align}
and
\begin{align}
\nabla^k\ell_{\nu}^{\perp}:=\nu_{1:q}^T\circ\nabla^k\ell\circ \nu_{1:q}^{\odot(k-1)}\in\Hom\big(V^1_x\times(\RB^q)^{\odot(k-1)},(M\setminus S)\times\RB^q\big).
\end{align}
The following lemma gives the expression for $\GD$ in tubular neighbourhood coordinates associated to the local frame $\nu$.

\begin{proposition}\label{prop:tubularcoords}
    Consider the embedding $\Phi^1:\RB\times V^1_x\times V^1_y\rightarrow\RB^p$ defined by 
    \begin{align}
        \Phi_1(\eta,x,y):=(\eta,E(x,y)),\qquad (\eta,x,y)\in\RB\times V.
    \end{align}
    Let $U^1:=U^1_{\eta}\times U^1_x\times U^1_y$ be a product neighbourhood of $(2/\lambda_1|_F,\bar{x},0)$ such that $\GD\circ \Phi^1(U^1)\subset \Phi^1(\RB\times V^1_x\times V^1_y)$. Then one has
    \begin{align}
    \big((\Phi_1)^{-1}\circ\GD\circ \Phi^1\big)&(\eta,x,y) =: \big(\eta, \GD_x(\eta,x,y),\GD_y(\eta,x,y)\big)\\ &= \begin{pmatrix}\eta\\
        x - \frac{\eta y_1^2}{2}\nabla_M\lambda_1(x) + O\big(|y_1|^3d_M(x,F),|y_1|\|y_{2:q}\|\big)\\T_1(\eta,x)[y]+T_2(\eta,x)[ y^{\otimes 2}] + T_3(\eta,x)[ y^{\otimes 3}] + O(\|y\|^4)
    \end{pmatrix},
    \end{align}
    for all $(\eta,x,y)\in U^1$, where
    \begin{align}
    T_1(\eta,\cdot) = I-\eta\nabla^2\ell_{\nu}^{\perp},\qquad T_2(\eta,\cdot) = -\frac{\eta}{2}\nabla^3\ell_{\nu}^{\perp},
    \end{align}
    and
    \begin{align}
    T_3(\eta,\cdot) = -\frac{\eta}{6}\nabla^4\ell_{\nu}^{\perp}-\eta\,\big(\nu^T\circ D_M\nu\circ(\nabla^3\ell_{\nu}^{\parallel}\otimes T_1)\big)
    \end{align}
\end{proposition}

\begin{proof}
    For $(\eta,x,y)\in U^1$, one has
    \begin{align}
    E^{-1}\GD(\eta,E(x,y)) &= E^{-1}\big(x + \nu(x)y-\eta\nabla\ell(x + \nu(x)y)\big).
    \end{align}
    We therefore apply Lemma \ref{lem:Einv} with $v = \nu(x)y-\eta\nabla\ell(x + \nu(x)y)$. We begin by computing $P_{TM}(x)[v]$, for which we first make the observation that by Assumption \ref{ass:morsebott}, for any $x_*\in F$ sufficiently close to $\bar{x}$, one has $P_{TM}(x_*)\big[\nabla\ell(x_* + y_1\nu_1(x_*))\big] = 0$ for all $y_1\in\RB$ sufficiently small, from which it follows by differentiability of $P_{TM}$, $\nabla\ell$ and $\nu_1$ that
    \begin{align}
    P_{TM}(x)\big[\nabla\ell(x + y_1\nu_1(x))\big] = O(d_M(x,F))
    \end{align}
    as $x\rightarrow F$. One then computes
    \begin{align}
        v_x = P_{TM}(x)[v] &= -\eta\, P_{TM}(x)\big[\nabla\ell(x + \nu(x)y)\big]\\&=-\eta \,P_{TM}(x)\big[\nabla\ell(x+y_1\nu_1(x))\big] + O(|y_1|\|y_{2:q}\|)\\&=-\frac{\eta y_1^2}{2}P_{TM}(x)\nabla^3\ell[\nu_1^{\odot 2}](x) + O\big(|y_1|^3d_M(x,F),|y_1|\|z_{2:q}\|\big)\\&=-\frac{\eta y_1^2}{2}\nabla_M\lambda_1(x) + O\big(|y_1|^3d_M(x,F),|y_1|\|y_{2:q}\|\big)
    \end{align}
    uniformly over $U^1$ as $y\rightarrow 0$ and $x\rightarrow F$, where on the second line we have used the fact that $\nabla\ell|_M\equiv 0$ and on the final line we have used the fact that $P_{TM}\nabla^3\ell[\nu_1^{\odot 2}] = P_{TM}\nabla\lambda_1 = \nabla_M\lambda_1$. On the other hand, letting $P_{\nu M}:T\RB^p|_M\rightarrow \nu M$ denote the projection, one has
    \begin{align}
    v_y = P_{\nu M}(x)[v] = (I-\eta\nabla^2\ell^{\perp}_{\nu}(x)y-\frac{\eta}{2}\nabla^3\ell^{\perp}_{\nu}(x)[y^{\odot 2}]-\frac{\eta}{6}\nabla^4\ell^{\perp}_{\nu}(x)[y^{\odot 3}] + O(\|y\|^4)
    \end{align}
    uniformly over $U^1$ as $y\rightarrow 0$ by Taylor expansion. Finally, noting the coarser estimates
    \begin{align}
    \|v_x\| = O\big(y_1^2d_M(x,F),|y_1|\|y_{2:q}\|\big),\qquad \|v_y\| = O\big(\|y\|\big)
    \end{align}
    and applying Lemma \ref{lem:Einv} gives the result.
\end{proof}

Our next coordinate transformation requires the following centre manifold and stable foliation, which are discrete-time versions of those used by Fenichel in \cite{fenichel_singular}; see also \cite{fenichel1974asymptotic, fenichel1977asymptotic, hpsinvariant}. Note here that, in the space $\RB\times\RB^p$, we denote by $\nu_1$ and $\nu_{2:q}$ the local frame fields defined by pulling back via the projection $\RB\times\RB^p\rightarrow\RB^p$.

\begin{lemma}\label{lem:invariantmanifoldandfoliation}
For any $\bar{x}\in F$, there exists:
\begin{enumerate}
\item An open neighbourhood $U$ of $(2/\lambda_1|_F,\bar{x})$ in $\RB\times\RB^p$.
\item A $C^4$ submanifold $\CC\subset U$ which contains $(\RB\times M)\cap U$, is tangent at $(\RB\times M)\cap U$ to  $T(\RB\times M)|_{\CC\cap (\RB\times M)} + \mathrm{span}\{\nu_1|_{\CC\cap (\RB\times M)}\}$, is locally invariant under $\GD$ in the sense that $\GD(\CC)\cap U\subset \CC$, and contains the invariant subspace $\mathrm{span}(\nu_1|_F)$.
\item A $C^4$ foliation $\FC$ of $U$ whose leaves $L_{(\eta,w)}$ are parametrised by the points $(\eta,w)\in\CC$, tangent at $(\RB\times M)\cap U$ to $\mathrm{span}\{\nu_{2:q}|_{(\RB\times M)\cap U}\}$ and satisfy $\GD(L_{(\eta,w)})\subset L_{\GD(\eta,w)}$ for all $(\eta,w)\in \CC\cap (\GD)^{-1}U$.
\end{enumerate} 
\end{lemma}

\begin{proof}
The proof amounts to a verification of the bunching assumptions required for a smooth centre manifold (\cite[Chapter 5]{hpsinvariant}) and stable foliation (\cite[Theorem 6]{fenichel1974asymptotic}) about a local region of the manifold $\RB\times M$ of fixed points.

For notational convenience, set $\alpha(\eta):=\eta\lambda_1|_F-2$. We prove the existence of $\CC$ first. For $x\in M$, denote $\mu_1(\eta,x):=\max_{2\leq j\leq q}|1-(\alpha(\eta)+2)\lambda_1|_F^{-1}\lambda_2(x)|$ and $\mu_2(\eta,x):=|1-(\alpha(\eta)+2)\lambda_1|_F^{-1}\lambda_1(x)|$. Then $\sigma_{\min}(D\GD(\eta,x)|_{\mathrm{span}\{\nu_1\}(x)}) = \mu_2(\eta,x)$ and $\sigma_{\max}(D\GD(\eta,x)|_{\mathrm{span}\{\nu_{2:q}(x)\}}) = \mu_1(\eta,x)$. In particular, $\mu_2(2/\lambda_1|_F,\bar{x}) = 1$ and $\mu_1(2/\lambda_1|_F,\bar{x})<1$ so that $\epsilon:=\max_{i=0,\dots,4}\{\mu_1(2/\lambda_1|_F,\bar{x})\min\{\mu_2(2/\lambda_1|_F,\bar{x},1\}^{-i}\}<1$. By continuity, there is an open neighbourhood of $(2/\lambda_1|_F,\bar{x})$ in $\RB\times\RB^p$ over which
\begin{align}
    \mu_1(\eta,x)\min\{\mu_2(\eta,x),1\}^{-i} < 1-\epsilon/2,\qquad \forall i=0,\dots,4
\end{align}
uniformly, from which the existence of $U$ and $\CC$ follows by a graph transform argument similar to that of \cite[Chapter 5]{hpsinvariant}, working in a space of graphs which vanish on $\mathrm{span}(\nu_1|_F)$.

To prove the existence of the foliation $\FC$, let $E^-$ be a $D\GD$-invariant subbundle of $T(\RB\times\RB^p)|_{\CC}$ such that $T\CC + E^- = T(\RB\times\RB^p)|_{\CC}$ and $E^-|_{(\RB\times M)\cap W} = \mathrm{span}\{\nu_{2:q}|_{(\RB\times M)\cap W}\}$; for $w\in\CC\cap (\GD^{-1})^{-1}W$, denote $\mu_1'(\eta,w):=\sigma_{\max}(D\GD(\eta,w)|_{E^-_{(\eta,w)}})$, $\beta_1(\eta,w):=\sigma_{\min}(D\GD(\eta,w)|_{T_{(\eta,w)}\CC})$ and $\beta_2(\eta,w):=\sigma_{\max}(D\GD(\eta,w)|_{T_{(\eta,w)}\CC})$. Since $\mu_1'(/\lambda_1|_F,\bar{x})<1$ and $\beta_1(2/\lambda_1|_F,\bar{x}) = \beta_2(2/\lambda_1|_F,\bar{x}) = 1$, setting $\epsilon':=\mu_1'(2/\lambda_1|_F,\bar{x})$, shrinking $W$ if necessary one has
\begin{align}
\mu_1'(\eta,w)\beta_2(\eta,w)^j\beta_1(\eta,w)^{-1}<1-\epsilon'/2
\end{align}
for all $j=0,\dots,4$ uniformly over $W\cap\CC$. These bunching conditions (see \cite[Theorem 6]{fenichel1974asymptotic}) suffice to guarantee smoothness of the foliation constructed by a further graph transform argument.
\end{proof}

With $\bar{x}$, $\CC$ and $\FC$ as in Lemma \ref{lem:invariantmanifoldandfoliation}, one can find a product neighbourhood $V^2:=V^2_{\eta}\times V^2_x\times V^2_{y_1}\times V^2_{y_{2:q}}$ of $(2/\lambda_1|_F,\bar{x},0,0)$ in $\RB\times M\times \RB\times\RB^{q-1}$ and $C^4$ function $h:V^2_{\eta}\times V^2_x\times V^2_{y_1}$ such that in the coordinates $\Phi_1$, $\CC$ may be realised as the graph
\begin{align}
    \Phi_1^{-1}(\CC) = \{(\eta,x,y_1,h(\eta,x,y_1)):(\eta,x,y_1)\in V^2_{\eta}\times V^2_x\times V^2_{y_1}\}
\end{align}
in some neighbourhood of $(2/\lambda_1|_F,\bar{x})$. The invariance and tangency conditions amount to the identities
\begin{align}\label{eq:h1}
    h(\eta,x,0) = 0,\quad Dh(\eta,x,0) = 0,\qquad \forall (\eta,x)\in V^2_{\eta}\times V^2_x.
\end{align} 
Furthermore, that $\CC$ contains $\mathrm{span}(\nu_1|_F)$ implies that $h(\eta,x,y_1) = 0$ for all $\eta\in V^2_{\eta},x\in V^2_x\cap F$ and $y_1\in V^2_{y_1}$; thus
\begin{align}\label{eq:h2}
    h(\eta,x,y_1) = O(|y_1|^2d_M(x,F))
\end{align}
uniformly over $V^2_{\eta}\times V^2_x\times V^2_{y_1}$ as $y_1\rightarrow 0$ and $x\rightarrow F$.

One can parametrise the foliation $\FC$ similarly using a $C^4$ vector field $\xi_x:V^2_{\eta}\times V^2_x\times V^2_{y_1}\times V^2_{y_{2:q}}\rightarrow TV^2_x$ and $C^4$ function $\xi_{y_1}:V^2_{\eta}\times V^2_x\times V^2_{y_1}\times V^2_{y_{2:q}}$ such that, in the coordinates $\Phi_1$, the leaf $L_{(\eta,x,y_1,h(\eta,x,y_1))}$ through $(\eta,x,y_1,h(\eta,x,y_1))\in \CC$ is given by
\begin{align}
    L&_{(\eta,x,y_1,h(\eta,x,y_1))} \\&= \{(\eta,\exp^M_x(\xi_x(\eta,x,y_1,y_{2:q})),y_1 + \xi_{y_1}(\eta,x,y_1,y_{2:q}),y_{2:q} + h(\eta,x,y_1)):y_{2:q}\in V^2_{y_{2:q}}\}.
\end{align}
The condition that $L_{(\eta,x,y_1,h(\eta,x,y_1))}$ pass through $(\eta,x,y_1,h(\eta,x,y_1))$ and the tangency condition for $\FC$ amount to
\begin{align}\label{eq:xi}
    \xi_x(\eta,x,y_1,0)= 0 &\qquad \xi_{y_1}(\eta,x,y_1,0)= 0\\ D_{y_{2:q}}\xi_x(\eta,x,0,0)= 0 & \qquad D_{y_{2:q}}\xi_{y_1}(\eta,x,0,0) = 0
\end{align}
for all $(\eta,x,y_1)\in V^2_{\eta}\times V^2_x\times V^2_{y_1}$. We thereby obtain a $C^4$ coordinate transformation $\Phi^2:V^2:=V^2_{\eta}\times V^2_x\times V^2_{y_1}\times V^2_{y_{2:q}}\rightarrow V^1$ defined by
\begin{align}
    \Phi^2&(\eta,x,y_1,y_{2:q})\\&:=(\eta,\exp^M_x(\xi_x(\eta,x,y_1,y_{2:q})),y_1 + \xi_{y_1}(\eta,x,y_1,y_{2:q}),y_{2:q} + h(\eta,x,y_1))
\end{align}
which is a foliated chart for $\FC$ with transversal $\CC$. Our next lemma expresses $\GD$ in these foliated chart coordinates. For notational convenience, we write
\begin{align}
    y_h(\eta,x,y_1):=(y_1,h(\eta,x,y_1))\in\RB^q,\qquad \forall (\eta,x,y_1)\in V^2_{\eta}\times V^2_x\times V^2_{y_1},
\end{align}
and note that
\begin{align}
    e_1^T T_1(\eta,x)e_1 = (1-\eta\lambda_1(x)),\qquad e_{2:q}^T T_1(\eta,x)e_{2:q} = (I_{q-1}-\eta\Lambda_{2:q}(x)),
\end{align}
where
\begin{align}
    \Lambda_{2:q}(x):=e_{2:q}^T\nabla^2\ell^{\perp}_{\nu}(x)e_{2:q}
\end{align}
and $e_1\in\RB^{q\times 1}$ and $e_{2:q}\in\RB^{q\times(q-1)}$ are the partial identity matrices such that $(e_1,e_{2:q})= I_q$.

\begin{proposition}
    On a sufficiently small product neighbourhood $U^2:=U^2_{\eta}\times U^2_x\times U^2_{y_1}\times U^2_{y_{2:q}}$ of $(2/\lambda_1|_F,\bar{x},0,0)$ in $\RB\times M\times\RB\times\RB^{q-1}$, one has
    \begin{align}
    (\Phi^1\circ\Phi^2)^{-1}\circ \GD\circ (\Phi^1\circ\Phi^2) = (\eta,\GD_x,\GD_{y_1},\GD_{y_{2:q}})
    \end{align}
    where $\GD_x,\GD_{y_1}$ are independent of $y_{2:q}$ and
    \begin{align}
    \GD_x(\eta,x,y_1) &= x-\frac{\eta y_1^2}{2}\nabla_M\lambda_1(x) + O(y_1^3d_M(x,F))\\\GD_{y_1}(\eta,x,y_1) &= (1-\eta\lambda_1(x))y_1+ e_1^T(T_2(\eta,x)[y_h^{\otimes 2}]+ T_3(\eta,x)[y_h^{\otimes 3}]) + O(y_1^4)\\\GD_{y_{2:q}}(\eta,x,y_1,y_{2:q})&= (I_{q-1}-\eta\Lambda_{2:q}(x))y_{2:q} + O(|y_1|\|y_{2:q}\|,\|y_{2:q}\|^2)
    \end{align}
    uniformly as $x\rightarrow F$ and $y = (y_1,y_{2:q})\rightarrow 0$.
\end{proposition}

\begin{proof}
    That $\GD_x$ and $\GD_{y_1}$ are independent of $y_{2:q}$ follows from the invariance of the foliation $\FC$. Specifically, $(\eta,\GD_x,\GD_{y_1})$ is given by the restriction of $\GD$ to the invariant manifold $\CC$. Using Proposition \ref{prop:tubularcoords} one thus sees that
    \begin{align}
    \GD_x(\eta,x,y_1) &= x-\frac{\eta y_1^2}{2}\nabla_M\lambda_1(x) + O(|y_1|^3d_M(x,F),|y_1|\|h(\eta,x,y_1)\|)\\ \GD_{y_1}(\eta,x,y_1) &= e_1^T(T_1(\eta,x)[y_h] + T_2(\eta,x)[y_h^{\otimes 2}] + T_3(\eta,x)[y_h^{\otimes 3}]) + O(\|y_h\|^4)
    \end{align}
    from which the claimed updates follow from \eqref{eq:h2} and the fact that $T_1(\eta,x)$ preserves the decomposition $\RB^q = e_1\RB\oplus e_{2:q}\RB^{q-1}$.
    
    Concerning $\GD_{y_{2:q}}$, the identities \eqref{eq:xi} imply that $D_{y_{2:q}}\Phi^2_y(\eta,x,0,0) = e_{2:q}$, so that
    \begin{align}
    D_{y_{2:q}}\GD_{y_{2:q}}(\eta,x,0,0) = e_{2:q}^TT_1(\eta,x)e_{2:q} = I_{q-1}-\eta\Lambda_{2:q}(x).
    \end{align}
    Since $\Phi^2(\eta,x,y_1,0) = 0$ by \eqref{eq:xi}, it follows that
    \begin{align}
    \GD(\eta,x,y_1,y_{2:q}) = (I_{q-1}-\eta\Lambda_{2:q}(x))y_{2:q} + O(|y_1|\|y_{2:q}\|,\|y_{2:q}\|^2)
    \end{align}
    as claimed.
\end{proof}

Our final coordinate transformation simplifies the formula for $\GD_{y_1}$ to the form seen in Theorem \ref{thm:normalformmain}. For this, it is necessary to compute the second derivative of $h$ in the following lemma.

\begin{lemma}
    For any $(\eta,x)\in V^2_{\eta}\times V^2_x$, one has
    \begin{align}
    D^2_{y_1}h(\eta,x,0) = \big((2\lambda_1(x)-\eta\lambda_1(x)^2)I_{q-1}-e_{2:q}^T\nabla^2\ell_{\nu}^{\perp}(x)e_{2:q}\big)^{-1}\,e_{2:q}^T\nabla^3\ell_{\nu}^{\perp}(x)[e_1^{\odot 2}].
    \end{align}
\end{lemma}

\begin{proof}
    Since $\CC$ is invariant, the function $h$ must satisfy the invariance equation
    \begin{align}\label{eq:hinvariant}
    \GD_{y_{2:q}}(\eta,x,y_1,h(\eta,x,y_1)) = h\big(\eta,\GD_x(\eta,x,y_1,h(\eta,x,y_1)),\GD_{y_1}(\eta,x,y_1,h(\eta,x,y_1))\big)
    \end{align}
    for all $(\eta,x,y_1)\in V^2_{\eta}\times V^2_x\times V^2_{y_1}$. The result follows by expanding both sides of \eqref{eq:hinvariant} about $(\eta,x,0)$ and equating coefficients. Denote $H(\eta,x,y_1):=(y_1,h(\eta,x,y_1))$ and observe that $H(\eta,x,0) = 0$, $D_{y_1}H(\eta,x,0) = (1,D_{y_1}h(\eta,x,0))^T = (1,0)^T$ and $D_{y_1}^2H(x,0) = (0,D^2_{y_1}h(x,0))^T$. Then the right hand side has Taylor expansion
    \begin{align}
        \GD_{y_{2:q}}&(\eta,x,y_1,h(\eta,x,y_1))\\=&\GD_{y_{2:q}}(\eta,x,0,0) + D_{y_{1:q}}\GD_{y_{2:q}}(\eta,x,0,0)D_{y_1}H(\eta,x,0)y_1 +\\&+(1/2)D^2_{y_{1:q}}\GD_{y_{2:q}}(\eta,x,0,0)[D_{y_1}H(\eta,x,0),D_{y_1}H(\eta,x,0)] y_1^2+\\& + (1/2)D_{y_{1:q}}\GD_{2:q}(\eta,x,0,0)D^2_{y_1}H(\eta,x,0) y_1^2 + O( y_1^3)\\=&\frac{1}{2}\big(e_{2:q}^T(I_q-\eta\nabla^2\ell_{\nu}^{\perp}(x))e_{2:q}\,D^2_{y_1}h(\eta,x,0) - \eta\,e_{2:q}^T\nabla^3\ell_{\nu}^{\perp}(x)[e_1,e_1]\big) y_1^2 + O\big( y_1^3\big).
    \end{align}
    Since $\GD_x(\eta,x,0) = x$, $\GD_{y_1}(\eta,x,0) = 0$, using $h|_{V^2_{\eta}\times V^2_x\times\{0\}}\equiv 0$ and $Dh|_{V^2_{\eta}\times V^2_x\times\{0\}}\equiv 0$ one sees that the left hand side has the Taylor expansion
    \begin{align}
        h&\big(\eta,\GD_x(\eta,x,H(x,y_1)),\GD_{y_1}(\eta,x,H(x,y_1))\big)\\=&(1/2)D^2_xh(\eta,x,0)\big[D_{y_{1:q}}\GD_x(\eta,x,0)\,D_{y_1}H(\eta,x,0)[y_1]^{\odot 2}\big]+\\&+D^2_{x,y_1}h(\eta,x,0)\big[D_{y_{1:q}}\GD_x(\eta,x,0)\,D_{y_1}H(\eta,x,0)[y_1],D_{y_{1:q}}\GD_{y_1}(\eta,x,0)\,D_{y_1}H(\eta,x,0)[y_1]\big]+\\&+(1/2)D^2_{y_1}h(\eta,x,0)\big[D_{y_{1:q}}\GD_{y_1}(\eta,x,0)\,D_{y_1}H(\eta,x,0)[y_1]^{\odot 2}\big] + O\big( y_1^3\big)\\=&\frac{1}{2}(1-\eta\lambda_1(x))^2\,D^2_{y_1}h(\eta,x,0)\, y_1^2
    \end{align}
    since $D_{y_{1:q}}\GD_x(\eta,x,0) = 0$. Equating coefficients, therefore, one has
    \begin{align}
    \bigg(e_{2:q}^T(I_q-\eta\nabla^2\ell_{\nu}^{\perp}(x))e_{2:q}\,D^2_{z_1}h(\eta,x,0) - \eta\,e_{2:q}^T\nabla^3\ell_{\nu}^{\perp}(x)[e_1^{\odot 2}]\bigg) = (1-\eta\lambda_1(x))^2\,D^2_{z_1}h(\eta,x,0),
    \end{align}
    which may be rearranged to obtain the claimed formula.
\end{proof}

We now come to simplifying the dynamics along the centre manifold via another change of coordinates. It is this step which requires the solution of a singular PDE, and in which the proof of the corresponding result in \cite{macdonaldeos} is erroneous since it cites a classical algebraic argument for the \emph{isolated} $y_1$-system \cite[Theorem 4.3]{kuznetsov} (in which $x$ is assumed to be fixed) without accounting for the fact that the $y_1$-component of the system is actually \emph{coupled} to the $x$-component of the system. As we demonstrate below, properly accounting for this coupling is what necessitates the solution of a PDE.

\begin{theorem}\label{thm:normalformappendix}
    There is a product neighbourhoood $V^3_{\eta}\times V^3_x\times V^3_{y_1}\subset V^2_{\eta}\times V^2_x\times V^2_{y_1}$ of $(2/\lambda_1|_F,\bar{x})$ in $\RB\times M$ and $C^1$ functions $\psi_1,\psi_2:V^3_{\eta}\times V^3_x\rightarrow\RB$, with $(\eta/2)\psi_1(\eta,\cdot)|_F^2 = \alpha_{\eta}^{-1}$ from Assumption \ref{ass:quality}, such that the map $\Phi^3:V^3_{\eta}\times V^3_x\times V^3_{y_1}\times V^2_{y_{2:q}}\rightarrow V^2_{\eta}\times V^2_{x}\times V^2_{y_1}\times V^2_{y_{2:q}}$ defined by
    \begin{align}
    \Phi^3(\eta,x,y_1,y_{2:q}) = (\eta,x,\psi_1(\eta,x)y_1 + (1/2)\psi_2(\eta,x)y_1^2,y_{2:q})
    \end{align}
    satisfies $(\Phi^1\circ\Phi^2\circ\Phi^3)^{-1}\circ\GD\circ(\Phi^1\circ\Phi^2\circ\Phi^3) = (\eta,\GD_x,\GD_{y_1},\GD_{y_{2:q}})$ on some product neighbourhood $U^3_{\eta}\times U^3_x\times U^3_{y_1}\times U^3_{y_{2:q}}\subset V^3_{\eta}\times V^3_x\times V^3_{y_1}\times V^2_{y_{2:q}}$, where $\GD_x,\GD_{y_1}$ are independent of $y_{2:q}$ and
    \begin{align}
    \GD_x(\eta,x,y_1) &= x-\frac{\eta\psi_1(\eta,x)^2y_1^2}{2}\nabla_M\lambda_1(x) + O(y_1^3d_M(x,F))\\\GD_{y_1}(\eta,x,y_1) &= (1-\eta\lambda_1(x))y_1+ y_1^3 + O(y_1^4)\\\GD_{y_{2:q}}(\eta,x,y_1,y_{2:q})&= (I_{q-1}-\eta\Lambda_{2:q}(x))y_{2:q} + O(|y_1|\|y_{2:q}\|,\|y_{2:q}\|^2)
    \end{align}
    for all $(\eta,x,y_1,y_{2:q})\in U^3_{\eta}\times U^3_x\times U^3_{y_1}\times U^3_{y_{2:q}}$.
\end{theorem}

\begin{proof}
    We begin by assuming a general form for $\psi$ in $y_1$; it will later be shown that the higher order terms can be taken to be zero. Write
    \begin{align}
    \psi(\eta,x,y_1):=\psi_1(\eta,x)y_1 + \frac{1}{2}\psi_2(\eta,x) y_1^2 + \frac{1}{6}\psi_3(\eta,x) y_1^3 + O\big( y_1^4\big).
    \end{align}
    From the desired identity
    \begin{align}
    &\GD_{y_1}\big(\eta,x,\psi(\eta,x,y_1),h(x,\psi(\eta,x,y_1))\big)\nonumber\\& = \psi\big(\eta,x - \frac{\eta\psi(\eta,x,y_1)^2}{2}\nabla_M\lambda_1(x) + O\big( y_1^3\big),(1-\eta\lambda_1(x))y_1 +  y_1^3 + O\big( y_1^4\big)\big)
    \end{align}
    we derive expressions for $\psi_i(\eta,x)$, $i=1,2,3$, by Taylor-expanding both sides and equating coefficients. To expand the left hand side, first note that
    \begin{align}
        h(\eta,x,\psi(\eta,x,y_1))= \frac{1}{2}D^2_{y_1}h(\eta,x,0)\psi_1(\eta,x)^2 y_1^2+ O\big( y_1^3\big)
    \end{align}
    Then:
    \begin{align}
        \GD_{y_1}&\big(\eta,x,\psi(\eta,x,y_1),h(\eta,x,\psi(\eta,x,y_1))\big)\\=&e_1^TT_1(x)\big[e_1\psi(\eta,x,y_1) + e_{2:q}h(\eta,x,\psi(\eta,x,y_1))\big]\\&+e_1^TT_2(x)\big[\big(e_1\psi(\eta,x,y_1) + e_{2:q}h(\eta,x,\psi(\eta,x,y_1))\big)^{\otimes 2}\big]\\&+e_1^TT_3(x)\big[\big(e_1\psi(\eta,x,y_1) + e_{2:q}h(\eta,x,\psi(\eta,x,y_1))\big)^{\otimes 3}\big] + O\big( y_1^4\big)\\=&(1-\eta\lambda_1(x))\psi_1(\eta,x)y_1\\&+\frac{1}{2}\big[(1-\eta\lambda_1(x))\psi_2(\eta,x)-\eta\,\psi_1(\eta,x)^2\,e_1^T\nabla^3\ell_{\nu}^{\perp}(x)[e_1^{\otimes 2}]\big] y_1^2\\&+\frac{1}{6}\big[(1-\eta\lambda_1(x))\psi_3(\eta,x) - 3\eta\,\psi_1(\eta,x)^3\,e_1^T\nabla^3\ell^{\perp}_{\nu}(x)[e_1,e_{2:q}D^2_{y_1}h(\eta,x,0)]]\\&-3\eta\,\psi_1(\eta,x)\psi_2(\eta,x)e_1^T\nabla^3\ell^{\perp}_{\nu}(x)[e_1^{\otimes 2}]-\eta\psi_1(\eta,x)^3\,e_1^T\nabla^4\ell^{\perp}_{\nu}(x)[e_1^{\otimes 3}]\\&-6\eta\,\psi_1(\eta,x)^3 e_1^T\big(\nu^T\circ D_M\nu\circ (\nabla^3\ell^{\parallel}_{\nu}\otimes T_1)\big)(x)[e_1^{\otimes3}]\big] y_1^3 + O\big( y_1^4\big).
    \end{align}
    Note that the matrix-valued function $\nu^T\circ D_M\nu\circ [w,\cdot]$ is antisymmetric for any tangent field $w$ on $TM$ since $\nu$ is an orthonormal frame; hence
    \begin{align}
    e_1^T&\big(\nu ^T\circ D_M\nu \circ(\nabla^3\ell^{\parallel}_{\nu}\otimes T_1)\big)[e_1^{\otimes 3}]\\&=(1-\eta\lambda_1(x))e_1^T\big(\nu^T\circ D_M\nu\circ [\nabla^2\ell^{\parallel}_{\nu}[e_1^{\otimes 2}],e_1]\big) = 0.
    \end{align}
    Thus
    \begin{align}
        &\GD_{y_1}\big(\eta,x,\psi(\eta,x,y_1),h(\eta,x,\psi(\eta,x,y_1))\big)\\&=(1-\eta\lambda_1(x))\psi_1(\eta,x)y_1\\&+\frac{1}{2}\big[(1-\eta\lambda_1(x))\psi_2(\eta,x)-\eta\,\psi_1(\eta,x)^2\,e_1^T\nabla^3\ell_{\nu}^{\perp}(x)[e_1^{\otimes 2}]\big] y_1^2\\&+\frac{1}{6}\big[(1-\eta\lambda_1(x))\psi_3(\eta,x) - 3\eta\,\psi_1(\eta,x)^3\,e_1^T\nabla^3\ell^{\perp}_{\nu}(x)[e_1,e_{2:q}D^2_{y_1}h(\eta,x,0)]\\&-3\eta\,\psi_1(\eta,x)\psi_2(\eta,x)e_1^T\nabla^3\ell^{\perp}_{\nu}(x)[e_1^{\otimes 2}]-\eta\psi_1(\eta,x)^3\,e_1^T\nabla^4\ell^{\perp}_{\nu}(x)[e_1^{\otimes 3}]\big] y_1^3 + O\big( y_1^4\big).
    \end{align}
    On the other hand,
    \begin{align}
        \psi&\bigg(\eta,x-\eta\frac{\psi(\eta,x,y_1)^2}{2}\nabla_M\lambda_1(x) + O\big( y_1^3\big),(1-\eta\lambda_1(x))y_1 +  y_1^3 + O\big( y_1^4\big)\bigg)\\=&\psi_1(\eta,x)\big((1-\eta\lambda_1(x))y_1+  y_1^3\big) - \frac{\eta}{2}\langle\nabla_M\psi_1(\eta,x),\nabla_M\lambda_1(x)\rangle(1-\eta\lambda_1(x))\psi_1(\eta,x)^2 y_1^3\\&+\frac{1}{2}\psi_2(\eta,x)(1-\eta\lambda_1(x))^2 y_1^2 + \frac{1}{6}\psi_3(\eta,x)(1-\eta\lambda_1(x))^3 y_1^3\\=&(1-\eta\lambda_1(x))\psi_1(\eta,x)y_1\\&+\frac{1}{2}(1-\eta\lambda_1(x))^2\psi_2(\eta,x) y_1^2\\&+\frac{1}{6}\bigg[6\psi_1(\eta,x) -6\eta\langle\nabla_M\psi_1(\eta,x),\nabla_M\lambda_1(x)\rangle(1-\eta\lambda_1(x))\psi_1(\eta,x)^2 + \psi_3(\eta,x)(1-\eta\lambda_1(x))^3\bigg] y_1^3.
    \end{align}
    For notational convenience, let us from now on suppress function evaluation, with evaluation assumed from the above. Equating coefficients of $y_1$ gives the trivial equation $(1-\eta\lambda_1)\psi_1 = (1-\eta\lambda_1)\psi_1$. Equating coefficients of $ y_1^2$ gives
    \begin{align}
    (1-\eta\lambda_1)\psi_2 - \eta\psi_1^2\,\nabla^3\ell^{\perp}_{\nu}[e_1^{\otimes 3}] = (1-\eta\lambda_1)^2\psi_2,
    \end{align}
    from which it follows that
    \begin{align}
    \psi_2 = \lambda_1^{-1}(1-\eta\lambda_1)^{-1}\psi_1^2\,\nabla^3\ell^{\perp}_{\nu}[e_1^{\otimes 3}].
    \end{align}
    Substituting this and equating coefficients of $ y_1^3$ gives
    \begin{align}
    (1&-\eta\lambda_1)\psi_3-3\eta\psi_1^3\bigg(e_1^T\nabla^3\ell^{\perp}_{\nu}[e_1,e_{2:q}D^2_{y_1}h]+ \lambda_1^{-1}(1-\eta\lambda_1)^{-1}\big(\nabla^3\ell^{\perp}_{\nu}[e_1^{\otimes 2}]\big)^2\bigg) - \eta\psi_1^3e_1^T\nabla^4\ell^{\perp}_{\nu}[e_1^{\otimes 3}]\\&=6\psi_1-3\eta\langle\nabla_M\psi_1,\nabla_M\lambda_1\rangle(1-\eta\lambda_1)\psi_1^2 + \psi_3(1-\eta\lambda_1)^3,
    \end{align}
    which gives
    \begin{align}
        \psi_3&(1-\eta\lambda_1)(1-(1-\eta\lambda_1)^2) \\&= \psi_1\bigg(6-3\eta\langle\nabla_M\psi_1,\nabla_M\lambda_1\rangle(1-\eta\lambda_1)\psi_1\\&+3\eta\psi_1^2\bigg(e_1^T\nabla^3\ell^{\perp}_{\nu}[e_1,e_{2:q}D^2_{y_1}h]+ \lambda_1^{-1}(1-\eta\lambda_1)^{-1}\big(\nabla^3\ell^{\perp}_{\nu}[e_1^{\otimes 3}]\big)^2\bigg)\\&+\eta\psi_1^2e_1^T\nabla^4\ell^{\perp}_{\nu}[e_1^{\otimes 3}]\bigg).
    \end{align}
    Finally, substituting
    \begin{align}
    D^2_{y_1}h = ((2\lambda_1-\eta\lambda_1^2)I_{2:q}-e_{2:q}^T\nabla^2\ell^{\perp}_{\nu}e_{2:q})^{-1}e_{2:q}^T\nabla^3\ell^{\perp}_{\nu}[e_1^{\otimes2}],
    \end{align}
    one obtains
    \begin{align}
        \psi_3&(1-\eta\lambda_1)(1-(1-\eta\lambda_1)^2) \\&= \psi_1\bigg(6-3\eta\langle\nabla_M\psi_1,\nabla_M\lambda_1\rangle(1-\eta\lambda_1)\psi_1\\&+3\eta\psi_1^2\,\nabla^3\ell^{\perp}_1(x)\big[e_1,((2\lambda_1-\eta\lambda_1^2)I_{2:q}-\nabla^2\ell^{\perp}_{2:q})^{-1}\nabla^3\ell^{\perp}_{2:q}[e_1^{\odot2}] + \lambda_1^{-1}(1-\eta\lambda_1)^{-1}\nabla^3\ell^{\perp}_1[e_1^{\odot2}]\big]\\&+\eta\psi_1^2e_1^T\nabla^4\ell^{\perp}_{\nu}[e_1^{\otimes 3}]\bigg).
    \end{align}
    One sees then that it suffices for $\psi_1$ to solve the partial differential equation
    \begin{align}
        \langle\nabla_M\psi_1,\nabla_M\lambda_1\rangle\psi_1 = (1-\eta\lambda_1)^{-1}\bigg(\nabla^3\ell^{\perp}_1\big[e_1,A\nabla^3\ell^{\perp}_{\nu}[e_1^{\odot 2}]\big] + \frac{1}{3}\nabla^4\ell^{\perp}_1[e_1^{\odot 3}]\bigg)\psi_1^2 + \frac{2}{\eta(1-\eta\lambda_1)},
    \end{align}
    where
    \begin{align}
    A:=e_{2:q}((2\lambda_1-\eta\lambda_1^2)I_{2:q}-\nabla^2\ell^{\perp}_{2:q})^{-1}e_{2:q}^T + \lambda_1^{-1}(1-\eta\lambda_1)^{-1}e_1e_1^T,
    \end{align}
    in which case one can take $\psi_3$ (and all higher order terms) to be zero. Setting $\phi:=\frac{1}{2}\psi_1^2$, it thus suffices to solve the PDE
    \begin{align}
        \langle\nabla_M\phi,\nabla_M\lambda_1\rangle = \frac{2}{1-\eta\lambda_1}\bigg(\nabla^3\ell_1\big[e_1,A\nabla^3\ell^{\perp}_{\nu}[e_1^{\odot 2}]\big] + \frac{1}{3}\nabla^4\ell^{\perp}_1[e_1^{\odot 3}]\bigg)u + \frac{2}{\eta(1-\eta\lambda_1)}.
    \end{align}
    That this PDE admits a $C^1$ solution which is positive on some $V^3_{\eta}\times V^3_x$ is a consequence of Assumption \ref{ass:quality} and Theorem \ref{thm:pdesolve}. Note in particular that along $F$, the left side of this PDE is zero so that the solution $\phi = (1/2)\psi_1^2$ must satisfy
    \begin{align}
    \phi|_F = -\bigg(\eta\bigg(\nabla^3\ell_1[e_1,A\nabla^3\ell^{\perp}_{\nu}[e_1^{\odot 2}]] + \frac{1}{3}\nabla^4\ell_1^{\perp}[e_1^{\odot 3}]\bigg)\bigg)^{-1}.
    \end{align}
    Now, setting $\zeta_{\eta}:=(1/2)\eta\psi_1(\eta,\cdot)^2$, one sees that this function restricts to $F$ to yield
    \begin{align}
    \zeta_{\eta} =\eta\bigg(-\eta \bigg(\nabla^3\ell_1[e_1,A\nabla^3\ell^{\perp}_{\nu}[e_1^{\odot 2}]] + \frac{1}{3}\nabla^4\ell_1^{\perp}[e_1^{\odot 3}]\bigg)\bigg)^{-1} = \alpha_{\eta}^{-1}
    \end{align}
    as claimed.
\end{proof}

\section{Singular PDE}\label{sec:singularpde}

In this section, we consider a singular partial differential equation of the form
\begin{align}\label{eq:singularpde}
X\phi = \alpha \phi +\beta
\end{align}
on $\RB^l\times\RB^m\times\RB^n$, with $X$, $\alpha$ and $\beta$ all $C^{\infty}$. Writing $(\eta,z,u)$ for the coordinates in $\RB^l\times\RB^m\times\RB^n$, we assume that
\begin{align}
X(\eta,z,u) = f(\eta,z,u)\partial_z + g(\eta,z,u)\partial_u,
\end{align}
where
\begin{align}
f(\eta,z,u) = O(\|u\|^2),\qquad g(\eta,z,u) = A(\eta,z)u + O(\|u\|^2)
\end{align}
with $A(\eta,z)$ being a $C^{\infty}$ field of positive definite matrices. We also assume that
\begin{align}
\alpha(\eta,z,0)>0,\qquad D_u\alpha(\eta,z,0)=D_u\beta(\eta,z,0) = 0
\end{align}
for all $(\eta,z)$. The singularity of the PDE manifests in that $X$ is zero along $u=0$, making the classical method of characteristics difficult to apply.

PDE which are singular at a \emph{point} (corresponding to $l=m=0$) are well-studied and can be solved via classical series expansion techniques, introducing logarithms to account for resonances where the linear equations defining the series recursion cannot be inverted \cite{frobeniusmethod, tahara1, tahara2, singularpde}, and resulting in solutions of low regularity. These techniques do not appear to extend straightforwardly to our setting, where we have a \emph{manifold} of singularities and where resonances can vary along this manifold. We instead exploit the linearity of our PDE to use dynamical systems techniques in proving that \eqref{eq:singularpde} admits $C^1$ solutions.

\begin{theorem}\label{thm:pdesolve}
Any point $(\eta_0,z_0,0)$ admits a neighbourhood over which \eqref{eq:singularpde} admits a $C^1$ solution.
\end{theorem}

\begin{proof}
The result essentially follows by reducing the PDE to an invariant graph problem, which is then solved by a classical linearisation argument \cite{belitskii1978}. We begin by changing to more convenient coordinates.

By the smooth strong stable foliation theorem \cite[Theorem 1.4]{dutilleul2018}, there is a neighbourhood of $(\eta_0,z_0,0)$ which is smoothly and invariantly foliated by the strong stable leaves of $-X$. Specifically, there is a smooth function $\xi(\eta,z,u) = z + O(\|u\|^2)$ whose level sets are the strong stable leaves of $-X$; consequently, $X\xi\equiv 0$. The invariant base coordinate $p(\eta,z,u) = (\eta,\xi(\eta,z,u))\in\RB^{l+m}$ is annihilated by $X$, so that in the coordinates $(\eta,z,u)\mapsto (p(\eta,z,u),u)$, the vector field $X$ takes the form
\begin{align}\label{eq:Xeqns}
X(p,u) = F(p,u)\partial_u,\qquad F(p,0) = 0,\quad D_uF(p,0) = A(p)
\end{align}
Thus the problem is reduced to that of solving a family of fibrewise PDE.

We now solve the PDE to zeroth order. Specifically, define
\begin{align}
a(p):=\alpha(p,0),\qquad \phi_0(p):=-\frac{\beta(p,0)}{\alpha(p,0)},\qquad \gamma(p,u):=\beta(p,u) +\alpha(p,u)\phi_0(p)
\end{align}
and note that
\begin{align}\label{eq:gammaeqns}
\gamma(p,0) = 0,\qquad D_u\gamma(p,0) = 0.
\end{align}
Since $X\phi_0\equiv 0$, a function $\phi:=\phi_0 + h$ solves the PDE if and only if $h$ solves
\begin{align}\label{eq:heqns}
D_uh(p,u)F(p,u) = \alpha(p,u)h(p,u) + \gamma(p,u),
\end{align}
if and only if the graph $\{(p,u,h(p,u)):(p,u)\text{ near $(\eta_0,x_0,0)$}\}$ of $h$ is invariant under the fibrewise vector field $Y_p$ on $\times\RB^n\times\RB$ defined by
\begin{align}
Y_p(u,v):=\big(F(p,u),\alpha(p,u)v + \gamma(p,u)).
\end{align}
By \eqref{eq:Xeqns} and \eqref{eq:gammaeqns}, the vector field $Y_p$ linearises to
\begin{align}
L_p:=DY_p(0,0) = \begin{pmatrix} A(p)& 0 \\ 0 & a(p)\end{pmatrix}
\end{align}
along $(u,v) = 0$. Since $a(p)>0$ and $A(p)\succ 0$ by hypothesis, all eigenvalues of $L_p$ are positive. Letting $\Theta^p_{s}$ denote the time $s$ flow of $Y$ on the fibre over $p$, denote $T_p:=\Theta^p_{-1}$. This $T_p$ is a contraction, and its linearisation $B_p:=DT_p = e^{-L_p}$ has all eigenvalues being strictly contained in the open interval $(0,1)$. The $C^1$-linearisation theorem of \cite[Theorem 2.5.7]{belitskii1978} then says that, shrinking the neighbourhood under consideration if necessary, there exists a $C^1$ family of local diffeomorphisms $H_p$ satisfying
\begin{align}
H_p\circ T_p= B_p\circ H_p,\qquad H_p(0) = 0,\qquad DH_p(0) = I_{n+1}.
\end{align}
We now apply Sternberg's averaging trick \cite{sternberg1957} to produce a linearisation of time-varying flow: the map
\begin{align}
K_p(u,v):=\int_{-1}^0e^{-tL_p}H_p(\Theta^p_t(u,v))\,dt
\end{align}
is easily verified to satisfy
\begin{align}
K_p(\Theta^p_s(u,v)) = e^{sL_p}K_p(u,v),
\end{align}
from which it follows that the smoothly-varying $B_p$-invariant subspaces $p\mapsto E_p:=\RB^n\times\{0\}$ are mapped to a $C^1$ family of $\Theta^p_s$-invariant submanifolds $W_p:=K_p^{-1}(E_p)$, which fit together into a $C^1$ submanifold $W$. The tangency condition $DK_p(0,0) = I_{n+1}$ then implies, by the implicit function theorem, that $W$ is the graph of some $C^1$ function $h$, which by invariance solves \eqref{eq:heqns} and thereby proves the result.
\end{proof}

\section{Convergence theorems}

In this section, we prove the convergence theorems stated in the main body of the paper. All three theorems require that one can choose invariant neighbourhoods of any point in $F$. This is taken up in Subsection \ref{subsec:invariantneighourhoodappendix}. Following this, we present our proofs of Theorems \ref{thm:subcriticalmain}, \ref{thm:criticalmain} and \ref{thm:supercriticalmain} in Subsections \ref{subsec:subcriticalappendix}, \ref{subsec:criticalappendix} and \ref{subsec:supercriticalappendix} respectively.

\subsection{Existence of invariant neighbourhoods}\label{subsec:invariantneighourhoodappendix}

The purpose of this subsection is to prove the existence of invariant sets in which the analysis may be conducted. The construction of such sets requires the following further transformation of the normal form using Fermi coordinates about the submanifold $F\subset M$.

\begin{lemma}\label{lem:fermicoordinates}
    Let $z_0\in F$ and consider a coordinate neighbourhood in which the normal form of Theorem \ref{thm:normalformappendix} is defined. Let $P_F$ denote the nearest-point projection from a neighbourhood of $F$ in $M$ to $F$ and denote $z = P_F(x)$ for $x$ in this neighbourhood. Then, under the Fermi coordinate change
    \begin{equation}
    u = \log_z(x),\qquad x= \exp^M_z(u)
    \end{equation}
    the normal form of Theorem \ref{thm:normalformappendix} becomes
    \begin{align}
    \GD_z(\eta,z,u,y_1) &= z + O(y_1^3\|u\|,y_1^2\|u\|^2),\\
    \GD_u(\eta,z,u,y_1) &= (I-\zeta(\eta,z)y_1^2\nabla^2\lambda_1(z))u + O(y_1^3\|u\|,y_1^2\|u\|^2),\\
    \GD_{y_1}(\eta,z,u,y_1) &= -y_1(1+(\eta\lambda_1|_F-2) + (\eta/2)D^2_M\lambda_1(z)[u,u] - y_1^2 + O(\|u\|^3,y_1^3)),\\
    \GD_{y_{2:q}}(\eta,z,u,y_1,y_{2:q}) &= (I-\eta\Lambda_{2:q}(z,u))y_{2:q} + O(y_1\|y_{2:q}\|,\|y_{2:q}\|^2).
    \end{align}
\end{lemma}

\begin{proof}
Letting $\Pi_{z\rightarrow x}:T_zM\rightarrow T_xM$ be parallel transport, the Taylor expansions $\nabla_M\lambda_1(x) = \Pi_{z\rightarrow x}\nabla^2_M\lambda_1(z)[u] + O(\|u\|^2)$ and $\zeta(\eta,x) = \zeta(\eta,z) + O(\|u\|)$ about $z = P_F(x)$ give
\begin{align}\label{eq:GDtaylor}
\GD(\eta,x,y_1) = x-\zeta(\eta,z)y_1^2\Pi_{z\rightarrow x}\nabla^2_M\lambda_1(z)[u] + O(y_1^2\|u\|^2,y_1^3\|u\|).
\end{align}
Hence, since $DP_F(z) = P_{T_zF}$ annihilates vectors normal to $F$ and $\nabla^2_M\lambda_1(z)[u]$ is normal to $F$, Taylor expanding $P_F$ gives
\begin{align}
    \GD_z(\eta,z,u,y_1) = P_F(\GD_x(\eta,x,y_1)) = z + O(y_1^2\|u\|^2,y_1^3\|u\|).
\end{align}
On the other hand, again invoking \eqref{eq:GDtaylor} and using the Taylor expansion of the logarithm,
\begin{align}
    \GD_u(\eta,z,u,y_1) &= \log_{\GD_z(\eta,z,u,y_1)}(\GD_x(\eta,x,y_1))\\&= (I-\zeta(\eta,x)y_1^2\nabla^2_M\lambda_1(z))u + O(y_1^3\|u\|,y_1^2\|u\|^2).
\end{align}
Finally, the claimed update formula for $\GD_{y_1}(\eta,z,u,y_1)$ follows from the Taylor expansion $\lambda_1(x) = \lambda_1|_F + (1/2)D^2_M\lambda_1(z)[u,u] + O(\|u\|^3)$.
\end{proof}

We can now prove our invariant set lemma, which is key to all of our convergence theorems.

\begin{lemma}\label{lem:reduction1}
    Given $\bar{x}\in F$, let $U$ be a neighbourhood of $(2/\lambda_1|_F,\bar{x})$ on which the coordinates of Lemma \ref{lem:fermicoordinates} are defined, and let $L>0$ satisfy
    \begin{align}
    \nabla^2_M\lambda_1(z)\preceq L I_{\nu^MF}
    \end{align}
    uniformly over $z\in F\cap U$. Then there are $\Delta,K>0$ such that for all $\rho$ sufficiently small, the neighbourhood
    \begin{align}\label{eq:Vrho}
    V_{\rho,\Delta}:=\{(\eta,z,u,y):\eta\lambda_1|_F-2\leq \rho^2,\, d_F(z,\bar{x}) + K\|u\| + \|y_{2:q}\|\leq\rho,\, 0<|y_1|\leq \sqrt{2(1+L)}\rho\}
    \end{align}
    is invariant under $\GD$ and one has
    \begin{align}\label{eq:2:qbound}
        \|\GD_{y_{2:q}}(\eta,z,u,y)\|\leq (1-\Delta)\|y_{2:q}\|
    \end{align}
    for all $(\eta,z,u,y)\in V_{\rho,\Delta}$.
\end{lemma}

\begin{proof}
Shrinking the coordinate neighbourhood $U$ and $\eta$ if necessary, since $\lambda_1$ is simple, there is $\Delta>0$ such that \eqref{eq:2:qbound} holds over the entire coordinate neighbourhood. By the update formulae of Lemma \ref{lem:fermicoordinates}, there are also constants $C_1,C_2>0$ such that
\begin{align}
    d_F(z,\GD_z(\eta,z,u,y_1))&\leq C_1(|y_1|+\|u\|)y_1^2\|u\|,\\
    \|\GD_u(\eta,z,u,y_1)\|&\leq (1-C_2y_1^2)\|u\|,\\
    |\GD_{y_1}(\eta,z,u,y_1)|&\leq |y_1|(1+(\eta\lambda_1|_F-2) + L\|u\|^2-(1/2)y_1^2)
\end{align}
uniformly over the neighbourhood. Set $K:=\max\{1,C_1/C_2\}$ and let $V_{\rho,\Delta}$ be defined as in \eqref{eq:Vrho}. For all $\rho$ sufficiently small one has $V_{\rho,\Delta}$ contained in $U$ and $(2^{1/2}(1+L)^{1/2}+1)\rho\leq 1$. These conditions, together with the update equations of Lemma \ref{lem:fermicoordinates} suffice to guarantee invariance of $V_{\rho,\Delta}$. Indeed, given $(\eta,z,u,y)\in V_{\rho,\Delta}$ one has
\begin{align}
d_F&(\GD_z(\eta,z,u,y_1),\bar{x}) + K\|\GD_u(\eta,z,u,y_1)\| + \|\GD_{y_{2:q}}(\eta,z,u,y)\|\\&\leq d_F(\GD_z(\eta,z,u,y_1),z) + d_F(z,\bar{x}) + K\|u\|(1-C_2y_1^2) + (1-\Delta)\|y_{2:q}\|\\&\leq y_1^2\|u\|(C_1(|y_1| + \|u\|) - KC_2) + d_F(z,\bar{x}) + K\|u\| + \|y_{2:q}\|\\&\leq d_F(z,\bar{x}) + K\|u\| + \|y_{2:q}\|\leq \rho,
\end{align}
where in the final line we have invoked $|y_1|+\|u\|\leq (2^{1/2}(1+L)^{1/2}+1)\rho\leq 1$ and $KC_2-C_1 = 0$, thus proving that the first estimate defining $V_{\rho,\Delta}$ is preserved by $\GD$. For the second estimate, fixing again $(\eta,z,u,y)\in V_{\rho,\Delta}$ and shrinking $\rho$ if necessary, observe that
\begin{align}
(1/2)|y_1|\leq |\GD_{y_1}(\eta,z,u,y_1)|\leq |y_1|(1+(1+L)\rho^2 - (1/2)y_1^2)
\end{align}
Thus $|y_1|>0$ is preserved. Moreover, the bounding function $t\mapsto t(1+(1+L)\rho^2-(1/2)t^2)$ is increasing on the interval $[0,2^{1/2}(1+L)^{1/2}\rho]$ with maximum value $2^{1/2}(1+L)^{1/2}\rho$ at $t = 2^{1/2}(1+L)^{1/2}\rho$, hence $|\GD_{y_1}(\eta,z,u,y_1)|\leq 2^{1/2}(1+L)^{1/2}\rho$, thus completing the proof that $V_{\rho,\Delta}$ is invariant.
\end{proof}

\subsection{Subcritical regime}\label{subsec:subcriticalappendix}

In this section, we prove Theorem \ref{thm:subcriticalmain}. With Lemma \ref{lem:reduction1} in hand, the result follows from a similar argument to that used in the proof of \cite[Theorem D.4]{macdonaldeos}. We thus first recall the following lemmas from \cite{macdonaldeos}, which will be used without change in this paper. We refer to \cite{macdonaldeos} for their proofs.

\begin{lemma}\cite[Lemma D.2]{macdonaldeos}\label{lem:contraction1}
    For $\alpha\in\RB$, define $f_{\alpha}:\RB\rightarrow\RB$ by
    \begin{align}
        f_{\alpha}(z):=-(1+\alpha)z+z^3+O(z^4).
    \end{align}
    For $\alpha_0,\alpha_1\in\RB$, consider the composite $f_{\alpha_1\alpha_0}:=f_{\alpha_1}\circ f_{\alpha_0}$. Then for all $\gamma>0$ sufficiently small and all $\alpha_0,\alpha_1\in[0,\gamma]$:
    \begin{enumerate}
        \item $f_{\alpha_1\alpha_0}$ is monotonically increasing on $[-2\sqrt{\gamma},2\sqrt{\gamma}]$.
        \item $f_{\alpha_1\alpha_0}$ admits the sole fixed points $0$,
        \begin{align}
            z_- = -\sqrt{\frac{\alpha_0+\alpha_1}{2}} + O(\alpha_0+\alpha_1),\qquad z_+ = \sqrt{\frac{\alpha_0+\alpha_1}{2}} + O(\alpha_0+\alpha_1)
        \end{align}
        in the interval $[-2\sqrt{\gamma},2\sqrt{\gamma}]$.
    \end{enumerate}
\end{lemma}

\begin{lemma}\cite[Lemma D.3]{macdonaldeos}\label{lem:contraction2}
    Let $\{\lambda_t\}_{t\in\NB}$ be a monotonically decreasing sequence of numbers. Then for all $\eta>2/\lambda_0$ sufficiently small and all $z_0\in\RB$ sufficiently close to zero, the iterates
    \begin{align}
    z_{t+1}:=(1-\eta\lambda_t)z_t + z_t^3 + O(z_t^4)
    \end{align}
    satisfy
    \begin{align}
    \frac{|z_0|}{\sqrt{1+3z_0^2t}}\leq |z_t|\leq 2\sqrt{\eta\lambda_0-2}
    \end{align}
    for all $t\in\NB$ such that $\eta\lambda_t\geq 2$
\end{lemma}

Our next lemma generalises the descent lemma \cite[Lemma D.1]{macdonaldeos} from the case where $\lambda_1$ is geodesically strongly convex and $F$ is a single point, to the more general case where $\lambda_1$ is Morse-Bott along the manifold $F$.

\begin{lemma}[Descent lemma for $x$-update on $V_{\rho,\Delta}$]\label{lem:morsebottdescent}
    Let $V_{\rho,\Delta}$ be as in Lemma \ref{lem:reduction1}. Then there is $C>0$ such that
    \begin{align}
    \lambda_1(\GD_x(\eta,x,y))-\lambda_1|_F\leq (1-Cy_1^2)\big(\lambda_1(x)-\lambda_1|_F\big)
    \end{align}
    for all $(\eta,x,y)\in V_{\rho,\Delta}$.
\end{lemma}

\begin{proof}
    That $\lambda_1$ is Morse-Bott implies that $\|\nabla_M\lambda_1(x)\|^2=\Theta(\lambda_1(x)-\lambda_1|_F) = \Theta(d_M(x,F)^2)$ as $x\rightarrow F$. There is then $\delta(y_1,x) = O(|y_1|^3d_M(x,F))\in T_xM$ such that
    \begin{align}
    \GD_x(\eta,x,y) = \exp_x\big(-\zeta(\eta,x)y_1^2\nabla_M\lambda_1(x) + \delta(y_1,x)\big).
    \end{align}
    There is then a constant $C>0$ such that, shrinking $\rho$ if necessary and using a covariant Taylor expansion for $\lambda_1$ around $x$, one has
    \begin{align}
    \lambda_1(\GD_x(\eta,x,y))-\lambda_1|_F&= \lambda_1(x)-\zeta(\eta,x)y_1^2\|\nabla_M\lambda_1(x)\|^2 + O(|y_1|^3d_M(x,F)^2)\\&\leq (1-Cy_1^2)(\lambda_1(x)-\lambda_1|_F)
    \end{align}
    for all $(\eta,x,y)\in V_{\rho,\Delta}$.
\end{proof}

Finally, we can prove our subcritical convergence theorem, which follows from a similar argument to that of \cite[Theorem D.4]{macdonaldeos}.

\begin{theorem}\label{thm:subcritappendix}
    Let $V_{\rho,\Delta}$ be an invariant set as in Lemma \ref{lem:reduction1}, and denote $V^-_{\rho,\Delta}:=V_{\rho,\Delta}\cap\{\eta\lambda_1|_F-2<0\}$. Then, shrinking $\rho$ if necessary, for any $(\eta,x,y)\in V^-_{\rho,\Delta}$ there is
    \begin{align}
    T = O\bigg( y_1^{-2}\bigg(\frac{\lambda_1(x)-\lambda_1|_F}{2/\eta-\lambda_1|_F}\bigg)^{\gamma}\bigg)
    \end{align}
    such that $\eta <2/\lambda_1(x_t)$ for all $t\geq T$, following which there is $\beta<1$ such that the iterates $(x_t,y_t)$ converge with rate $O(\beta^t)$ to a global minimum $(x_{\infty},0)$ of $\ell$. If, moreover, $d_M(x_0,F)>0$, this minimum is suboptimally flat: 
    \begin{align}
    \lambda_1(x_{\infty})-\lambda_1|_F\geq \exp\big(-O(y_{1,T}^2(1-\beta^2)^{-1})\big)(\lambda_1(x_T)-\lambda_1|_F).
    \end{align}
\end{theorem}

\begin{proof}
    For notational convenience, denote
    \begin{align}
        L(x):=\lambda_1(x)-\lambda_1|_F,\qquad L_t:=L(x_t).
    \end{align}
    By Lemma \ref{lem:reduction1}, after shrinking $\rho$ if necessary, the iterates remain in $V^-_{\rho,\Delta}$ and
    \begin{align}
        \|y_{2:q,t}\|\leq (1-\Delta)^t\|y_{2:q}\|.
    \end{align}
    Thus the $y_{2:q}$ variables decay exponentially and, in particular, do not affect the estimates below.

    We now prove that $\eta\lambda_1(x_t)<2$ in finite time. If $\eta<2/\lambda_1(x)$, there is nothing to prove, so suppose that $\eta\lambda_1(x)\geq 2$. By Lemma \ref{lem:morsebottdescent}, there is $c_0>0$ such that
    \begin{align}
        L_{t+1}\leq (1-c_0y_{1,t}^2)L_t
    \end{align}
    for all $t$, so that $\{\lambda_1(x_t)\}_{t\in\NB}$ is monotonically decreasing. Hence Lemma \ref{lem:contraction2} applies to the $y_1$-iterates, yielding
    \begin{align}
        |y_{1,t}|\geq \frac{|y_1|}{\sqrt{1+3y_1^2t}}
    \end{align}
    for all $t$ such that $\eta\lambda_1(x_t)\geq 2$. Therefore, as long as $\eta\lambda_1(x_t)\geq 2$,
    \begin{align}
        L_{t+1}
        &\leq \bigg(1-\frac{c_0y_1^2}{1+3y_1^2t}\bigg)L_t.
    \end{align}
    Taking logarithms, using $\log(1-z)\leq -z$, and summing gives
    \begin{align}
        \log L_t-\log L_0
        &\leq -c_0\sum_{s=0}^{t-1}\frac{y_1^2}{1+3y_1^2s}  \\
        &\leq -c_1\log(1+3y_1^2t)
    \end{align}
    for some $c_1>0$. Thus
    \begin{align}
        L_t\leq L_0(1+3y_1^2t)^{-c_1}.
    \end{align}
    Since $\eta<2/\lambda_1|_F$, the quantity
    \begin{align}
        \epsilon_{\eta}:=2/\eta-\lambda_1|_F
    \end{align}
    is positive. Consequently, if
    \begin{align}
        t\geq C y_1^{-2}\bigg(\frac{L_0}{\epsilon_{\eta}}\bigg)^{1/c_1}
    \end{align}
    with $C>0$ sufficiently large, then $L_t<\epsilon_{\eta}$, equivalently $\eta<2/\lambda_1(x_t)$. Hence, setting $\gamma:=1/c_1$, there is
    \begin{align}
        T=O\bigg(y_1^{-2}\bigg(\frac{\lambda_1(x)-\lambda_1|_F}{2/\eta-\lambda_1|_F}\bigg)^\gamma\bigg)
    \end{align}
    such that $\eta<2/\lambda_1(x_T)$. Since $\lambda_1(x_t)$ is monotonically decreasing, it follows that $\eta<2/\lambda_1(x_t)$ for all $t\geq T$.

    We now prove convergence after time $T$. Since $\eta\lambda_1(x_t)<2$ for all $t\geq T$, the $y_1$ update satisfies
    \begin{align}
        |y_{1,t+1}|
        &= |y_{1,t}|\big|1-(2-\eta\lambda_1(x_t))-y_{1,t}^2+O(y_{1,t}^3)\big|.
    \end{align}
    Shrinking $\rho$ if necessary and using the monotonicity of $\lambda_1(x_t)$, we obtain
    \begin{align}
        |y_{1,t+1}|\leq \beta |y_{1,t}|,
        \qquad 
        \beta:=1-(2-\eta\lambda_1(x_T))<1,
    \end{align}
    for all $t\geq T$. Thus $y_{1,t}=O(\beta^{t-T})$, while Lemma \ref{lem:reduction1} gives $\|y_{2:q,t}\|=O((1-\Delta)^t)$. The $x$ update has size
    \begin{align}
        d_M(x_{t+1},x_t)
        =O(y_{1,t}^2\|\nabla_M\lambda_1(x_t)\|)
        +O(|y_{1,t}|^3d_M(x_t,F)) = O(\beta^{2(t-T)}),
    \end{align}
    for all $t\geq T$; hence $x_t$ converges to some $x_{\infty}\in M$. Increasing $\beta$ if necessary such that $1>\beta>1-\Delta$ then gives the convergence rate.

    It remains to prove the claimed suboptimal flatness bound. By the Morse--Bott condition, there is $C'>0$ such that
    \begin{align}
        \|\nabla_M\lambda_1(x)\|^2\leq C'(\lambda_1(x)-\lambda_1|_F)
    \end{align}
    throughout $V^-_{\rho,\Delta}$. Using the normal form for the $x$ update and a covariant Taylor expansion from below gives, for all $t\geq T$,
    \begin{align}
        L_{t+1}
        &\geq L_t-C'y_{1,t}^2\|\nabla_M\lambda_1(x_t)\|^2+O(|y_{1,t}|^3L_t) \\
        &\geq (1-C'y_{1,t}^2)L_t,
    \end{align}
    after shrinking $\rho$ once more if necessary. Since $|y_{1,t}|\leq |y_{1,T}|\beta^{t-T}$ for $t\geq T$, it follows that
    \begin{align}
        L_t
        &\geq L_T\prod_{s=T}^{t-1}\big(1-C'y_{1,T}^2\beta^{2(s-T)}\big).
    \end{align}
    Taking $\rho$ sufficiently small, the factors in the product are positive and $\log(1-z)\geq -2z$ applies. Therefore
    \begin{align}
        \log\bigg(\prod_{s=T}^{t-1}\big(1-C'y_{1,T}^2\beta^{2(s-T)}\big)\bigg)
        &\geq -2C' y_{1,T}^2\sum_{s=T}^{t-1}\beta^{2(s-T)} \\
        &\geq -2C'\frac{y_{1,T}^2}{1-\beta^2}.
    \end{align}
    Letting $t\rightarrow\infty$ yields
    \begin{align}
        \lambda_1(x_{\infty})-\lambda_1|_F
        \geq 
        \exp\bigg(-O\bigg(\frac{y_{1,T}^2}{1-\beta^2}\bigg)\bigg)
        \big(\lambda_1(x_T)-\lambda_1|_F\big),
    \end{align}
    as claimed.
\end{proof}

\subsection{Critical regime}\label{subsec:criticalappendix}

In this subsection, we prove convergence in the critical regime, $\eta=2/\lambda_1|_F$, as in Theorem \ref{thm:criticalmain}. As we will see in Theorem \ref{thm:criticalappendix}, our proof is essentially a Lyapunov function proof and is as such distinct from (and much easier than) the corresponding theorem in \cite{macdonaldeos}, which requires a parabolic invariant manifold argument.

\begin{theorem}\label{thm:criticalappendix}
    Suppose that Assumptions \ref{ass:regularity}, \ref{ass:unique}, \ref{ass:morsebott} and \ref{ass:quality} hold. 
    Let $V_{\rho,\Delta}$ be an set as in Lemma \ref{lem:reduction1} and denote $V^0_{\rho,\Delta}:=V_{\rho,\Delta}\cap\{\eta = 2/\lambda_1|_F\}$. Then, for all $\rho$ sufficiently small, the iterates $(2/\lambda_1|_F,x_t,y_t)$ of any $(2/\lambda_1|_F,x,y)\in V^0_{\rho,\Delta}$ converge to a point in $F$ with rate $\Theta(t^{-1/2})$.
\end{theorem}

\begin{proof}
Work in coordinates $(z,u,y_1,y_{2:q})$ as in Lemma \ref{lem:fermicoordinates}. Our convergence proof essentially reduces to the construction of a suitable Lyapunov function on $V^0_{\rho,\Delta}$ as follows.

For notational convenience, denote $U:=\|u\|^2$, $Y:=y_1^2$, $h:=\langle \nabla_M^2\lambda_1(z)u,u\rangle$ and $\zeta:=\zeta({2/\lambda_1|_F},\cdot)$. Choose $a>0$ sufficiently large that $a\inf_{z\in V}\zeta(z)-\lambda_1|_F^{-1}>0$ and define
\begin{align}
E:=Y+aU.
\end{align}
For all sufficiently small $\rho>0$, the function $E$ on $V^0_{\rho,\Delta}$ is a Lyapunov function for the square of $\GD$. To see this, denote the values of $E,Y,U$ under the square of $\GD$ by $E^+$, $Y^+$ and $U^+$ respectively. It follows from the update formulae of Lemma \ref{lem:fermicoordinates} that
\begin{align}\label{eq:U}
U^+=U-4\zeta Yh+O(Y(U+Y)^{3/2})
\end{align}
\begin{align}\label{eq:Y}
Y^+=Y+4Y(\lambda_1|_F^{-1} h-Y)+O(Y(U+Y)^{3/2})
\end{align}
Since $E\asymp Y+U$, adding \eqref{eq:U} and \eqref{eq:Y} gives
\begin{align}\label{eq:E}
E^+ = E -4YB
+O(YE^{3/2})
\end{align}
where $B:=Y+(a\zeta-\lambda_1|_F^{-1})h$. Note that our choice of $a$ gives $a\zeta-\lambda_1|_F^{-1}>0$; since $\lambda_1$ is Morse-Bott and smooth, there are constants $m,M>0$ such that the coefficient $B$ satisfies
\begin{align}\label{eq:B}
mE\le B\le ME
\end{align}
uniformly over $V^0_{\rho,\Delta}$. Shrinking $\rho$ if necessary, this implies $E^+\leq E$ uniformly over $V^0_{\rho,\Delta}$. To derive the convergence rate, denote $E_t:=E\circ \GD^{2t}$ and $Y_t:=Y\circ\GD^{2t}$ on $V^0_{\rho,\Delta}$. The convergence rate then follows from the following three claims.

\noindent\textbf{Claim 1:} \emph{There is $K>0$ such that
\begin{align}
E_{t+1}\ge E_t-KE_t^2,\qquad\forall t.
\end{align}
Consequently, $\sum_{t=0}^{\infty}E_t$ diverges everywhere on $V^0_{\rho,\Delta}$.}
Indeed, by the estimates \eqref{eq:E} \eqref{eq:B} and $Y\le E$,
\begin{align}
E^+\geq E -4MYE+O(YE^{3/2})\geq E(1-KE)
\end{align}
for some constant $K>0$ and all sufficiently small $E$. To prove divergence of the series, shrinking $\rho$ if necessary, take reciprocals and apply the estimate $(1-x)^{-1}\leq 1+2x$ (valid for small $x$) to give
\begin{align}
\frac{1}{E_{t+1}}\leq
\frac{1}{E_t}+2K,
\end{align}
so that
\begin{align}
E_t\ge \frac{1}{E^{-1}+2Kt}.
\end{align}
Divergence of the harmonic series thus gives
\begin{align}\label{eq:D}
\sum_{t=0}^{\infty}E_t=\infty.
\end{align}

\noindent\textbf{Claim 2:}\emph{ For some sufficiently small $b>0$, there is $t_0\in\NB$ such that
\begin{align}
Y_t\ge bE_t,\qquad \forall t\ge t_0.
\end{align}}
To see this, define $q:=Y/E$, which is well-defined on $V^0_{\rho,\Delta}$ since $E\neq 0$ thereon. Since $Y> 0$ by hypothesis, shrinking $\rho$ if necessary the update \eqref{eq:Y} says that $Y_t> 0$ for all $t$. We can therefore divide the formula \eqref{eq:Y} by $Y$ and consider the multiplicative update
\begin{align}\label{eq:Ymult}
Y^+/Y = 1+4(\lambda_1|_F^{-1} h-Y) + O(E^{3/2})
\end{align}
We estimate the coefficient $4(\lambda_1|_F^{-1} h-Y)$ in terms of $E$. Rearranging $E=Y+aU$ gives $W=\frac{1-q}{a}E$, and the Morse-Bott condition gives $\mu>0$ such that $h\geq \mu U$ uniformly over $V^0_{\rho,\Delta}$. Hence $h\ge \frac{\mu(1-q)}{a}E$, from which it follows that
\begin{align}
\lambda_1|_F^{-1} h-Y\geq(\frac{\lambda_1|_F^{-1}\mu(1-q)}{a}-q)E.
\end{align}
Choose $b>0$ sufficiently small that $p:=\frac{\lambda_1|_F^{-1}\mu(1-2b)}{a}-2b>0$. Then whenever $q\le 2b$, one has $\lambda_1|_F^{-1} h-Y\ge pE$; using \eqref{eq:Ymult}, $q\le 2b$ therefore implies
\begin{align}
Y^+/Y\ge 1+cE
\end{align}
for some $c>0$. Since $E^+\le E$, it follows that
\begin{align}\label{eq:q}
\frac{q^+}{q} = \frac{Y^+/Y}{E^+/E}\geq 1+cE\qquad\text{whenever }q\leq 2b.
\end{align}
If $q_t<b$ for every $t$, then iterating \eqref{eq:q} gives
\begin{align}
q_t\ge q_0\prod_{j=0}^{t-1}(1+cE_j).
\end{align}
By \eqref{eq:D}, $\sum_jE_j=\infty$, so this product diverges, contradicting $q_t<b$. This gives $t_0$ such that $q_{t_0}\geq b$. 

To complete the proof of Claim 2, it remains to show that the region $q\ge b$ is forward invariant. If $b\le q\le 2b$, then \eqref{eq:q} gives $q^+\ge q\ge b$. Moreover, in general, \eqref{eq:Ymult} implies $Y^+/Y\ge 1-O(E)$. Since $E^+\le E$, one has
\begin{align}
q^+ = q\frac{Y^+/Y}{E^+/E}\ge q(1 -O(E))\ge (1/2)q
\end{align}
for all sufficiently small $E$, so that $q\ge 2b$ implies $q^+\ge b$. Thus $q\ge b$ is forward-invariant, completing the proof of Claim 2.

\noindent\textbf{Claim 3:} \emph{There is $k>0$ such that, for all $t\ge t_0$,
\begin{align}
E_t-KE_t^2\leq E_{t+1}\leq E_t-kE_t^2.
\end{align}}
The lower bound is Claim 1. For the upper bound, \eqref{eq:E}, \eqref{eq:B} and Claim 2 give $k>0$ such that $E_{t+1}\leq E_t-4mYE_t + O(YE_t^{3/2})\leq E_t-4mbE_t^2 + O(E_t^{5/2})\leq E_t-kE_t^2$ for all $t\geq t_0$.

\textbf{Convergence: } Claim 1 demonstrates that $E_{t+1}\geq E_t(1-KE_t)$ implies $E_t = \Omega(t^{-1})$. The same argument with reversed inequalities says that $E_{t+1}\leq E_t(1-kE_t)$ implies $E_t = O(t^{-1})$. Thus Claim 3 implies that $E_t = \Theta(t^{-1})$, hence also that the iterate coordinates $(2/\lambda_1|_F,z_t,u_t,y_t) = \GD^t(2/\lambda_1|_F,z,u,y)$ of any $(2/\lambda_1|_F,z,u,y)\in V^0_{\rho,\Delta}$ satisfy $|w_t|+|y_{1,t}| = \Theta(\sqrt{E_t}) = \Theta(t^{-1/2})$. Finally, by Lemma \ref{lem:fermicoordinates} the $z$-updates are $O(E_t^2) = O(t^{-2})$, so $z_t$ converges to some $z_{\infty}\in F$ with rate $O(t^{-1})$.
\end{proof}

\subsection{Supercritical regime}\label{subsec:supercriticalappendix}

In this subsection, we prove Theorem \ref{thm:supercriticalmain}. The argument is similar to the corresponding result in \cite{macdonaldeos}; however, unlike \cite{macdonaldeos}, our result gives convergence in a neighbourhood of radius $\rho = \Theta(1)$ in $\eta\lambda_1|_F-2$.

\begin{theorem}\label{thm:supercriticalappendix}
    Let $V_{\rho,\Delta}$ be an invariant set as in Lemma \ref{lem:reduction1}, and set $V^+_{\rho,\Delta}:=V_{\rho,\Delta}\cap\{\eta\lambda_1|_F-2>0\}$. Then, shrinking $\rho$ if necessary, there is $C>0$ such that any $(\eta,x,y)\in V^+_{\rho,\Delta}$ has iterates $(\eta,x_t,y_t)$ of $\GD$ converging to a stable, period-two orbit about some point $z_{\infty}\in F$ with amplitude $\Theta(\sqrt{\eta\lambda_1|_{F}-2})$ with rates $O\big((1-C(\eta\lambda_1|_F-2))^t\big)$.
\end{theorem}

\begin{proof}
    We work in the coordinates of Lemma \ref{lem:fermicoordinates}. In these coordinates, the $y_{2:q}$ coordinate of the iterates contracts exponentially to zero by definition of $V^+_{\rho,\Delta}$, so it suffices to prove convergence in the $(z,u,y_1)$ coordinates. Since the span of the normal vector through any point $z\in F$ is invariant by Assumption \ref{ass:morsebott}, the subspace $u=0$ is preserved by $\GD$, and thereon $\GD_{y_1}$ takes the form
    \begin{align}
        \GD_{y_1}(\eta,z,0,\cdot):y_1\mapsto -(1+(\eta\lambda_1|_F-2))y_1+y_1^3+O(y_1^4),
    \end{align}
    Hence, by Lemma \ref{lem:contraction1}, for each $z$ there is a period-two orbit $\{\xi_+(\eta,z),\xi_-(\eta,z)\}$ of the form $\xi_{\pm}(\eta,z)=\pm\sqrt{\eta\lambda_1|_F-2}+O(\eta\lambda_1|_F-2)$. It follows that for every $z$ the pair $\{(z,0,\xi_+(\eta,z)),(z,0,\xi_-(\eta,z))\}$ is a period-two orbit of $\GD$.

    To prove convergence, we first define a neighbourhood on which the desired uniform estimates can be obtained. Specifically, define
    \begin{align}
        \alpha(\eta):=\eta\lambda_1|_F-2,\qquad a(\eta):=\frac{\alpha(\eta)}{(1+\alpha(\eta))(1+(1+\alpha(\eta))^2)}
    \end{align}
    and, given $\beta>0$,
    \begin{align}
    V_{\beta,\rho,\Delta}^+:=V^+_{\rho,\Delta}\cap\{a(\eta)\leq y_1^2\leq 4\alpha(\eta),\quad \|u\|\leq \beta\sqrt{\alpha(\eta)}\}.
    \end{align}
    We first prove that, for sufficiently small $\rho,\beta$, $V^+_{\beta,\rho,\Delta}$ is invariant and the iterates starting from any point in $V^+_{\rho,\Delta}$ enter $V^+_{\beta,\rho,\Delta}$ in finite time. Second, third and fourth, we establish that the iterates starting from any point in $V^+_{\beta,\rho,\Delta}$ have $u$-component, $z$-component and $y_1$-component admitting the claimed convergence behaviour.
    
    \noindent\textbf{Claim 1: }\emph{For any sufficiently small $\beta>0$, $V^+_{\beta,\rho,\Delta}$ is invariant, and any point in $V^+_{\rho,\Delta}$ has iterates contained in $V^+_{\beta,\rho,\Delta}$ after a finite number of steps.} Shrinking $\rho$ if necessary, there are constants $c,C_1>0$ such that
    \begin{align}\label{eq:supercriticalapproach}
        \|\GD_u(\eta,z,u,y_1)\|&\leq (1-cy_{1}^2)\|u\|,\\
        |\GD_{y_1}(\eta,z,u,y_1)|
        &=|y_{1}|\left(1+\alpha(\eta)+\frac{\eta}{2}D^2_M\lambda_1(z)[u,u]-y_{1}^2+R(\eta,z,u,y)\right)
    \end{align}
    with $|R(\eta,z,u,y)|\leq C_1(\|u\|^3+|y_{1}|^3)$ and
    \begin{align}\label{eq:factorbounds}
        \left(1+\alpha(\eta)+\frac{\eta}{2}D^2_M\lambda_1(z)[u,u]-y_{1}^2+R(\eta,z,u,y)\right)\in[1/2,3/2]
    \end{align}
    uniformly over $V^+_{\rho,\Delta}$. To see that $V^+_{\beta,\rho,\Delta}$ is invariant, observe that \eqref{eq:supercriticalapproach} already shows that the bound on $u$ is preserved. Moreover, uniformly on $V^+_{\beta,\rho,\Delta}$,
    \begin{align}
        |\GD_{y_1}(\eta,z,u,y_1)|^2
        =
        y_1^2\left(1+2(\alpha(\eta)-y_1^2)
        +O(\beta^2\alpha(\eta)+\alpha(\eta)^{3/2})\right).
    \end{align}
    For either choice of the sign of $y_1$, for all $\rho$ sufficiently small the right-hand side is an increasing function of $y_1^2\in[a(\eta),4\alpha(\eta)]$. At the two endpoints one has
    \begin{align}
        |\GD_{y_1}(\eta_0,z,u,\sqrt{a(\eta)})|^2
        &=a(\eta)\left(1+2(\alpha(\eta)-a(\eta))
        +O(\beta^2\alpha(\eta)+\alpha(\eta)^{3/2})\right)
        \geq a(\eta),\\
        |\GD_{y_1}(\eta,z,u,2\sqrt{\alpha(\eta)})|^2
        &=4\alpha(\eta)\left(1-6\alpha(\eta)
        +O(\beta^2\alpha(\eta)+\alpha(\eta)^{3/2})\right)
        \leq4\alpha(\eta),
    \end{align}
    after first choosing $\beta,\rho$ sufficiently small. Hence
    \begin{align}
        a(\eta)
        \leq |\GD_{y_1}(\eta,z,u,y_1)|^2
        \leq4\alpha(\eta).
    \end{align}
    Since $V^+_{\rho,\Delta}$ is invariant, all the defining conditions of $V^+_{\beta,\rho,\Delta}$ are preserved by $\GD$.
    
    We now come to showing that the iterates of any point in $V^+_{\rho,\Delta}$ enter $V^+_{\beta,\rho,\Delta}$ in finite time. For this, we shrink $\beta$ further if necessary so that
    \begin{align}\label{eq:betaeqn}
        0\leq\frac{\eta}{2}D^2_M\lambda_1(z)[u,u]\leq\frac{\alpha(\eta)}{8},
        \qquad
        |R(\eta,z,u,y)|\leq\frac{\alpha(\eta)}{8}+\frac{y_{1}^2}{8}
    \end{align}
    uniformly over $V^+_{\rho,\Delta}$ whenever $\|u\|\leq\beta\sqrt{\alpha(\eta)}$. Now fix $(\eta,z_0,u_0,y_0)\in V^+_{\rho,\Delta}$ with iterates $(\eta,z_t,u_t,y_t)$. We first claim that $\|u_t\|\rightarrow 0$. Indeed, suppose for contradiction that $\sum_{t\geq 0}y_{1,t}^2<\infty$, so that $y_{1,t}\rightarrow 0$. Using the fact that $\lambda_1$ is Morse-Bott and shrinking $\rho$ if necessary,
    \begin{align}
        \frac{\eta}{2}D^2_M\lambda_1(z_t)[u_t,u_t]-C_1\|u_t\|^3\geq 0
    \end{align}
    for every $t$. Consequently, for all $t$ sufficiently large,
    \begin{align}
        \alpha(\eta)+\frac{\eta}{2}D^2_M\lambda_1(z_t)[u_t,u_t]-y_{1,t}^2+R_t
        \geq \alpha(\eta)-y_{1,t}^2-C_1|y_{1,t}|^3
        \geq \frac{\alpha(\eta)}{2}.
    \end{align}
    It would follow that
    \begin{align}
        |y_{1,t+1}|\geq(1+\alpha(\eta)/2)|y_{1,t}|
    \end{align}
    for all sufficiently large $t$, contradicting $y_{1,t}\rightarrow 0$. Hence $\sum_{t\geq 0}y_{1,t}^2=\infty$, and the first estimate in \eqref{eq:supercriticalapproach} gives
    \begin{align}
        \|u_t\|\leq \|u_0\|\prod_{j=0}^{t-1}(1-cy_{1,j}^2)\leq \|u_0\|\exp\left(-c\sum_{j=0}^{t-1}y_{1,j}^2\right)
        \longrightarrow 0.
    \end{align}
    There is therefore $T\in\NB$ such that
    \begin{align}
        \|u_t\|\leq \beta\sqrt{\alpha(\eta)},\qquad \forall t\geq T,
    \end{align}
    and this estimate is preserved thereafter since $(\|u_t\|)_{t\geq 0}$ is non-increasing.

    If $y_{1,T}^2\in[a(\eta),4\alpha(\eta)]$ then we are done. Otherwise, using \eqref{eq:betaeqn}, since $a(\eta)\leq \alpha(\eta)/2$, it follows that
    \begin{align}
        y_{1,t}^2
        \leq a(\eta)
        \quad\Longrightarrow\quad
        |y_{1,t+1}|\geq(1+\alpha(\eta)/4)|y_{1,t}|,
        \label{eq:supercriticallowerentry}
    \end{align}
    whereas
    \begin{align}
        y_{1,t}^2\geq4\alpha(\eta)
        \quad\Longrightarrow\quad
        |y_{1,t+1}|\leq(1-\alpha(\eta))|y_{1,t}|.
        \label{eq:supercriticalupperentry}
    \end{align}
    Thus, if $y_{1,T}^2$ is below $a(\eta)$, its iterates cannot remain below this bound by \eqref{eq:supercriticallowerentry}; and if $y_{1,T}^2>4\alpha(\eta)$, its iterates cannot remain above $4\alpha(\eta)$ by \eqref{eq:supercriticalupperentry}. In the former case, at the first time $\tau\geq T$ for which $y_{1,\tau}^2\geq a(\eta)$, the upper bound $|y_{1,\tau}|/|y_{1,\tau-1}|\leq3/2$ from \eqref{eq:factorbounds} gives $y_{1,\tau}^2\leq (9/4)a(\eta)<4\alpha(\eta)$. In the latter case, at the first time $\tau\geq T$ for which $y_{1,\tau}^2\leq4\alpha(\eta)$, the lower bound $|y_{1,\tau}|/|y_{1,\tau-1}|\geq1/2$ of \eqref{eq:factorbounds} gives $y_{1,\tau}^2\geq\frac14y_{1,\tau-1}^2>\alpha(\eta)>a(\eta)$. In either case there is therefore a finite $\tau\geq T$ such that
    \begin{align}
        \|u_\tau\|\leq \beta\sqrt{\alpha(\eta)},
        \qquad
        a(\eta)
        \leq y_{1,\tau}^2\leq4\alpha(\eta),
    \end{align}
    so that $(\eta,z_\tau,u_\tau,y_{1,\tau})\in V^+_{\beta,\rho,\Delta}$. For notational convenience, we assume without loss of generality that $\tau = 0$ in what follows, i.e. $(\eta,z_0,u_0,y_0)\in V^+_{\beta,\rho,\Delta}$.

    \noindent\textbf{Claim 2: }\emph{There is $C_2>0$ such that, once inside $V^+_{\beta,\rho,\Delta}$, the iterates $u_t$ converge to zero with rate $O((1-C_2\alpha(\eta))^t)$.} Indeed, shrinking $\rho$ if necessary, inside $V^+_{\beta,\rho,\Delta}$ one has $y_{1,t}^2\geq a(\eta)\geq (1/3)\alpha(\eta)$. Hence
    \begin{align}
    \|u_{t+1}\|\leq (1-(c/3)\alpha(\eta))\|u_{t}\|\Rightarrow \|u_{t}\|\leq (1-(c/3)\alpha(\eta))^t\|u_{0}\|
    \end{align}
    for all $t\in\NB$; setting $C_2:=c/3$ gives the claim.
    
    \noindent\textbf{Claim 3: }\emph{The iterates $z_t$ converge to some $z_{\infty}$ with rate $O\big((1-C_2\alpha(\eta))^t\big)$.} The update equation for $z$ gives
    \begin{align}
        d_F(z_{t+1},z_{t})
        = O(y_{1,t}^3\|u_{t}\|,y_{1,t}^2\|u_{t}\|^2)= O\big(\alpha(\eta)(1-C_2\alpha(\eta))^t\big)
    \end{align}
    since $\|u_t\| = O\big((1-C_2\alpha(\eta))^t\big)$ and $y_{1,t}^2 = O(\alpha(\eta))$. The right side is summable, so $\{z_{t}\}_{t\geq0}$ is Cauchy and hence converges to some $z_\infty$. Moreover, summing the geometric series gives
    \begin{align}
        d_F(z_{t},z_\infty) \leq\sum_{s=t}^{\infty}\|z_{s+1}-z_s\| =  O\big((1-C_2\alpha(\eta))^t\big)
    \end{align}
    as claimed.

    \noindent\textbf{Claim 4: }\emph{There is $0<C\leq C_2$ such that the iterates $y_{1,t}$ converge with rate $O\big((1-C\alpha(\eta))^t\big)$ to the period-two orbit $\{\xi_{\pm}(\eta,z_{\infty})\}$.} It is evident from the formula for $\GD_{y_1}$ that the sign of $y_{1,t}$ alternates; assume without loss of generality that $\mathrm{sign}(y_{1,0}) = +$. It suffices then to find $0<C\leq C_2$ for which the iterates $y_{1,2t}$ converge with rate $O\big((1-C\alpha(\eta))^{2t}\big)$ to $\xi_+(\eta,z_{\infty})$.
    
    The square $\GD^2_{y_1}$ acts on $(\eta,z,u,y_1)$ to give
    \begin{align}\label{eq:GD2}
    y_1\big((1+\alpha(\eta))^2 + (1+\alpha(\eta))\eta D^2_M\lambda_1(z)[u,u] -(1+\alpha(\eta))(1+(1+\alpha(\eta))^2)y_1^2\big)
    \end{align}
    up to a possibly $z$-dependent term $O((\|u\|+|y_1|)^3)$. The restriction of this map to the invariant subspace $u=0$ will be denoted
    \begin{align}\label{eq:gd2}
        g_{\eta_0,z}(y_1):=\GD_{y_1}^2(\eta_0,z,0,y_1) = y_1\big((1+\alpha(\eta))^2 -(1+\alpha(\eta))(1+(1+\alpha(\eta))^2)y_1^2+ O(y_1^3)\big)
    \end{align}
    and has the following nice properties: first, the two points $\xi_\pm(\eta,z)$ are the nonzero fixed points of $g_{\eta,z}$, and, second
    \begin{align}
        D_{y_1}g_{\eta,z}(y_1)
        &=(1+\alpha(\eta))^2
        -3(1+\alpha(\eta))(1+(1+\alpha(\eta))^2)y_1^2
        +O(|y_1|^3).
    \end{align}
    Consequently, there is $C_3>0$ such that
    \begin{align}\label{eq:scalarcontractionsupercritical}
        |D_{y_1}g_{\eta,z}(y_1)|\leq1-C_3\alpha(\eta)
    \end{align}
    uniformly over $V^+_{\beta,\rho,\Delta}$.

    We next compare the true $y_1$-update under $\GD^2$ with $g_{\eta,z_\infty}$. It is apparent from \eqref{eq:GD2} and \eqref{eq:gd2} that
    \begin{align}\label{eq:yperturbationsupercritical}
        \left|
        \GD_{y_1}^2(\eta,z,u,y_1)
        -g_{\eta,z_\infty}(y_1)
        \right|= O\left(
        \sqrt{\alpha(\eta)}\|u\|^2
        +\alpha(\eta)^{3/2}\|u\|
        +\alpha(\eta)^2d_F(z,z_\infty)
        \right).
    \end{align}
    Since $g_{\eta,z_\infty}(\xi_+(\eta,z_\infty))=\xi_+(\eta,z_\infty)$, the mean-value theorem, \eqref{eq:scalarcontractionsupercritical} and \eqref{eq:yperturbationsupercritical} give
    \begin{align}
        &|y_{1,2(t+1)}-\xi_+(\eta,z_\infty)|\\
        &\qquad\leq
        (1-C_3\alpha(\eta))
        |y_{1,2t}-\xi_+(\eta,z_\infty)|\\
        &\qquad\quad+
        O\left(
        \sqrt{\alpha(\eta)}\|u_{2t}\|^2
        +\alpha(\eta)^{3/2}\|u_{2t}\|
        +\alpha(\eta)^2d_F(z_{2t},z_\infty)
        \right).
    \end{align}
    By Claims 2 and 3, after decreasing $C_2>0$ if necessary, the final line is
    \begin{align}
        O\left(\alpha(\eta)^{3/2}(1-C_2\alpha(\eta))^{2t}\right).
    \end{align}
    Iterating the preceding inequality and summing the resulting geometric convolution therefore gives some $C_4>0$ such that
    \begin{align}
        |y_{1,2t}-\xi_+(\eta,z_\infty)|
        =
        O\left(\sqrt{\alpha(\eta)}(1-C_4\alpha(\eta))^t\right).
    \end{align}
    Since $t$ here counts applications of $\GD^2$, decreasing $C_4$ allows this to be written as
    \begin{align}
        |y_{1,2t}-\xi_+(\eta,z_\infty)|
        =
        O\left((1-C_4\alpha(\eta))^{2t}\right).
    \end{align}
    Finally, setting $C:=\min\{C_4,C_2\}$ gives the result.
\end{proof}

\section{Matrix factorisation}\label{sec:matrixfactorisation}

In this section we prove that deep matrix factorisation problems satisfy all of the assumptions made in the paper. Given $L\in\NB_{\geq 2}$ and $d_0,\dots,d_L\in\NB$, consider $f:\prod_{l=1}^{L}\RB^{d_l\times d_{l-1}}\rightarrow \RB^{d_L\times d_0}$ defined by
\begin{align}\label{eq:matrixfactorisationf}
f(w):=W_L\cdots W_1,\qquad w:=(W_1,\dots,W_L)\in \prod_{l=1}^{L}\RB^{d_l\times d_{l-1}}.
\end{align}
Given $l\geq m$, we denote $W_{l:m}:=W_l\cdots W_m$; if $l<m$ then $W_{l:m}$ will denote the identity. 

\begin{proposition}\label{prop:regularvalue}
    Assume that $d_l\geq d_L$ for all $l\leq L$, and let $\tau\in\RB^{d_L\times d_0}$ have singular values $\sigma_1 >\sigma_2\geq\cdots\geq \sigma_{d_L}>0$. Then $\tau$ is in the range of $f$ and is a regular value of $f$.
\end{proposition}

\begin{proof}
    That $\tau$ is contained in the range of $f$ is obvious from the assumption that $d_l\geq d_L$ for all $l\leq L$. To see that it is a regular value of $f$, note that the derivative of $f$ evaluates on a tangent vector $(\xi_1,\dots,\xi_L)\in \prod_{l=1}^L\RB^{d_l\times d_{l-1}}$ to give
    \begin{align}
    Df(w)[\xi_1,\dots,\xi_L] = \sum_{l=1}^LW_{L:l+1}\,\xi_l\,W_{l-1:1}.
    \end{align}
    Since $\tau$ has rank $d_L$, any $w\in M:=f^{-1}\{\tau\}$ must have $W_L\cdots W_2$ being of rank $d_L$, implying that
    \begin{align}
    \mathrm{range}(Df(w))\supseteq \{W_L\cdots W_2\xi_1:\xi_1\in\RB^{d_1\times d_0}\} = \RB^{d_L\times d_0},
    \end{align}
    implying that $Df(w)$ is full-rank and thus making $\tau$ a regular value of $f$ as claimed.
\end{proof}

Define
\begin{align}
\ell(w):=\frac{1}{2}\|\tau-f(w)\|_F^2,\qquad w:=(W_1,\dots,W_L)\in\prod_{l=1}^L\RB^{d_l\times d_{l-1}}.
\end{align}
Since we are concerned only with quantities that depend on the metric and the derivatives of $f$, we may assume without loss of generality that $\tau = \mathrm{diag}(\sigma_1,\dots,\sigma_{d_L})$ is diagonal. We now consider the eigendata of the Hessian $\nabla^2\ell|_M = Df^TDf$, whose eigenvalue fields we denote $\lambda_1\geq\lambda_2\geq\dots$. Denote
\begin{align}
S:=\{w\in M:\lambda_1(w) = \lambda_2(w)\}
\end{align}
for the singular set; since the $\lambda_i$ are continuous, $S$ is closed. 

The set of nonsingular flat minima $F$ will be realised as an analytic fibre bundle over a product of spheres. To see this, first observe that the subset
\begin{equation}\label{eq:submanifold}
\{(\bW_1,\dots,\bW_L):\bW_L\cdots \bW_1 = \mathrm{diag}(\sigma_2,\dots,\sigma_{d_L})\}
\end{equation}
of $\prod_{l=1}^L\RB^{(d_l-1)\times (d_{l-1}-1)}$ is in fact a submanifold by the argument of Proposition \ref{prop:regularvalue}. Given $\bw:=(\bW_1,\dots,\bW_L)\in\prod_{l=1}^L\RB^{(d_l-1)\times (d_{l-1}-1)}$, define the linear maps
\begin{align}
\Delta_{01}(\bw):=\sum_{l=1}^L\sigma_1^{2(L-l)/L}\bW^T_{l-1:1}\bW_{l-1:1}\in\RB^{(d_0-1)\times (d_0-1)}
\end{align}
\begin{align}
\Delta_{10}(\bw):=\sum_{l=1}^L\sigma_1^{2(l-1)/L}\bW_{L:l+1}\bW^T_{L:l+1}\in \RB^{(d_L-1)\times (d_L-1)}
\end{align}
\begin{align}
\Delta_{11}(\bw):=\sum_{l=1}^L\bW_{L:l+1}\bW_{L:l+1}^T\otimes \bW_{l-1:1}^T\bW_{l-1:1}:\RB^{(d_L-1)\times (d_0-1)}\rightarrow\RB^{(d_L-1)\times (d_0-1)}
\end{align}
Denote by $\bM$ the intersection of the manifold \eqref{eq:submanifold} with the open set $\{\tw:\max\{\lambda_1(\Delta_{01}(\bw)),\lambda_1(\Delta_{10}(\bw)),\lambda_1(\Delta_{11}(\bw))\}<L\sigma_1^{2-2/L}\}$; as the intersection of a submanifold with an open set, $\bM$ is itself a manifold, and will be the typical fibre of the fibre bundle $F$. To construct the bundle itself, letting $O(d)$ denote the orthogonal group in $d$ dimensions, observe that $\bM$ carries an action of the group $O:=\prod_{l=1}^{L-1}O(d_l-1)$ defined by
\begin{align}\label{eq:oaction}
(Q_1,\dots,Q_{L-1})\cdot(\bW_1,\dots,\bW_L):=(Q_1\bW_1,Q_2\bW_2Q_1^T,\dots,\bW_LQ_{L-1}^T).
\end{align}
Letting $S^{d-1}\subset\RB^d$ denote the sphere, consider the principal $O$-bundle $P\rightarrow \prod_{l=1}^{L-1}S^{d_l-1}$ whose fibre over $(u_1,\dots,u_{L-1})$ is the space of tuples $((u_1,U_1),\dots,(u_{L-1},U_{L-1}))\in \prod_{l=1}^{L-1}S^{d_{l}-1}\times\RB^{d_l\times(d_l-1)}$ where each $U_l$ is an orthonormal frame for the orthogonal complement $u_l^{\perp}$ of $u_l$. Note that $P$ carries the canonical right action
\begin{align}
((u_1,U_1),\dots,(u_{L-1},U_{L-1}))\cdot(Q_1,\dots,Q_{L-1}):=((u_1,U_1Q_1),\dots,(u_{L-1},U_{L-1}Q_{L-1}))
\end{align}
of $O$. Finally, consider the associated bundle
\begin{align}
B:=P\times_{O}\bM\rightarrow \prod_{l=1}^{L-1}S^{d_l-1},
\end{align}
whose total space is the quotient of the product $P\times\bM$ by the group action of $O$ given by
\begin{align}
(Q_1,\dots,Q_{L-1})&\cdot ((u_1,U_1),\dots,(u_{L-1},U_{L-1});\bW_1,\dots,\bW_L)\\&:=((u_1,U_1Q_1),\dots,(u_{L-1},U_{L-1}Q_{L-1});Q_1^T\bW_1,Q_2^T\bW_2Q_1,\dots,\bW_LQ_{L-1}).
\end{align}
Since $P$ and $\bM$ are analytic, $O$ is compact and its action on $P\times\bM$ is free, $B\rightarrow\prod_{l=1}^{L-1}S^{d_l-1}$ is an analytic fibre bundle with typical fibre $\bM$.

\begin{proposition}\label{prop:fibrebundle}
    The map $\phi:B\rightarrow M$ defined by
    \begin{align}
    \phi\big([(u_1,U_1),\dots,(u_{L-1},U_{L-1});\bW_1,\dots,\bW_L]\big):=\bigg([u_l,U_l]\begin{pmatrix} \sigma_1^{1/L} & 0 \\ 0 & \bW_l\end{pmatrix}[u_{l-1},U_{l-1}]^T\bigg)_{l=1}^L,
    \end{align}
    where $[u_0,U_0] := I_{d_0}$ and $[u_L,U_L]:=I_{d_L}$, is an analytic embedding whose range coincides with $F$; consequently, $F$ is an analytic fibre bundle over $\prod_{l=1}^{L-1}S^{d_l-1}$ with typical fibre $\bM$. Moreover, any $w\in F$ may be mapped via an isometry which preserves $f$ to a point of the form
    \begin{align}\label{eq:wnormalform}
    \Bigg(\begin{pmatrix} \sigma_1^{1/L} & 0 \\ 0 & \bW_l\end{pmatrix}\Bigg)_{l=1}^L
    \end{align}
    for some $(\bW_1,\dots,\bW_L)\in \bM$.
\end{proposition}

\begin{proof}
    It must first be demonstrated that $\phi$ is well-defined. Given $(Q_1,\dots,Q_{L-1})\in O$ and $((u_1,U_1),\dots,(u_{L-1},U_{L-1});\bW_1,\dots,\bW_L)\in P\times\bM$, define
    \begin{align}
    ((u_1,U_1'),&\dots,(u_{L-1},U_{L-1}');\bW_1',\dots,\bW_L')\\&:=(Q_1,\dots,Q_{L-1})\cdot ((u_1,U_1),\dots,(u_{L-1},U_{L-1});\bW_1,\dots,\bW_L).
    \end{align}
    Then simply observe that for any $l\in\{1,\dots,L\}$, setting $Q_0 = I_{d_0-1}$ and $Q_L:=I_{d_L-1}$,
    \begin{align}
    [u_l,U_l']\begin{pmatrix}\sigma_1^{1/L} & 0 \\ 0 & \bW_l'\end{pmatrix}[u_{l-1},U_{l-1}'] &= [u_l,U_lQ_l]\begin{pmatrix} \sigma_1^{1/L} & 0 \\ 0 & Q_l^T\bW_l Q_{l-1}\end{pmatrix}[u_{l-1},U_{l-1}Q_{l-1}]^T\\&=[u_l,U_l]\begin{pmatrix} \sigma_1^{1/L} & 0 \\ 0 & \bW_l\end{pmatrix}[u_{l-1},U_{l-1}]^T,
    \end{align}
    so that $\phi$ is well-defined. It is also easy to see that $\phi$ is injective; supposing that $((u_1,U_1),\dots,(u_{L-1},U_{L-1});\bW_1,\dots,\bW_L),((u_1',U_1'),\dots,(u_{L-1}',U_{L-1}');\bW_1',\dots,\bW_L')\in P\times\bM$ satisfy
    \begin{align}
    [u_l,U_l]\begin{pmatrix} \sigma_1^{1/L} & 0 \\ 0 & \bW_l\end{pmatrix}[u_{l-1},U_{l-1}]^T = [u_l',U_l']\begin{pmatrix} \sigma_1^{1/L} & 0 \\ 0 & \bW_l'\end{pmatrix}[u_{l-1}',U_{l-1}']^T,
    \end{align}
    then $u_l = u_l'$ for all $l$,
    \begin{align}
    U_l\bW_lU_{l-1}^T  = U_l'\bW_l'U_{l-1}'^T\Rightarrow \bW_l = (U_l'^TU_l)^T\bW_l'(U_{l-1}'^TU_{l-1})
    \end{align}
    and clearly $U_l = U_l'(U_l'^TU_l)$ for all $l$; thus $\big[(u_1,U_1),\dots,(u_{L-1},U_{L-1});\bW_1,\dots,\bW_L] = [(u_1',U_1'),\dots,(u_{L-1}',U_{L-1}');\bW_1',\dots,\bW_L']$ in $B$. Finally, that $\phi$ is an analytic embedding is clear from its definition via matrix multiplication. 

    It remains to be shown that the range of $\phi$ coincides with $F$. We argue in a similar fashion to \cite{mulayoff2}. We first derive a lower-bound for $\lambda_1\equiv\lambda_1(\nabla^2\ell|_M)\equiv\lambda_1(DfDf^T|_M)$; we then prove that this lower-bound is met on the desired manifold $F$ and prove the normal form.
    
    First, the lower-bound. For any $w = (W_1,\dots,W_L)\in M$, one has 
    \begin{align}
    Df\,Df^T(w) = \sum_{l=1}^LW_{L:l+1}W_{L:l+1}^T\otimes W_{l-1:1}^TW_{l-1:1},
    \end{align}
    so that 
    \begin{align}
    \lambda_1(w) = \sup_{\|Z\|_F = 1}\langle Z,Df\,Df^T(w)[Z]\rangle = \sup_{\|Z\|_F=1}
\sum_{l=1}^L\|W_{L:l+1}^TZW_{l-1:1}^T\|_F^2.
    \end{align}
    Consider in particular $Z:=e_1\,e_1^T$, where the former $e_1$ is the first standard basis vector in $\RB^{d_L}$ and the latter is the first standard basis vector in $\RB^{d_0}$ (these being the top left and right singular vectors of the diagonal $\tau$ respectively). One has
    \begin{align}
    \lambda_1(w)&\geq \sum_{l=1}^L\|W_{L:l+1}^Te_1\|^2\|W_{l-1:1}e_1\|^2\nonumber\\ &=\sum_{l=1}^L\|W_{L:l+1}^Te_1\|^2\,\|W_{l:1}e_1\|^2\frac{\|W_{l-1:1}e_1\|^2}{\|W_{l:1}e_1\|^2}\nonumber\\&\geq \sum_{l=1}^L\langle W_{L:l+1}^Te_1,W_{l:1}e_1\rangle^2\frac{\|W_{l-1:1}e_1\|^2}{\|W_{l:1}e_1\|^2}\nonumber\\& = \sum_{l=1}^L\langle e_1,W_{L:1}e_1\rangle^2\frac{\|W_{l-1:1}e_1\|^2}{\|W_{l:1}e_1\|^2}\nonumber\\&=\sigma_1^2\sum_{l=1}^L\frac{\|W_{l-1:1}e_1\|^2}{\|W_{l:1}e_1\|^2}\nonumber\\&\geq L\sigma_1^2\bigg(\prod_{l=1}^L\frac{\|W_{l-1:1}e_1\|}{\|W_{l:1}e_1\|}\bigg)^{\frac{2}{L}}\\& = L\sigma_1^{2-2/L},\label{eq:lowerbound}
    \end{align}
    where the third line follows from Cauchy-Schwarz and the sixth follows from the inequality between arithmetic and geometric means. We next consider when these estimates are equalities, so that the lower-bound is achieved.

    The Cauchy-Schwarz estimate is an equality precisely when $W_{L:l+1}^Te_1$ is a scalar multiple of $W_{l:1}e_1$ for all $l=1,\dots,L$. Setting $u_0:=e_1\in\RB^{d_0}$ and $u_l:=W_{l:1}u_0/\|W_{l:1}u_0\|$ for $l=1,\dots, L$ so that $u_L = e_1\in\RB^{d_L}$, it follows that $(u_L,u_l)$ is a singular vector pair for $W_{L:l+1}$ with singular value $\sigma_1/\|W_{l:1}u_0\|$, for all $l=1,\dots, L$. Indeed, for any $l=1,\dots,L$, equality in the Cauchy-Schwarz estimate implies that there is a scalar $c_l$ such that $W_{L:l+1}^Tu_L= c_lu_l$. Taking the inner product of both sides with $W_{l:1}u_0$ then gives
    \begin{align}
    c_l\|W_{l:1}u_0\| = u_L^TW_{L:l+1}W_{l:1}u_0 = \sigma_1\Rightarrow c_l = \frac{\sigma_1}{\|W_{l:1}u_0\|}
    \end{align}
    so that
    \begin{equation}\label{eq:align1}
        W_{L:l+1}^Tu_L= \frac{\sigma_1}{\|W_{l:1}u_0\|}u_l,\quad \text{and}\quad  W_{L:l+1}u_l = \frac{1}{\|W_{l:1}u_0\|}W_{L:1}u_0 = \frac{\sigma_1}{\|W_{l:1}u_0\|}u_L.
    \end{equation}
    This further implies that the $u_l$ are a chain of common singular vectors for the $W_l$; indeed, for any $l$ one straightforwardly has
    \begin{align}
    W_lu_{l-1} = W_l\frac{W_{l-1:1}u_0}{\|W_{l-1:1}u_0\|} = \frac{1}{\|W_{l-1:1}u_0\|}W_{l:1}u_0 = \frac{\|W_{l:1}u_0\|}{\|W_{l-1:1}u_0\|}u_l,
    \end{align}
    while \eqref{eq:align1} implies that
    \begin{align}
    W_l^Tu_l = W_l^T\frac{\|W_{l:1}u_0\|}{\sigma_1}W_{L:l+1}^Tu_L = \frac{\|W_{l:1}u_0\|}{\sigma_1}W_{L:l}^Tu_L = \frac{\|W_{l:1}u_0\|}{\|W_{l-1:1}u_0\|}u_{l-1}.
    \end{align}
    
    Setting $s_l:=\|W_{l:1}u_0\|/\|W_{l-1:1}u_0\|$ for these singular values, equality between the arithmetic and geometric means above occurs if and only if the $s_l$ are all equal, with
    \begin{align}
    \prod_{l=1}^Ls_l = \sigma_1\Rightarrow s_l = \sigma_1^{1/L},\quad\forall l\in\{1,\dots,L\}.
    \end{align}
    It follows that for any $(W_1,\dots,W_L)$ achieving the lower-bound \eqref{eq:lowerbound}, with $u_0,\dots,u_L$ as above there must exist orthonormal frames $(U_1,\dots,U_{L-1})$ in the fibre of the principal bundle $P$ over $(u_1,\dots,u_{L-1})$ and, since $w\notin S$, $(\bW_1,\dots,\bW_L)\in\bM$ such that
    \begin{align}
    W_l = [u_l,U_l]\begin{pmatrix} \sigma_1^{1/L} & 0 \\ 0 & \bW_l\end{pmatrix}[u_{l-1},U_{l-1}]^T,\qquad \forall l\in[L].
    \end{align}
    This proves that $F$ is in the range of $\phi$; that every element in the range of $\phi$ meets the lower-bound \eqref{eq:lowerbound} is then clear, thus proving the first claim.

    For the second claim, fix $w\in\mathrm{range}(\phi)$ of the form
    \begin{align}
    \Bigg([u_l,U_l]\begin{pmatrix}\sigma_1^{1/L} & 0 \\ 0 & \bW_l\end{pmatrix}[u_{l-1},U_{l-1}]^T\Bigg)_{l=1}^L,
    \end{align}
    with $[u_0,U_0] = I_{d_0}$ and $[u_L,U_L] = I_{d_L}$. This $w$ may then be transformed into a point of the desired form \eqref{eq:wnormalform} by acting on it via the element $([u_1,U_1]^T,\dots,[u_{L-1},U_{L-1}]^T)$ of $\prod_{l=1}^{L-1}O(d_l)$ according to the 1-higher-dimensional version of \eqref{eq:oaction}.
\end{proof}

With $F$ identified as the fibre bundle $B\rightarrow\prod_{l=1}^{L-1}S^{d_l-1}$, it is relatively easy to characterise the directions normal and tangent to $F$, and in this way to prove that $\lambda_1$ is Morse-Bott along $F$. By the final claim of Proposition \ref{prop:fibrebundle}, it suffices to work at a point of the form \eqref{eq:wnormalform}.

\begin{proposition}\label{prop:tangents}
    At a point $w\in F$ of the form
    \begin{align}
    \Bigg(\begin{pmatrix} \sigma_1^{1/L} & 0 \\ 0 & \bW_l\end{pmatrix}\Bigg)_{l=1}^L
    \end{align}
    as in \eqref{eq:wnormalform}, with $\bar{w}=(\bW_1,\dots,\bW_L)\in\overline{M}$, given $a:=(a_l)_{l=1}^{L-1}\in \RB^{L-1}$, $b:=(b_l)_{l=1}^{L-1},c:=(c_l)_{l=1}^{L-1}\in\prod_{l=1}^{L-1}\RB^{d_l-1}$ and $D:=(D_l)_{l=1}^{L}\in T_{\bw}\bM$, denote
    \begin{align}
        v_1(w;a):=\Bigg(\begin{pmatrix} \sigma_1^{1/L}(a_l-a_{l-1}) & 0 \\ 0 & 0\end{pmatrix}\Bigg)_{l=1}^{L},
    \end{align}
    \begin{align}
        v_2(w;b,c):=\Bigg(\begin{pmatrix} 0 & \big(\bW_l^Tb_l - \sigma_1^{1/L}b_{l-1}\big)^T \\ \sigma_1^{1/L}c_l-\bW_lc_{l-1} & 0\end{pmatrix}\Bigg)_{l=1}^{L}
    \end{align}
    and
    \begin{align}
        v_3(w;D):=\bigg(\begin{pmatrix} 0 & 0 \\ 0 & D_l\end{pmatrix}\bigg)_{l=1}^{L},
    \end{align}
    where $a_0,b_0,c_0$ and $a_L,b_L,c_L$ are taken to be zero. Then all such $v_1(w;a),v_2(w;b,c),v_3(w;D)$ span the tangent space $T_wM$. Moreover, the tangent space to $F$ at $w$ is
    \begin{equation}
        T_wF = \mathrm{span}\bigg\{v_2(w,\xi,-\xi),v_3(w,D):\xi\in\prod_{l=1}^{L-1}\RB^{d_l-1},\,D\in T_{\bw}\bM\bigg\}
    \end{equation}
    and its orthogonal complement $\nu_wF$ in $T_wM$ is
    \begin{equation}
        \nu_wF = \mathrm{span}\bigg\{v_1(w,a),v_2(w,b,c):a\in\RB^{L-1},\,b,c\in\prod_{l=1}^{L-1}\RB^{d_l-1}\text{ s.t. }K_1(w)[b] = K_2(w)[c]\bigg\},
    \end{equation}
    where
    \begin{equation}
        K_1(w)[b]:=\bigg(\big(\sigma_1^{2/L}I_{d_l-1}+\bW_l\bW_l^T\big)b_l - \sigma_1^{1/L}\bW_lb_{l-1} - \sigma_1^{1/L}\bW_{l+1}^Tb_{l+1}\bigg)_{l=1}^{L-1}
    \end{equation}
    and
    \begin{equation}
        K_2(w)[c]:=\bigg(\big(\sigma_1^{2/L}I_{d_l-1}+\bW_{l+1}^T\bW_{l+1}\big)c_l-\sigma_1^{1/L}\bW_lc_{l-1}-\sigma_1^{1/L}\bW_{l+1}^Tc_{l+1}\bigg)_{l=1}^{L-1}.
    \end{equation}
\end{proposition}

\begin{proof}
    The $v_3(w,D)$ are elements of $T_wM$ since they are tangents to the fibre in the bundle $F$. That the $v_1(w;a),v_2(w;b,c)$ are elements of $T_wM$ follows from the fact that $(a,b,c)$ parametrise the family
    \begin{align}
    A:=\Bigg(A_l:=\begin{pmatrix} a_l & b_l^T \\ c_l & 0\end{pmatrix}\Bigg)_{l=1}^{L-1}\in\prod_{l=1}^{L-1}\RB^{d_l\times d_l},
    \end{align}
    which maps injectively into $M$ via plots about $w = (W_1,\dots,W_L)$ defined by
    \begin{align}
    \varphi(A_1,\dots,A_{L-1}):=\big(\exp(A_1)W_1,\exp(A_2)W_2\exp(-A_1),\dots,W_L\exp(-A_{L-1})\big).
    \end{align}
    Thus the derivatives at zero of the curves
    \begin{align}
    t\mapsto \big(\exp(tA_1)W_1,\exp(tA_2)W_2\exp(-tA_1),\dots,W_L\exp(-tA_{L-1})\big)
    \end{align}
    parametrised by $\prod_{l=1}^{L-1}\RB^{d_l\times d_l}$ form the tangent space $T_wM$. Fixing $(a,b,c)$ defining $A\in\prod_{l=1}^{L-1}\RB^{d_l\times d_l}$, this tangent vector is given by
    \begin{align}
    \Bigg(\begin{pmatrix} \sigma_1^{1/L}(a_l-a_{l-1}) & \big(\bW_l^Tb_l-\sigma_1^{1/L}b_{l-1}\big)^T \\ \sigma_1^{1/L}c_l-\bW_lc_{l-1} & 0\end{pmatrix}\Bigg)_{l=1}^{L},
    \end{align}
    with $a_0,b_0,c_0$ equal to zero. The family of all $v_1(w;a), v_2(w;b,c),v_3(w;D)$ determine $T_wM$ by dimension count.

    We now turn to the tangent and normal spaces to $F$, exploiting the fibre bundle structure thereof worked out in Proposition \ref{prop:fibrebundle}. At $w\in F$, the tangent directions are either tangent to the base $\prod_{l=1}^{L-1}S^{d_l-1}$ of the bundle $B$, or tangent to its fibre. At $w$ in the normal form \eqref{eq:wnormalform}, tangents to the base correspond to rotating the first standard basis vector into any of the others in each of the $L-1$ factors; these tangents are precisely of the form $v_2(w;\xi,-\xi)$ for any $\xi\in\prod_{l=1}^{L-1}\RB^{d_l-1}$. On the other hand, tangents to the fibre are precisely those which act solely on the bottom-right block of each of the $L$ factors, and are thus precisely of the form $v_3(w;D)$ for any $D\in T_{\bw}\bM$.

    Finally, $\nu_wF$ consists precisely of those elements of $T_wM$ that are orthogonal to $T_wF$ with respect to the factor-wise Frobenius inner product. Since $v_3(w;D)\in T_wF$ for all $D\in T_{\bw}\bM$, $\nu_wF$ must be spanned only by some subset of $v_1(w;a)$ and $v_2(w;b,c)$ for $a\in\RB^{L-1}$ and $b,c\in\prod_{l=1}^{L-1}\RB^{d_l-1}$. Blockwise orthogonality of $v_1(w;a)$ to $T_wF$ imply that every such vector is in $\nu_wF$; these directions correspond to changes of the top singular values of the factors away from $\sigma_1^{1/L}$. To determine which of the $v_2(w;b,c)$ lie in $\nu_wF$, fix $b,c,\xi\in\prod_{l=1}^{L-1}\RB^{d_l-1}$; orthogonality of $v_2(w;b,c)$ to $v_2(w;\xi,-\xi)$ requires
    \begin{align}
        0 &= \sum_{l=1}^{L}\tr\big(v_2(w;\xi,-\xi)^T_lv_2(w;b,c)_l\big)\\&=\sum_{l=1}^{L}\tr\Bigg(\begin{pmatrix}0 & \sigma_1^{1/L}\xi_l^T - \xi_{l-1}^T\bW_l^T \\ -\bW_l^T\xi_l + \sigma_1^{1/L}\xi_{l-1}\end{pmatrix}\begin{pmatrix} 0 & b_l^T\bW_l - \sigma_1^{1/L}b_{l-1}^T \\ \sigma_1^{1/L}c_l-\bW_lc_{l-1} & 0\end{pmatrix}\Bigg)\\&=\sum_{l=1}^{L-1}\xi_l^T\big(K_2(w)[c] - K_1(w)[b]\big)
    \end{align}
    where $\xi_L,b_L,c_L = 0$. Since this must hold for all $\xi\in\prod_{l=1}^{L-1}\RB^{d_l-1}$, it follows that $K_1(w;b) = K_2(w;c)$. This concludes the proof.
\end{proof}

Having determined $\nu_wF$, we may now move to proving that $\lambda_1$ is Morse-Bott along $F$, i.e. that the restriction of $\nabla^2_M\lambda_1$ to $\nu F$ is positive-definite. We first recall the following lemma.

\begin{lemma}\label{lem:secondorderperturbation}[\cite[3.2.2]{eigenvalue_perturbation}]
    Let $t\mapsto A(t)$ be a $C^2$ family of symmetric matrices. Assume that the top eigenvalue function $\lambda_1:t\mapsto \lambda_1(A(t))$ has $\lambda_1(0)$ being simple with corresponding eigenvector $u_1(0)$. Then:
    \begin{align}
    (\lambda_1\circ A)''(0) = u_1(0)^TA''(0)u_1(0) + 2u_1(0)^TA'(0)^T\big(\lambda_1(0)I-A(0)\big)|_{u_1(0)^{\perp}}^{-1}A'(0)u_1(0),
    \end{align}
    where $(\lambda_1(0)I-A(0))|_{u_1(0)^{\perp}}$ refers to the restriction of $\lambda_1(0)I-A(0)$ to the subspace $u_1(0)^{\perp}$, on which the inverse makes sense.
\end{lemma}

We also require some additional notation. Given $w\in F$ of the form \eqref{eq:wnormalform}, define $\Gamma_{01}(w):\prod_{l=1}^{L-1}\RB^{d_l-1}\rightarrow\RB^{d_0-1}$ and $\Gamma_{10}(w):\prod_{l=1}^{L-1}\RB^{d_l-1}\rightarrow\RB^{d_L-1}$ by
\begin{align}
\Gamma_{01}(w)[z]:=\sum_{l=1}^{L-1}\sigma_1^{(L-l-1)/L}\bW_{l:1}^Tz_l,\qquad \Gamma_{10}(w)[z]:=\sum_{l=1}^{L-1}\sigma_1^{(l-1)/L}\bW_{L:l+1}z_l.
\end{align}

\begin{proposition}\label{prop:morsebott}
    Let $w\in F$ be of the form \eqref{eq:wnormalform}. With $\lambda_1:=\lambda_1(Df^TDf)$, the subspaces
    \begin{align}
    V_1(w):=\mathrm{span}\bigg\{v_1(w;a):a\in\RB^{L-1}\bigg\},
    \end{align}
    \begin{align}
    V_2(w):=\mathrm{span}\bigg\{v_2(w;b,c):b,c\in\prod_{l=1}^{L-1}\RB^{d_l-1}\text{ s.t. }K_1(w)[b] = K_2(w)[c]\bigg\}
    \end{align}
    of $\nu_wF$ are invariant under $\nabla^2_M\lambda_1(w)$, with
    \begin{align}
    \mathrm{spec}(\nabla^2_M\lambda_1(w)|_{V_1(w)}) = 4\sigma_1^{2-4/L},
    \end{align}
    \begin{align}
    \mathrm{spec}(\nabla^2_M\lambda_1(w)|_{V_2(w)}) = \text{solutions $\mu(w)$ of generalised eigenproblem }H(w)v = \mu(w) G(w)v,
    \end{align}
    where, suppressing evaluation at $w$,
    \begin{align}
    H:=2\sigma_1^{2-2/L}\big(I_{\prod_{l=1}^{L-1}\RB^{d_l-1}} + \Gamma_{01}^T(\lambda_1I_{d_0-1}-\Delta_{01})^{-1}\Gamma_{01} + \Gamma_{10}^T(\lambda_1I_{d_L-1}-\Delta_{10})^{-1}\Gamma_{10}\big)
    \end{align}
    and
    \begin{align}
    G:=(K_1^{-1}+K_2^{-1})^{-1}
    \end{align}
    are positive-definite. Consequently, $\lambda_1$ is Morse-Bott along $F$.
\end{proposition}

\begin{proof}
    We first demonstrate invariance of the subspaces $V_1$ and $V_2$. Consider the tuple
    \begin{align}
    J:=(J_l)_{l=0}^L:=\Bigg(\begin{pmatrix} 1 & 0 \\ 0 & -I_{d_l-1}\end{pmatrix}\Bigg)_{l=0}^L\in\prod_{l=0}^L\RB^{d_l\times d_l}.
    \end{align}
    This $J$ acts as an involutive isometry
    \begin{align}\label{eq:Theta}
    \Theta: (W_1,\dots,W_L) \mapsto \big(J_1W_1J_0,\dots,J_LW_LJ_{L-1}\big)
    \end{align}
    satisfying $f(\Theta(w)) = J_Lf(w)J_0$ with $J_L\tau J_0= \tau$, and consequently acts as an isometry of $M$ preserving $\ell$ and $\lambda_1$. Hence $\nabla^2_M\lambda_1$ is equivariant with respect to $\Theta$ in the sense that
    \begin{align}
    D\Theta(w)\nabla^2_M\lambda_1(w)[v] = \nabla^2_M\lambda_1(\Theta(w))[D\Theta(w)[v]].
    \end{align}
    Since $w$ of the form \eqref{eq:wnormalform} is fixed by $\Theta$, this equivariance reduces to commutativity:
    \begin{align}
    D\Theta(w)\cdot \nabla^2_M\lambda_1(w) = \nabla^2_M\lambda_1(w)\cdot D\Theta(w).
    \end{align}
    It is moreover easily checked that $D\Theta(w)$ preserves $\nu_wF$. It thus follows that $\nabla^2_M\lambda_1(w)$ preserves the $+1$ and $-1$ eigenspaces of $D\Theta(w)|_{\nu_wF}$, which are precisely $V_1(w)$ and $V_2(w)$ respectively.

    Having demonstrated invariance of $V_1$ and $V_2$, we now compute the spectrum of $\nabla^2_M\lambda_1(w)$ restricted to each of these subspaces. We first consider $V_1(w)$. Fix $a\in\RB^{L-1}$, and consider the curve
    \begin{align}
    \gamma:t&\mapsto \Bigg(\begin{pmatrix}\exp(t(a_l-a_{l-1}))\sigma_1^{1/L} & 0 \\ 0 & \bW_l\end{pmatrix}\Bigg)_{l=1}^L,
    \end{align}
    where $a_L = a_0 = 0$. One then computes
    \begin{align}
        \lambda_1(\gamma(t)) &=\lambda_1(Df\,Df^T(\gamma(t)))\\&=\sum_{l=1}^L\prod_{m\neq l}\sigma_1^{2/L}\exp(2t(a_m-a_{m-1})).
    \end{align}
    Since $\gamma(t)\in M$ for all $t$, one has $\prod_{l=1}^L\sigma_1^{1/L}\exp(t(a_l-a_{l-1})) = \sigma_1$, implying that $\prod_{m\neq l}\sigma_1^{1/L}\exp(t(a_m-a_{m-1})) = \sigma_1/\big(\sigma_1^{1/L}\exp(t(a_l-a_{l-1}))\big)$, so that
    \begin{align}
    \lambda_1(\gamma(t)) &= \sum_{l=1}^L\frac{\sigma_1^2}{\sigma_1^{2/L}\exp(2t(a_l-a_{l-1}))} = \sigma_1^{2-2/L}\sum_{l=1}^L\exp\big(-2t(a_l-a_{l-1})\big).
    \end{align}
    Differentiating twice and evaluating at $t=0$ then gives
    \begin{align}
    (\lambda_1\circ\gamma)''(0)= 4\sigma_1^{2-2/L}\sum_{l=1}^L(a_l-a_{l-1})^2.
    \end{align}
    Finally, noting that $\|v_1(w;a)\|^2 =\sigma_1^{2/L}\sum_{l=1}^L(a_l-a_{l-1})^2$, one deduces that
    \begin{align}
    \frac{\langle v_1(w;a),\nabla^2_M\lambda_1(w)v_1(w;a)\rangle}{\|v_1(w;a)\|^2} = \frac{1}{\|v_1(w,a)\|^2}(\lambda_1\circ\gamma)''(0) = 4\sigma_1^{2-4/L},
    \end{align}
    so that $\mathrm{spec}(\nabla^2_M\lambda_1(w)|_{V_1(w)}) = 4\sigma_1^{2-4/L}$ as claimed.

    We now turn to computing the spectrum of $\nabla^2_M\lambda_1(w)|_{V_2(w)}$. Fix $b,c\in\prod_{l=1}^{L-1}\RB^{d_l-1}$ with $K_1(w)[b]=K_2(w)[c]$, and set
    $b_0=c_0=b_L=c_L=0$. Define
    \[
    A_l:=\begin{pmatrix}0&b_l^T\\ c_l&0\end{pmatrix}\in\RB^{d_l\times d_l},
    \qquad l=1,\dots,L-1,
    \]
    and set $A_0:=0\in\RB^{d_0\times d_0}$ and $A_L:=0\in\RB^{d_L\times d_L}$.
    Consider the curve
    \begin{align}
    t\mapsto w(t):=\big(W_l(t)\big)_{l=1}^L
    :=\big(\exp(tA_l)W_l\exp(-tA_{l-1})\big)_{l=1}^L.
    \end{align}
    Then $w(0)=w$ and
    \begin{align}
    \dot W_l(0)=A_lW_l-W_lA_{l-1}
    =
    \begin{pmatrix}
    0 & \big(\bW_l^Tb_l-\sigma_1^{1/L}b_{l-1}\big)^T \\
    \sigma_1^{1/L}c_l-\bW_lc_{l-1} & 0
    \end{pmatrix},
    \end{align}
    so that $\dot w(0)=v_2(w;b,c)$ as required.
    
    Since the intermediate exponentials telescope, one has
    \begin{align}
    W_{L:l+1}(t)=W_{L:l+1}\exp(-tA_l),\qquad
    W_{l-1:1}(t)=\exp(tA_{l-1})W_{l-1:1}.
    \end{align}
    Writing $z_l:=b_l+c_l$ for $l=0,\dots,L$, note that
    \begin{align}
    A_l+A_l^T
    =
    \begin{pmatrix}
    0&z_l^T\\ z_l&0
    \end{pmatrix},
    \end{align}
    and
    \begin{align}
    A_l^TA_l+\frac12\big(A_l^2+(A_l^T)^2\big)
    &=
    \begin{pmatrix}
    b_l^Tc_l+\|c_l\|^2 & 0\\
    0 & \frac12\big(c_lb_l^T+b_lc_l^T+2b_lb_l^T\big)
    \end{pmatrix},
    \\
    A_lA_l^T+\frac12\big(A_l^2+(A_l^T)^2\big)
    &=
    \begin{pmatrix}
    b_l^Tc_l+\|b_l\|^2 & 0\\
    0 & \frac12\big(c_lb_l^T+b_lc_l^T+2c_lc_l^T\big)
    \end{pmatrix}.
    \end{align}
    Using the expansions $\exp(tA)=I+tA+\frac12t^2A^2+O(t^3)$ and
    $\exp(-tA)=I-tA+\frac12t^2A^2+O(t^3)$, it follows that
    \begin{align}
    \exp(tA_{l-1}^T)\exp(tA_{l-1})
    &=
    I+t(A_{l-1}+A_{l-1}^T)
    +t^2\Big(A_{l-1}^TA_{l-1}+\frac12(A_{l-1}^2+(A_{l-1}^T)^2)\Big)
    +O(t^3),
    \\
    \exp(-tA_l)\exp(-tA_l^T)
    &=
    I-t(A_l+A_l^T)
    +t^2\Big(A_lA_l^T+\frac12(A_l^2+(A_l^T)^2)\Big)
    +O(t^3).
    \end{align}
    Hence, defining
    \begin{align}
    P_l&:=
    \begin{pmatrix}
    0 & \sigma_1^{(l-1)/L}z_{l-1}^T\bW_{l-1:1}\\
    \sigma_1^{(l-1)/L}\bW_{l-1:1}^Tz_{l-1} & 0
    \end{pmatrix},
    \\
    Q_l&:=
    \begin{pmatrix}
    0 & \sigma_1^{(L-l)/L}z_l^T\bW_{L:l+1}^T\\
    \sigma_1^{(L-l)/L}\bW_{L:l+1}z_l & 0
    \end{pmatrix},
    \\
    R_l&:=
    \begin{pmatrix}
    \sigma_1^{2(l-1)/L}\big(b_{l-1}^Tc_{l-1}+\|c_{l-1}\|^2\big) & 0\\
    0 & \frac12\bW_{l-1:1}^T\big(b_{l-1}c_{l-1}^T+c_{l-1}b_{l-1}^T+2b_{l-1}b_{l-1}^T\big)\bW_{l-1:1}
    \end{pmatrix},
    \\
    S_l&:=
    \begin{pmatrix}
    \sigma_1^{2(L-l)/L}\big(b_l^Tc_l+\|b_l\|^2\big) & 0\\
    0 & \frac12\bW_{L:l+1}\big(b_lc_l^T+c_lb_l^T+2c_lc_l^T\big)\bW_{L:l+1}^T
    \end{pmatrix},
    \end{align}
    one obtains
    \begin{align}
    W_{l-1:1}(t)^TW_{l-1:1}(t)
    &=
    W_{l-1:1}^TW_{l-1:1}
    +tP_l+t^2R_l+O(t^3),
    \\
    W_{L:l+1}(t)W_{L:l+1}(t)^T
    &=
    W_{L:l+1}W_{L:l+1}^T
    -tQ_l+t^2S_l+O(t^3).
    \end{align}
    Therefore
    \begin{align}
    Df\,Df^T(w(t))
    &=
    \sum_{l=1}^L W_{L:l+1}(t)W_{L:l+1}(t)^T\otimes W_{l-1:1}(t)^TW_{l-1:1}(t)
    \\
    &=
    \sum_{l=1}^L W_{L:l+1}W_{L:l+1}^T\otimes W_{l-1:1}^TW_{l-1:1}
    \\
    &\quad
    +t\sum_{l=1}^L\Big(
    W_{L:l+1}W_{L:l+1}^T\otimes P_l
    -
    Q_l\otimes W_{l-1:1}^TW_{l-1:1}
    \Big)
    \\
    &\quad
    +t^2\sum_{l=1}^L\Big(
    W_{L:l+1}W_{L:l+1}^T\otimes R_l
    +
    S_l\otimes W_{l-1:1}^TW_{l-1:1}
    -
    Q_l\otimes P_l
    \Big)
    +O(t^3).
    \end{align}
    
    Let $E_{00}:=e_1e_1^T\in\RB^{d_L\times d_0}$, where the first $e_1$ is the first standard
    basis vector in $\RB^{d_L}$ and the second is the first standard basis vector in $\RB^{d_0}$.
    Then
    \begin{align}
    \big\langle E_{00},(Df\,Df^T\circ w)''(0)E_{00}\big\rangle
    &=
    2\sum_{l=1}^L
    \sigma_1^{2(L-l)/L}\sigma_1^{2(l-1)/L}
    \big(
    b_{l-1}^Tc_{l-1}+\|c_{l-1}\|^2+b_l^Tc_l+\|b_l\|^2
    \big)
    \\
    &=
    2\sigma_1^{2-2/L}\sum_{l=1}^L
    \big(
    \|b_l\|^2+\|c_l\|^2+2b_l^Tc_l
    \big)
    \\
    &=
    2\sigma_1^{2-2/L}\sum_{l=1}^{L-1}\|z_l\|^2.
    \end{align}
    On the other hand,
    \begin{align}
    (Df\,Df^T\circ w)'(0)E_{00}
    &=
    \sum_{l=1}^L
    \sigma_1^{2(L-l)/L}E_{00}P_l
    -
    \sum_{l=1}^L
    Q_lE_{00}\sigma_1^{2(l-1)/L}
    \\
    &=
    \sigma_1^{1-1/L}
    \sum_{l=1}^L
    \begin{pmatrix}
    0 & \sigma_1^{(L-l)/L}z_{l-1}^T\bW_{l-1:1}\\
    -\sigma_1^{(l-1)/L}\bW_{L:l+1}z_l & 0
    \end{pmatrix}
    \\
    &=
    \sigma_1^{1-1/L}
    \begin{pmatrix}
    0 & \Gamma_{01}(w)[z]^T\\
    -\Gamma_{10}(w)[z] & 0
    \end{pmatrix},
    \end{align}
    where $z:=(z_l)_{l=1}^{L-1}=b+c$.
    
    Applying Lemma \ref{lem:secondorderperturbation}, and using that
    $\lambda_1(w)=L\sigma_1^{2-2/L}$ while the restriction of
    $\lambda_1(w)I-Df\,Df^T(w)$ to the top-right and bottom-left blocks is
    $\lambda_1(w)I-\Delta_{01}(w)$ and $\lambda_1(w)I-\Delta_{10}(w)$ respectively, gives
    \begin{align}
    (\lambda_1\circ w)''(0)
    &=
    2\sigma_1^{2-2/L}\Big(
    \|z\|^2
    +
    \big\langle \Gamma_{01}(w)[z],(\lambda_1(w)I-\Delta_{01}(w))^{-1}\Gamma_{01}(w)[z]\big\rangle
    \\
    &\hspace{4cm}
    +
    \big\langle \Gamma_{10}(w)[z],(\lambda_1(w)I-\Delta_{10}(w))^{-1}\Gamma_{10}(w)[z]\big\rangle
    \Big)
    \\
    &=
    \langle z,H(w)z\rangle,
    \end{align}
    where
    \begin{align}
    H(w):=
    2\sigma_1^{2-2/L}\Big(
    I
    +\Gamma_{01}(w)^T(\lambda_1(w)I-\Delta_{01}(w))^{-1}\Gamma_{01}(w)
    +\Gamma_{10}(w)^T(\lambda_1(w)I-\Delta_{10}(w))^{-1}\Gamma_{10}(w)
    \Big).
    \end{align}
    
    Finally, since $K_1(w)[b]=K_2(w)[c]$, setting
    \[
    \delta:=K_1(w)[b]=K_2(w)[c]
    \]
    gives
    \[
    z=b+c=(K_1(w)^{-1}+K_2(w)^{-1})\delta.
    \]
    Hence
    \[
    \delta=(K_1(w)^{-1}+K_2(w)^{-1})^{-1}z=G(w)z,
    \]
    so that
    \[
    b=K_1(w)^{-1}G(w)z,\qquad c=K_2(w)^{-1}G(w)z.
    \]
    Therefore
    \begin{align}
    \|\dot w(0)\|^2
    &=
    \langle b,K_1(w)b\rangle+\langle c,K_2(w)c\rangle
    =
    \langle b+c,G(w)z\rangle
    =
    \langle z,G(w)z\rangle.
    \end{align}
    Thus the spectrum of $\nabla^2_M\lambda_1(w)|_{V_2(w)}$ consists of the critical values of the
    Rayleigh quotient
    \[
    z\mapsto \frac{\langle z,H(w)z\rangle}{\langle z,G(w)z\rangle},
    \]
    equivalently the solutions of the generalised eigenvalue problem
    \[
    H(w)z=\mu\,G(w)z.
    \]
    Since $H(w)$ and $G(w)$ are positive-definite, all such eigenvalues are strictly positive.
\end{proof}

We next prove invariance of $\mathrm{span}(\nu_1)$ at points in $F$.

\begin{proposition}\label{prop:GDlineinvariance}
    For every $\eta\in\RB$ and every $w\in F$, the affine line
    \begin{align}
        w+\mathrm{span}\{\nu_1(w)\}
    \end{align}
    is invariant under $\GD(\eta,\cdot)$. Consequently, point 2 of Assumption \ref{ass:morsebott} holds.
\end{proposition}

\begin{proof}
    By the final claim of Proposition \ref{prop:fibrebundle}, it suffices to consider a point
    \begin{align}
        w
        =
        \Bigg(
        \begin{pmatrix}
            \sigma_1^{1/L}&0\\
            0&\bW_l
        \end{pmatrix}
        \Bigg)_{l=1}^L
    \end{align}
    of the form \eqref{eq:wnormalform}. At such a point,
    \begin{align}
        \nu_1(w)
        =
        \frac{1}{\sqrt L}
        \big(e_1(d_l)e_1(d_{l-1})^T\big)_{l=1}^L.
    \end{align}
    Fixing $y\in\RB$, set
    \begin{align}
        r:=\sigma_1^{1/L}+\frac{y}{\sqrt L}.
    \end{align}
    The factors of $\widetilde w:=w+y\nu_1(w)$ are then
    \begin{align}
        \widetilde W_l
        =
        \begin{pmatrix}
            r&0\\
            0&\bW_l
        \end{pmatrix},
    \end{align}
    so that
    \begin{align}
        f(\widetilde w)-\tau
        =
        (r^L-\sigma_1)e_1(d_L)e_1(d_0)^T.
    \end{align}
    It follows that, for every $l=1,\dots,L$,
    \begin{align}
        \nabla_{W_l}\ell(\widetilde w)
        &=
        \widetilde W_{L:l+1}^T
        \big(f(\widetilde w)-\tau\big)
        \widetilde W_{l-1:1}^T\\
        &=
        (r^L-\sigma_1)r^{L-1}
        e_1(d_l)e_1(d_{l-1})^T.
    \end{align}
    Therefore
    \begin{align}
        \nabla\ell(\widetilde w)
        =
        \sqrt L(r^L-\sigma_1)r^{L-1}\nu_1(w),
    \end{align}
    and hence
    \begin{align}
        \GD(\eta,w+y\nu_1(w))
        =
        \left(
        \eta,\,
        w+
        \left(
        y-\eta\sqrt L(r^L-\sigma_1)r^{L-1}
        \right)\nu_1(w)
        \right).
    \end{align}
    Thus the affine line through $w$ spanned by $\nu_1(w)$ is invariant. Finally, the internal orthogonal transformations used in Proposition \ref{prop:fibrebundle} are isometries preserving $f$, and gradient descent is equivariant under these transformations. The same conclusion therefore holds at every $w\in F$, proving the claim.
\end{proof}

We now recall the function $\alpha$ defined in Assumption \ref{ass:quality} by the formula
\begin{align}\label{eq:alpha}
\alpha_{\eta}(x):=-\bigg(D^3\ell[\nu_1^{\otimes 2},A_{\eta}\nabla^3\ell[\nu_1^{\otimes 2}]](x) + \frac{1}{3}D^4\ell[\nu_1^{\otimes 4}](x)\bigg)
\end{align}
for all $x\in M$ and $\eta$ in a neighbourhood of $2/\lambda_1(x)$, where $\nu_1$ is the top eigenfield of $\nabla^2\ell|_M$ and
\begin{align}
    A_{\eta}:=((2\lambda_1-\eta\lambda_1^2)I-\nabla^2\ell)|_{\nu_{2:q}M}^{-1}P_{2:q} + \lambda_1^{-1}(1-\eta\lambda_1)^{-1}P_1,
\end{align}
as in Assumption \ref{ass:quality}, with $P_1:T\RB^p|_M\rightarrow \nu_1M$ and $P_{2:q}:T\RB^p|_M\rightarrow \nu_{2:q}M$ being the orthogonal projections. In order to carry out computations of $\alpha_{\eta}$ in our matrix factorisation example, we introduce some notation.

Speaking generally, given symmetric multilinear maps $A$ and $B$ taking $k$ and $l$ inputs respectively and taking values in an inner product space $V$, we will denote
\begin{align}
\langle A\odot B\rangle[v_1,\dots,v_{k+l}]:=\frac{1}{(k+l)!}\sum_{\pi\in\Pi(k+l)}\langle A[v_{\pi(1)},\dots,v_{\pi(k)}],B[v_{\pi(k+1)},\dots,v_{\pi(k+l)}]\rangle
\end{align}
where $\Pi(k+l)$ denotes the set of all permutations of $k+l$ elements. Then in the general setup of Subsection \ref{subsec:problemsetting}, since $\ell = \frac{1}{2}\langle f-\tau,f-\tau\rangle$ one has
\begin{align}\label{eq:Dell}
    D\ell =\langle Df,f-\tau\rangle\Rightarrow D\ell|_M\equiv 0,
\end{align}
\begin{align}\label{eq:D2ell}
    D^2\ell = \langle Df\odot Df\rangle + \langle D^2f,f-\tau\rangle\Rightarrow D^2\ell|_M = \langle Df\odot Df\rangle,
\end{align}
\begin{align}\label{eq:D3ell}
    D^3\ell = 3\langle D^2f\odot Df\rangle + \langle D^3f,f-\tau\rangle\Rightarrow D^3\ell|_M = 3\langle D^2f\odot Df\rangle,
\end{align}
\begin{align}\label{eq:D4ell}
    D^4\ell = 4\langle D^3f\odot Df\rangle + 3\langle D^2f\odot D^2f\rangle + \langle D^4f,f-\tau\rangle\Rightarrow D^4\ell|_M = 4\langle D^3f\odot Df\rangle + 3\langle D^2f\odot D^2f\rangle.
\end{align}
For the matrix factorisation function $f:\prod_{l=1}^L\RB^{d_l\times d_{l-1}}\rightarrow\RB^{d_L\times d_0}$ of \eqref{eq:matrixfactorisationf}, at a point $w = (W_1,\dots,W_L)$ and for tangent vectors $\xi^i:=(\xi_1^i,\dots,\xi_L^i)$, $i=1,\dots,k$, one in particular has
\begin{align}\label{eq:Dkf}
    D^kf(w)[\xi^1,\dots,\xi^k] = \sum_{1\leq l_1<\dots<l_k\leq L}\sum_{\pi\in\Pi(k)}W_{L:l_k+1}\xi^{\pi(1)}_{l_k}W_{l_k-1:l_{k-1}+1}\cdots W_{l_2-1:l_1+1}\xi^{\pi(k)}_{l_1}W_{l_{1}-1:1}.
\end{align}
We now prove that $\alpha_{\eta}$ has the desired properties.

\begin{proposition}
    For the matrix factorisation model $f$ of \eqref{eq:matrixfactorisationf}, the function $\alpha_{\eta}$ of \eqref{eq:alpha} is constant on $F$, with
    \begin{align}
    \alpha_{\eta}|_{F} = -\bigg(\frac{9(L-1)^2}{(1-\eta\lambda_1|_{F})} + \frac{(L-1)(7L-11)}{3}\bigg)\sigma_1^{2-4/L},
    \end{align}
    for all $\eta$ in a neighbourhood of $2/\lambda_1|_{F}$, and in particular satisfies
    \begin{align}
    \alpha_{2/\lambda_1}|_{F} = (L-1)\sigma_1^{2-4/L}\bigg(\frac{20L-16}{3}\bigg) >0
    \end{align}
    for all $L\geq 2$. Moreover, $D\alpha_{\eta}|_{F} = 0$.
\end{proposition}

\begin{proof}
    Since $\alpha_{\eta}$ is defined solely in terms of derivatives of $f$ and in terms of the metric, it suffices to do the computations at a point $w$ of the form \eqref{eq:wnormalform}, i.e.
    \begin{align}
    w  = \Bigg(\begin{pmatrix}\sigma_1^{1/L} & 0 \\ 0 & \bW_{l}\end{pmatrix}\Bigg)_{l=1}^L
    \end{align}
    for some $(\bW_1,\dots,\bW_L)\in \bM$. It is then clear that one has
    \begin{align}
    \nu_1(w) = \frac{1}{\sqrt{L}}\big(e_1(d_l)e_1(d_{l-1})^T\big)_{l=1}^L.
    \end{align}
    We begin by proving that $\alpha_{\eta}$ is constant on $F$ for all $\eta$ sufficiently close to $2/\lambda_1|_{F}$. Suppressing evaluation at $w$ for notational convenience, using \eqref{eq:Dkf} one computes
    \begin{align}
    Df[\nu_1] = \sqrt{L}\sigma_1^{1-1/L}e_1(d_L)e_1(d_0)^T,
    \end{align}
    \begin{align}
    D^2f[\nu_1^{\odot 2}] = (L-1)\sigma_1^{1-2/L}e_1(d_L)e_1(d_0)^T,
    \end{align}
    and
    \begin{align}
    D^3f[\nu_1^{\odot 3}] =\frac{(L-2)(L-1)}{\sqrt{L}}
\sigma_1^{1-3/L}e_1(d_L)e_1(d_0)^T.
    \end{align}
    Using \eqref{eq:D4ell}, these identities imply that
    \begin{align}
        D^4\ell[\nu_1^{\odot 4}] &= 4\langle D^3f[\nu_1^{\odot 3}],Df[\nu_1]\rangle + 3\langle D^2f[\nu_1^{\odot 2}],D^2f[\nu_1^{\odot 2}]\rangle\\&=(L-1)(7L-11)\sigma_1^{2-4/L}.
    \end{align}
    To compute the term $D^3\ell[\nu_1^{\odot 2},A\nabla^3\ell[\nu_1^{\odot 2}]]$, using \eqref{eq:Dkf} and fixing a tangent vector $h = (h_1,\dots,h_L)\in\prod_{l=1}^L\RB^{d_l\times (d_{l-1})}$, one computes
    \begin{align}
    D^2f[\nu_1,h] = \frac{\sigma_1^{1-2/L}}{\sqrt{L}}\bigg(\sum_{l_1=1}^L\sum_{l_2\neq l_1}e_1(d_{l_2})^Th_{l_2}e_1(d_{l_2-1})\bigg) \,e_1(d_L)e_1(d_0)^T + B,
    \end{align}
    and
    \begin{align}
    Df[h] = \sigma_1^{1-1/L}\bigg(\sum_{l=1}^Le_1(d_l)^Th_le_1(d_{l-1})\bigg)e_1(d_L)e_1(d_0)^T + B'
    \end{align}
    where $B,B'\in \big(e_1(d_L)e_1(d_0)^T\big)^{\perp}$. Then by \eqref{eq:D3ell} one has
    \begin{align}
    D^3\ell[\nu_1^{\odot 2},h] &= 2\langle D^2f[\nu_1,h],Df[\nu_1]\rangle + \langle D^2f[\nu_1^{\odot 2}],Df[h]\rangle\\&=3(L-1)\sigma_1^{2-3/L}\sum_{l=1}^Le_1(d_l)^Th_le_1(d_{l-1})\\&=3\sqrt{L}(L-1)\sigma_1^{2-3/L}\langle\nu_1,h\rangle.
    \end{align}
    It follows that $\nabla^3\ell[\nu_1^{\odot 2}] = 3\sqrt{L}(L-1)\sigma_1^{2-3/L}\nu_1$, hence that
    \begin{align}
        A_{\eta}\nabla^3\ell[\nu_1^{\odot2}] = \frac{3\sqrt{L}(L-1)\sigma_1^{2-3/L}}{\lambda_1(1-\eta\lambda_1)}\nu_1,
    \end{align}
    implying further that
    \begin{align}
        D^3\ell[\nu_1^{\odot 2},A_{\eta}\nabla^3\ell[\nu_1^{\odot 2}]] &=\frac{3\sqrt{L}(L-1)\sigma_1^{2-3/L}}{\lambda_1(1-\eta\lambda_1)}D^3\ell[\nu_1^{\odot 3}]\\&=\frac{9L(L-1)^2\sigma_1^{4-6/L}}{\lambda_1(1-\eta\lambda_1)}\\&=\frac{9(L-1)^2\sigma_1^{2-4/L}}{1-\eta\lambda_1},
    \end{align}
    with the final line following from $\lambda_1|_{F} = L\sigma_1^{2-2/L}$. Putting everything together, one arrives at the claimed formula for $\alpha_{\eta}$.

    We now move on to proving that $D\alpha_{\eta}(w) = 0$. Since $\alpha_{\eta}|_{F}$ is constant, it suffices to show that $D\alpha_{\eta}(w)$ kills the normal directions to $F$ as worked out in Proposition \ref{prop:tangents}. Recall from Proposition \ref{prop:morsebott} that $\nu_wF$ splits into subspaces $V_1(w)$ and $V_2(w)$ consisting of modifications to the top-left entries of the factors and of rotations of their corresponding singular vectors respectively.
    
    It is relatively easy to show that $D\alpha_{\eta}(w)|_{V_2(w)}\equiv 0$. Recall the isometry $\Theta$ from \eqref{eq:Theta}, under which $f$ is equivariant and for which $D\Theta(w)$ has $-1$ eigenspace precisely equal to $V_2(w)$. Since  $f$ is equivariant under $\Theta$, $\alpha_{\eta}$ is invariant under $\Theta$; and since $\Theta(w) = w$, $D\alpha_{\eta}(w)$ is also invariant under $\Theta$ so that for any $\xi\in V_2$ one has
    \begin{align}
    D\alpha_{\eta}(w)[\xi] = D\alpha_{\eta}(w)[D\Theta(w)\xi] = -D\alpha_{\eta}(w)[\xi],
    \end{align}
    implying that $D\alpha_{\eta}(w)[\xi] = 0$ and hence that $D\alpha_{\eta}(w)|_{V_2(w)}\equiv 0$.

    We now turn to demonstrating $D\alpha_{\eta}(w)|_{V_1(w)}\equiv 0$. Fix a tuple $a = (a_l)_{l=1}^{L-1}\in\RB^{L-1}$ and for notational convenience denote $r_l:=a_l-a_{l-1}$ for all $l=1,\dots,L$, where we set $a_0,a_L:=0$. As in Proposition \ref{prop:morsebott}, we consider the curve
    \begin{align}
    \gamma:t\mapsto \Bigg(\begin{pmatrix}\exp(tr_l)\sigma_1^{1/L} & 0 \\ 0 & \bW_l\end{pmatrix}\Bigg)_{l=1}^L,
    \end{align}
    with $\gamma(0) = w$, and evaluate $\alpha_{\eta}(\gamma(t))$. To this end, we denote
    \begin{align}
    S_j(t):=\sum_{l=1}^L\exp(-2jtr_l),\qquad j=1,2,3.
    \end{align}
    One has
    \begin{align}
    \nu_1(\gamma(t)) &= \frac{1}{\|Df(\gamma(t))^T[e_1(d_L)e_1(d_0)^T]\|}Df(\gamma(t))^T[e_1(d_L)e_1(d_0)^T]\\&=S_1(t)^{-1/2}\big(\exp(-tr_l)e_1(d_l)e_1(d_{l-1})^T\big)_{l=1}^L.
    \end{align}
    From now on, we suppress evaluation at $\gamma(t)$ for notational convenience. Using the fact that $\prod_{l=1}^L\exp(tr_l) = 1$, note that for any tuple $l_1<\dots<l_k$ one has
    \begin{align}\label{eq:prod}
        \prod_{m\neq l_1,\dots,l_k}\exp(tr_m) = \prod_{j=1}^k\exp(-tr_{l_j}).
    \end{align}
    We now turn to computing the quantities appearing in \eqref{eq:alpha}. Invoking \eqref{eq:Dkf} gives
    \begin{align}
        Df[\nu_1] &= \bigg(\sum_{l=1}^LS_1^{-1/2}\exp(-tr_l)\sigma_1^{1-1/L}\prod_{m\neq l}\exp(tr_m)\bigg)e_1(d_L)e_1(d_0)^T\\&=S_1^{-1/2}\sigma_1^{1-1/L}\bigg(\sum_{l=1}^L\exp(-2tr_l)\bigg)e_1(d_L)e_1(d_0)^T\\&=\sigma_1^{1-1/L}S_1^{1/2}e_1(d_L)e_1(d_0)^T
    \end{align}
    by invoking \eqref{eq:prod} for the second line. Similarly, invoking \eqref{eq:Dkf} and \eqref{eq:prod} gives
    \begin{align}
        D^2f[\nu_1^{\odot 2}] &= \sigma_1^{1-2/L}S_1^{-1}\bigg(\sum_{l_1=1}^L\sum_{l_2\neq l_1}\exp(-tr_{l_1})\exp(-tr_{l_2})\prod_{m\neq l_1,l_2}\exp(tr_m)\bigg)e_1(d_L)e_1(d_0)^T\\&=\sigma_1^{1-2/L}S_1^{-1}\bigg(\sum_{l_1=1}^L\exp(-2tr_{l_1})\sum_{l_2\neq l_1}\exp(-2tr_{l_2})\bigg)e_1(d_L)e_1(d_0)^T\\&=\sigma_1^{1-2/L}\frac{S_1^2-S_2}{S_1}e_1(d_L)e_1(d_0)^T
    \end{align}
    by using the identity $\sum_{l_1\neq l_2}p_{l_1}p_{l_2} = \big(\sum_lp_l\big)^2 - \sum_lp_l^2$ for the final line, and
    \begin{align}
        D^3f[\nu_1^{\odot 3}] &= S_1^{-3/2}\sigma_1^{1-3/L}\bigg(\sum_{l_1=1}^L\sum_{l_2\neq l_1}\sum_{l_3\neq l_2,l_1}\exp(-t(r_{l_1} + r_{l_2} + r_{l_3}))\prod_{m\neq l_1,l_2,l_3}\exp(tr_m)\bigg)e_1(d_L)e_1(d_0)^T\\&=S_1^{-3/2}\sigma_1^{1-3/L}\bigg(\sum_{l_1\neq l_2\neq l_3}\exp(-2t(r_{l_1}+r_{l_2}+r_{l_3}))\bigg)e_1(d_L)e_1(d_0)^T\\&=\sigma_1^{1-3/L}\frac{S_1^3-3S_2S_1 + 2S_3}{S_1^{3/2}}e_1(d_L)e_1(d_0)^T
    \end{align}
    by using the identity $\sum_{l_1\neq l_2\neq l_3}p_{l_1}p_{l_2}p_{l_3} = \big(\sum_lp_l\big)^3-3\big(\sum_lp_l\big)\big(\sum_lp_l^2\big) + 2\big(\sum_lp_l^3\big)$ for the final line. It follows from \eqref{eq:D4ell} that
    \begin{align}
        D^4\ell[\nu_1^{\odot 4}] = \sigma_1^{2-4/L}\bigg(4\frac{S_1^3-3S_1S_2+2S_3}{S_1} + 3\bigg(\frac{S_1^2-S_2}{S_1}\bigg)^2\bigg).
    \end{align}
    Finally, we come to computing the term $D^3\ell[\nu_1^{\odot 2},A_{\eta}\nabla^3\ell[\nu_1^{\odot 2}]]$. For arbitrary $h = (h_1,\dots,h_L)\in\prod_{l=1}^L\RB^{d_l\times d_{l-1}}$, using \eqref{eq:Dkf} and \eqref{eq:prod} one has
    \begin{align}
        Df[h] &= \sigma_1^{1-1/L}\bigg(\sum_{l=1}^L\prod_{m\neq l}\exp(tr_m)\,e_1(d_l)^Th_le_1(d_{l-1})\bigg)e_1(d_L)e_1(d_0)^T + C\\&=\sigma_1^{1-1/L}\bigg(\sum_{l=1}^L\exp(-tr_l)\,e_1(d_l)^Th_le_1(d_{l-1})\bigg)e_1(d_L)e_1(d_0)^T + C
    \end{align}
    for some $C\in \big(e_1(d_L)e_1(d_0)\big)^{\perp}$, and
    \begin{align}
        D^2f[\nu_1,h] &= \sigma_1^{1-2/L}S_1^{-1/2}\bigg(\sum_{l_1=1}^L\sum_{l_2\neq l_1}\exp(-tr_{l_1})\,e_1(d_{l_2})^Th_{l_2}e_1(d_{l_2-1})\prod_{m\neq l_1,l_2}\exp(tr_m)\bigg)e_1(d_L)e_1(d_0)^T + C'\\&=\sigma_1^{1-2/L}S_1^{-1/2}\bigg(\sum_{l_1=1}^L\sum_{l_2\neq l_1}\exp(-t(2r_{l_1} + r_{l_2}))e_1(d_{l_2})^Th_{l_2}e_1(d_{l_2-1})\bigg)e_1(d_L)e_1(d_0)^T + C'
    \end{align}
    for some $C'\in\big(e_1(d_L)e_1(d_0)^T\big)^{\perp}$. By \eqref{eq:D3ell}, it follows that
    \begin{align}
        D^3\ell[\nu_1^{\odot 2},h] &= 2\langle D^2f[\nu_1,h],Df[\nu_1]\rangle + \langle D^2f[\nu_1^{\odot 2}],Df[h]\rangle\\&=\sigma_1^{2-3/L}\sum_{l=1}^L\bigg(\exp(-tr_l)\bigg(\frac{S_1^2-S_2}{S_1} + 2(S_1-\exp(-2tr_l))\bigg)e_1(d_l)^Th_le_1(d_{l-1})\bigg),
    \end{align}
    so that
    \begin{align}
        \nabla^3\ell[\nu_1^{\odot 2}] = \bigg(\sigma_1^{2-3/L}\exp(-tr_l)\bigg(\frac{S_1^2-S_2}{S_1} + 2(S_1-\exp(-2tr_l))\bigg)e_1(d_l)e_1(d_{l-1})^T\bigg)_{l=1}^L.
    \end{align}
    Since each of the components of this tuple is nonzero only in the top-left entry, its projection into the span of $\nu_{2:q}$ is zero while its projection into the span of $\nu_1$ is given by
    \begin{align}
        \nu_1^T\nabla^3\ell[\nu_1^{\odot 2}] &= D^3\ell[\nu_1^{\odot 3}] = 3\sigma_1^{2-3/L}\frac{S_1^2-S_2}{S_1^{1/2}}
    \end{align}
    by \eqref{eq:D3ell}, hence
    \begin{align}
        D^3\ell[\nu_1^{\odot 2},A\nabla^3\ell[\nu_1^{\odot 2}]] &= \lambda_1^{-1}(1-\eta\lambda_1)^{-1}\big(D^3\ell[\nu_1^{\odot 3}]\big)^2\\&=\frac{9\sigma_1^{2-4/L}(S_1^2-S_2)^2}{(1-\eta\lambda_1)S_1^2}
    \end{align}
    after substituting $\lambda_1 = \sigma_1^{2-2/L}S_1$. Thus, along the curve $\gamma(t)$ one has
    \begin{align}
        \alpha_{\eta} = -\sigma_1^{2-4/L}\bigg(\frac{9(S_1^2-S_2)^2}{(1-\eta\lambda_1)S_1^2} + \frac{4(S_1^3-3S_2S_1+2S_3)}{3S_1} + \frac{(S_1^2-S_2)^2}{S_1^2}\bigg).
    \end{align}
    Finally, note that for any $j=1,2,3$ one has $\dot{S}_j(0) = -2j\sum_{l=1}^Lr_l = -2j\sum_{l=1}^L(a_l-a_{l-1}) = 0$ since $a_0 = a_L = 0$, and similarly $\dot{\lambda_1}(0) = \sigma_1^{2-2/L}\dot{S}_1(0) = 0$. It follows that $\dot{\alpha}_{\eta}(0) = 0$, thus completing the proof.
\end{proof}




\end{document}